\documentclass[twoside,11pt]{article}

\usepackage{amsthm}

\usepackage{jmlr2e}

\usepackage{amsmath,amssymb,mathabx,amsfonts}
\usepackage{nicefrac}       % compact symbols for 1/2, etc.
\usepackage{ascmac}
\usepackage{bm}
\usepackage{natbib}
\usepackage{graphicx,here,comment,float,wrapfig}
\usepackage{microtype}
\usepackage{subfigure}
\usepackage{booktabs} % for professional tables
\usepackage{tikz}

\newtheorem{assumption}{Assumption}

\newcommand{\R}{\mathbb{R}}
\newcommand{\N}{\mathbb{N}}
\newcommand{\Z}{\mathbb{Z}}

\newcommand{\mF}{\mathcal{F}}
\newcommand{\mG}{\mathcal{G}}

\newcommand{\mB}{\mathcal{B}}

\newcommand{\mI}{\mathcal{I}}
\newcommand{\mJ}{\mathcal{J}}
\newcommand{\mH}{\mathcal{H}}
\newcommand{\mL}{\mathcal{L}}
\newcommand{\mS}{\mathbb{S}}

\newcommand{\mM}{\mathcal{M}}
\newcommand{\mN}{\mathcal{N}}
\newcommand{\mV}{\mathcal{V}}

\newcommand{\mR}{\mathcal{R}}
\newcommand{\Ep}{\mathbb{E}}
\renewcommand{\Pr}{\mathrm{P}}

\renewcommand{\hat}{\widehat}
\renewcommand{\tilde}{\widetilde}
\renewcommand{\check}{\widecheck}
\newcommand{\mone}{\textit{\textbf{1}}}
\newcommand{\argmin}{\operatornamewithlimits{argmin}}

\newcommand{\vect}{\operatornamewithlimits{vec}}

\usepackage{color}

\newcommand{\bc}{\color{black}}

\allowdisplaybreaks

\jmlrheading{1}{2000}{1-48}{4/00}{10/00}{meila00a}{Marina Meil\u{a} and Michael I. Jordan}
\ShortHeadings{Advantage of Deep Neural Networks with Singularity on Hypersurfaces}{Imaizumi and Fukumizu}
\firstpageno{1}

\usepackage{lastpage}

\begin{document}:

\jmlrheading{23}{2022}{1-\pageref{LastPage}}{5/21; Revised
11/21}{2/22}{21-0542}{Masaaki Imaizumi and Kenji Fukumizu}
\ShortHeadings{Advantage of Deep Neural Networks with Singularity on Hypersurfaces}{Imaizumi and Fukumizu}

\title{Advantage of Deep Neural Networks for Estimating \\Functions with Singularity on Hypersurfaces}

\author{\name Masaaki Imaizumi \email imaizumi@g.ecc.u-tokyo.ac.jp \\
       \addr The University of Tokyo\\
       Meguro, Tokyo, Japan
       \AND
       \name Kenji Fukumizu \email fukumizu@ism.ac.jp \\
       \addr The Institute of Statistical Mathematics\\
       Tachikawa, Tokyo, Japan}

\editor{Samory Kpotufe}

\maketitle

\begin{abstract}
We develop a minimax rate analysis to describe the reason that deep neural networks (DNNs) perform better than other standard methods.
For nonparametric regression problems, it is well known that many standard methods attain the minimax optimal rate of estimation errors for smooth functions, and thus, it is not straightforward to identify the theoretical advantages of DNNs. 
This study tries to fill this gap by considering the estimation for a class of non-smooth functions that have singularities on hypersurfaces. 
Our findings are as follows: (i) We derive the generalization error of a DNN estimator and prove that its convergence rate is almost optimal. (ii) We elucidate a phase diagram of estimation problems, which describes the situations where the DNNs outperform a general class of estimators, including kernel methods, Gaussian process methods, and others.
We additionally show that DNNs outperform harmonic analysis based estimators.
This advantage of DNNs comes from the fact that a shape of singularity can be successfully handled by their multi-layered structure.
\end{abstract}


\section{Introduction} \label{submission}

Learning with deep neural networks (DNNs) has been applied extensively owing to their remarkable performance in many tasks. 
It has been observed that DNNs empirically achieve substantially higher accuracy in some tasks than many existing approaches \citep{schmidhuber2015deep, lecun2015deep,hinton2006fast, le2011optimization, kingma14adam}.
To understand such empirical successes, we investigate DNNs through nonparametric regression problems.

Suppose that we have $n$ independently and identically distributed (i.i.d.) pairs $(X_i,Y_i) \in [0,1]^D \times \R$ for $ i=1,...,n$ generated from the model
\begin{align}
    Y_i = f^*(X_i) + \xi_i, ~i=1,..,n, \label{eq:reg}
\end{align}
where $f^*:[0,1]^D \to \R$ is an unknown function and $(\xi_i)_{i=1}^n$ is an i.i.d.~noise that is independent of $(X_i)_{i=1}^n$.  For simplicity, we assume $\xi_i$ is Gaussian. 
The aim of this study is to investigate the generalization error of the maximum likelihood estimator $\hat{f}^{DL}$ given by DNNs; that is, we analyze $\|\hat{f}^{DL} - f^*\|_{L^2(P_X)}^2 := \Ep_{X \sim P_X}[ (\hat{f}^{DL}(X) - f^*(X))^2 ]$.

This paper argues that deep learning has an advantage over other standard models in terms of the generalization error when $f^*$ has  singularities on a hypersurface in the domain. It will be also shown that the the existence of the advantage exhibits a phase transition in the phase space consisting of the {\em shape of the singularities}, which is defined by the smoothness of the hypersurfaces and the smoothness of the target function $f^*$. 
These results are based on the analysis of the minimax optimal rate of the generalization error.

Understanding the advantages of DNNs remains challenging when analyzing deep learning with the nonparametric regression problem.
This is due to the well-known fact that some popular existing methods achieve the minimax optimal rate with a \textit{smoothness assumption} for the regression model \eqref{eq:reg}.
Namely, when $f^*$ is $\beta$-times (continuously) differentiable, the standard knowledge in the nonparametric statistics  \citep{tsybakov2003introduction,wasserman2006all} tells that a variety of existing methods, such as the kernel method, the Fourier series method, and Gaussian process methods, provide an estimator $\tilde{f}$, which satisfies
\begin{align*}
    \|\tilde{f}-f^*\|_{L^2}^2 = O_\Pr \left( n^{-2\beta/(2\beta + D)} \right).
\end{align*}
Since this convergence rate is known to be optimal in the minimax sense \citep{stone1982optimal}, it is almost impossible to present theoretical evidence for the empirical advantage of DNNs if the smoothness assumption is satisfied.

To overcome this limitation of theoretical understanding, we consider the estimation of \textit{a piecewise smooth function}, which is a natural class of non-smooth functions.
More specifically, the functions are singular (non-differentiable or discontinuous) only on smooth hypersurfaces in their multi-dimensional domain.
Figure \ref{fig:p-smooth} presents an example of such a function, the domain of which is divided into three pieces with piecewise smooth boundaries.
This class is flexible enough to express functions of singularities, while being broader than the usual spaces of smooth functions.
As being suitable for representing edges in images, similar function classes have been studied in the areas of image analysis and harmonic analysis \citep{korostelev2012minimax,candes2004new,candes2002recovering,kutyniok2012shearlets}.

\begin{figure}
\begin{center}
\includegraphics[width=0.7\hsize]{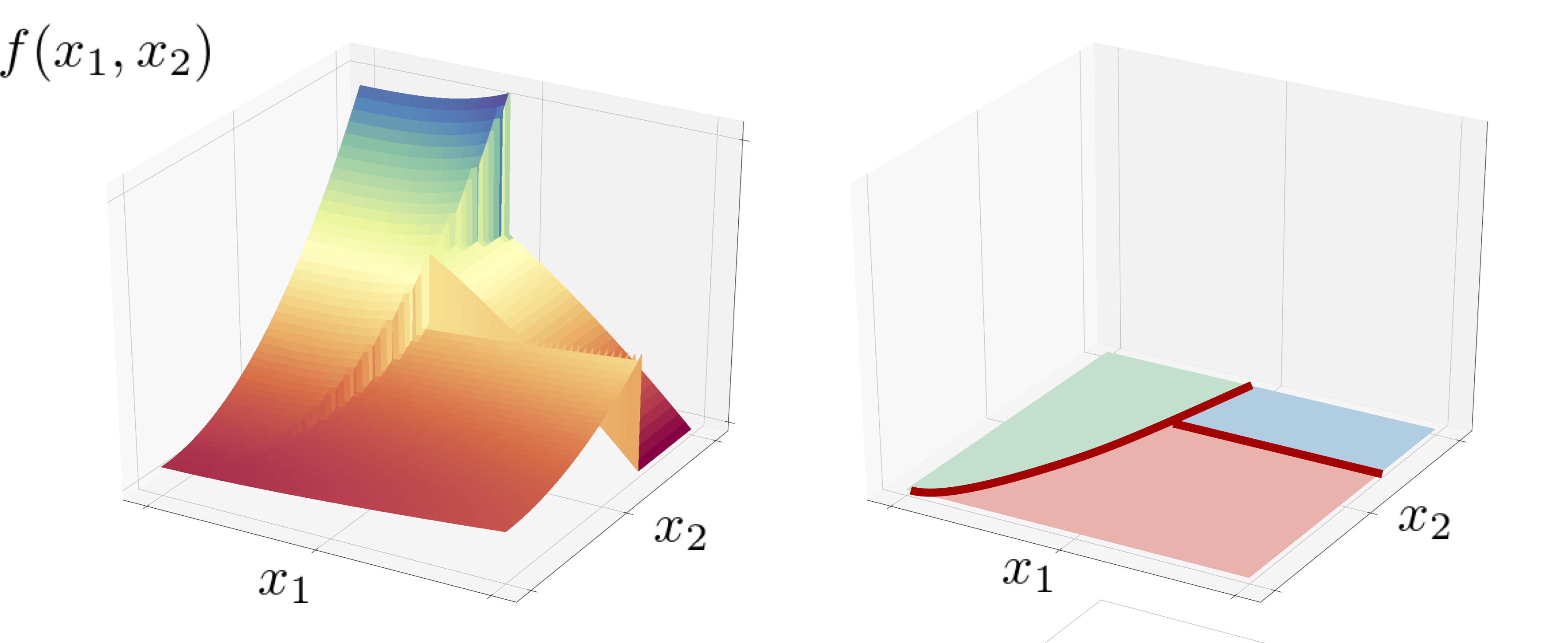}
\caption{[Left] Example of piecewise smooth function $f(x_1,x_2)$ with two-dimensional input. The function is smooth within the pieces and singular on the boundaries of the pieces. [Right] Domain of example function, divided into three pieces. The bold red curve, which is a boundary of the pieces, is a hypersurface for the singularity of $f$. \label{fig:p-smooth}}
\end{center}
\end{figure}

This study makes the following two contributions.
The first one is to prove that the estimator by DNNs almost achieves the minimax optimal rate when estimating functions of singularities.
Specifically, let ${\mF}_{\alpha,\beta,M}^{PS}$ be the set of piecewise smooth functions such that their domain is divided into $M$ pieces, they are $\beta$-times differentiable except on the boundaries of the pieces, and the boundary is piecewise $\alpha$-times differentiable 
(Its rigorous definition will be provided in Section \ref{sec:non-smooth_func}).
We prove that the least-square estimator $\hat{f}^{DL}$ by DNNs for $f^* \in {\mF}_{\alpha,\beta,M}^{PS}$ satisfies
\begin{align}
    \Ep \left[ \|\hat{f}^{DL} - f^*\|_{L^2(P_X)}^2 \right] = \tilde{O} \left(\max\left\{ n^{-2\beta/(2\beta + D)} , n^{-\alpha/(\alpha + D-1)} \right\} \right), \label{rate:main}
\end{align}
as $n \to \infty$ (Corollary \ref{cor:non-bayes}).
Here, $\tilde{O}$ denotes the Big O notation ignoring logarithmic factors.
It is interesting to note that this result holds even if a DNN does not contain any non-smooth elements, such as non-differentiable activation functions. 
That is, even smooth DNNs can estimate such a non-smooth function without being affected by singularities.
We also evaluate the effect of the number of pieces in the domain.

\begin{figure}[t]
    \centering
    \includegraphics[width=0.6\hsize]{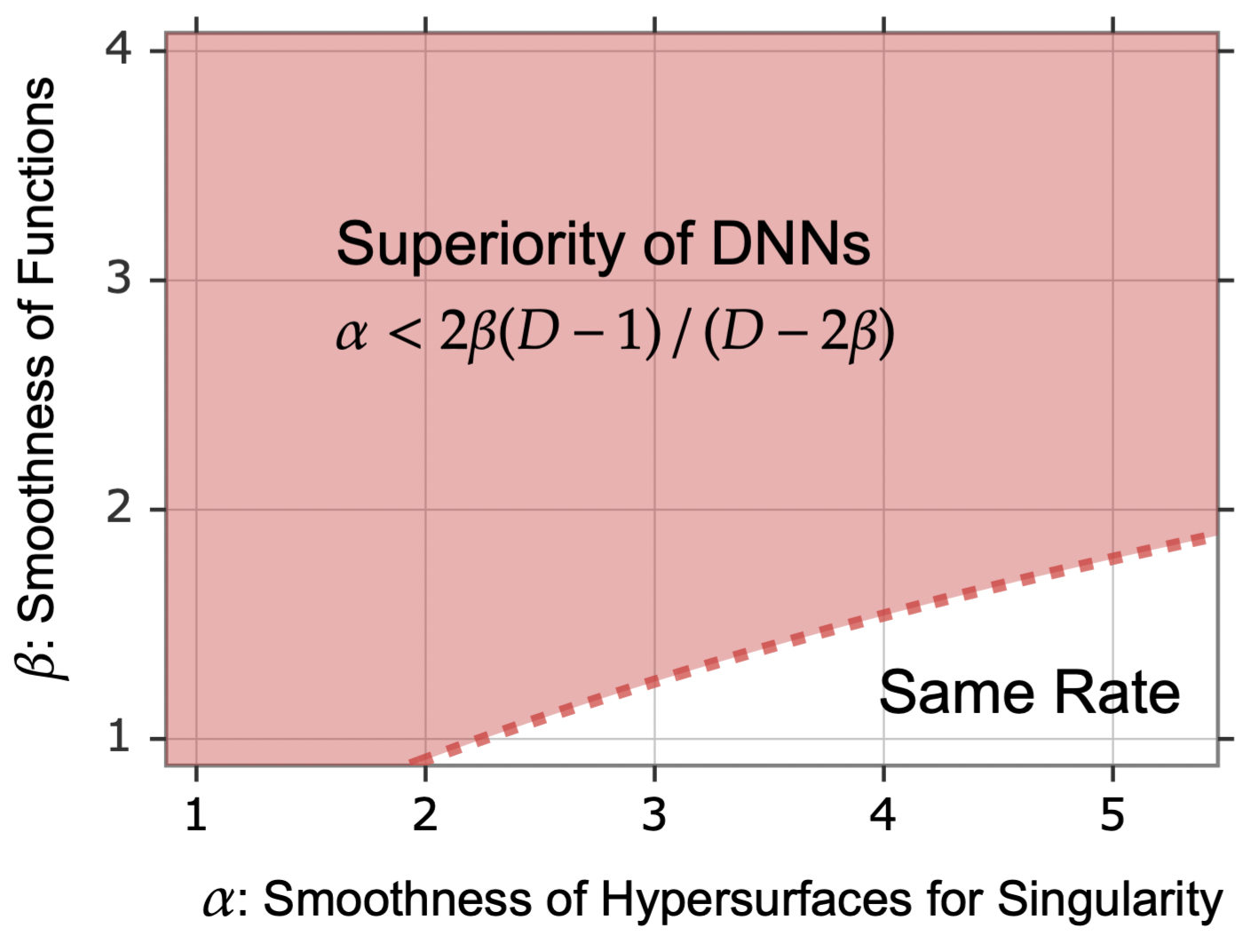}
    \caption{The phase diagram of the parameters $\alpha$ (diffentiability of the hypersurface for singularity) and $\beta$ (differentiability of the function inside the pieces). The red dashed line shows $\alpha = 2 \beta (D-1)/(D-2\beta)$ with $D=5$. If a parameter configuration $(\alpha,\beta)$ is above the line, the minimax rate of DNNs is better than that of the linear estimator.
    Otherwise, the DNN and the linear estimator have the same minimax rate.}
    \label{fig:phase}
\end{figure}
As the second contribution, we develop a phase diagram of parameters that explains the superiority of DNNs over a certain class of nonparametric methods.
{\bc 
Specifically, we consider a class of existing nonparametric estimators known as linear estimators $\hat{f}^{\mathrm{lin}}$, which includes commonly used methods such as estimators by kernel ridge regression, spline regression, and Gaussian process regression, and then derive a lower bound of its error rate as $  n^{-\alpha/(2\alpha + D-1)}$ (Proposition \ref{prop:linear}).
Combined with the result that the estimator by DNNs $\hat{f}^{DL}$ has the rate in \eqref{rate:main}, we can state that $\hat{f}^{DL}$ has a faster convergence when $-\alpha/(2\alpha + D - 1) > \max\{ {-2\beta/(2\beta + D)}, -\alpha/(\alpha + D - 1) \}$.
As a result, when $\alpha < 2\beta(D-1)/(D-2\beta)$ holds, 
the estimator by DNNs is superior to linear estimators.}
Figure \ref{fig:phase} illustrates the configuration of the parameters of smoothness which retains the theoretical advantage of DNNs, that is, it describes a set of parameters such that 
\begin{align*}
     \inf_{\hat{f}^{\mathrm{lin}}}\sup_{f^* \in {\mF}_{\alpha,\beta,M}^{PS}} \Ep_{f^*}\left[\| \hat{f}^{\mathrm{lin}} - f^* \|_{L^2}\right] \gnsim \inf_{\hat{f}^{\mathrm{DL}}}\sup_{f^* \in {\mF}_{\alpha,\beta,M}^{PS}} \Ep_{f^*}\left[\|  \hat{f}^{DL} - f^*\|_{L^2}\right]
\end{align*}
holds.
Here, we denote $a_n \gnsim b_n$ by $|a_n / b_n| \to \infty$ holds as $n \to \infty$ with the sequences $\{a_n\}_n$ and $\{b_n\}_n$.
The phase transition shows that the superiority of DNNs becomes apparent when $f^*$ has a hypersurface of singularities that has a less smooth shape, i.e., the boundaries of the pieces are complicated.
Otherwise, DNNs and the linear estimator achieve the same minimax rate.
Additionally, we also consider other candidates used under the existence of singularities in the field of image analysis, such as the wavelet $\hat{f}^{\mathrm{wav}}$ and curvelet $\hat{f}^{\mathrm{curve}}$ methods, and then derive another phase diagram about the advantage of DNNs.
These results indicate that DNNs certainly offer a theoretical advantage over the other methods under functions of singularities.

As intuitive reasons for these results, we discuss the following two roles of DNNs.
First, a model by DNNs, which is a composition of several transforms, is suitable for decomposing non-smooth functions into simple elements.
Let us begin with considering a simple example of the indicator function $\mone_{S^{D-1}}:\R^D \to \R$ of the unit sphere $S^{D-1} \subset \R^D$; that is, $\mone_{S^{D-1}}(x) = 1$ when $ x \in S^{D-1}$, and $\mone_{S^{D-1}}(x) = 0$ otherwise.
Although $\mone_{S^{D-1}}(x)$ is discontinuous, DNNs can approximate $\mone_{S^{D-1}}(x)$ without a loss of efficiency from the discontinuity.
Note that the function has the form $\mone_{S^{D-1}}(x) = \mone_{\{\cdot \geq 0\}} \circ h(x)$ with a certain smooth function $h(x)$ and a step function $\mone_{\{\cdot \geq 0\}}$, such as $\mone_{\{\cdot \geq 0\}}(x) = 1$ for $x \geq 0$ and $\mone_{\{\cdot \geq 0\}}(x) = 0$ otherwise.
The multi-transform structure of DNNs plays an important role: a first transform approximates $h(x)$, and a second one  $\mone_{\{\cdot \geq 0\}}$, following which the entire DNN model consists of a composite function of the two transforms.
Owing to the composition, DNNs can approximate $\mone_{S^{D-1}}$ as if it were a smooth function.
Second, the activation function of DNNs plays a significant role.
Several common activation functions, such as the sigmoid and rectified linear unit (ReLU) activations, can easily approximate a step function with an arbitrarily small error.
The general conditions for achieving the approximation are explained in Assumption \ref{asmp:activation} in Section \ref{sec:DNN}.

This paper is an extension of a conference proceeding \cite{imaizumi2018deep}.
The contributions in the current paper differ from the results of the conference paper in the following ways.
First, this study handles the general class of activation functions of DNNs, while the proceeding investigates  the Rectified Linear Unit (ReLU) activation function.
Hence, the proof of the first contribution with singularities is significantly different.
Especially, this study develops a new proof for resolving singularities by possibly smooth activation functions, while the proof of the proceeding for singularity relies on the non-smoothness of the ReLU activation.
{\bc 
Second, this paper develops the phase diagram by newly derived lower bounds on the minimax rate for the other estimators.
Our new derivation of a specific lower bound on the rate allows us to describe such a phase transition, while \cite{imaizumi2018deep} does not derive the lower rate.
}
Further, this study additionally examines several methods that are adept at handling singularities, such as the wavelet and curvelet approaches, and subsequently demonstrates that they do not achieve optimality.

It is important to mention that the superiority of DNNs shown in our results does not depend on the property of learning algorithms, but only on an expressive power and a degree of freedom of DNNs. 
Toward a very practical aspect of DNN, it is certainly important to investigate the error from learning algorithms.
However, the environment surrounding practical DNNs is highly complicated that it is not very meaningful to consider all of these factors at the same time.
For a more rigorous analysis and correct understanding, we follow the line of research in nonparametric regression, such as \citet{schmidt2017nonparametric} and \cite{bauer2019deep}, we focus on the advantage of DNNs on the expressive power and the degree of freedom of DNNs with singularities of data.

\subsection{Related Studies}
Several pioneering works have investigated deep learning in terms of the nonparametric regression problems.
A recent study \cite{schmidt2017nonparametric} considered the case in which $f^*$ is expressed as a composition of several smooth functions, and then derived the minimax optimal rate with this setting.
\cite{bauer2019deep} also derived the convergence rate of errors when $f^*$ had the form of a generalized hierarchical interaction model, and revealed that the obtained rate was dependent on the lower dimensionality of the model.
These studies focused on the composition structure of $f^*$, and they did not consider the non-smoothness or discontinuity of $f^*$.
These works thus did not take into account singularities, which is the main focus of our study.
We also mention that the comparison with the Harmonic and linear estimators is our unique result.

\cite{suzuki2018adaptivity} and \cite{hayakawa2020minimax} also demonstrated the superiority of DNNs.
These works investigated the generalization error of DNNs when $f^*$ belongs to the Besov space.
Interestingly, the convergence rate of DNNs was faster than that of the linear estimator when the norm parameter of the Besov space was less than $1$, following the theory of  \cite{donoho1998minimax}.
Although their motivation is similar to ours, the Besov space is not suitable for representing functions of singularities on a smooth hypersurface.
This is because the wavelet decomposition, which is used to define the Besov space, loses its efficiency for handling the hypersurfaces for singularity, as explained in Section \ref{sec:compare}.
Note also that the comparison in \cite{suzuki2018adaptivity} and \cite {hayakawa2020minimax} between DNNs and linear estimators was motivated by the proceeding version of this paper.

The current work has been technically inspired by \citet{petersen2017optimal}, which investigated the approximation power of DNNs with discontinuity.
The main difference between the above paper and current work is the focus on the advantage of deep learning. 
Their study mainly investigated the approximation error of DNNs.
Thus, a comparison with existing methods was not the main purpose.
Another major difference is that their study focused on the approximation power, whereas we investigate the generalization error, including the variance control of DNNs.
A further difference is that our study investigates a broader class of discontinuous functions, as our definition directly controls hypersurfaces in the domain, whereas \cite{petersen2017optimal} defined the discontinuity by a transform of the Heaviside function.

\subsection{Paper Organization}
The remainder of this paper is organized as follows. 
Section \ref{sec:DNN} introduces a functional model by DNNs, following which an estimator for the regression problem with DNNs is defined.
The notion of functions with singularities is explained in Section \ref{sec:non-smooth_func}.
Section \ref{sec:rate} derives the convergence rate of the estimator by means of DNNs.
Furthermore, the minimax optimality of the convergence rate is derived.
Section \ref{sec:compare} presents the non-optimal convergence rate obtained by a certain class of other estimators, and compares this rate with that of DNNs.
Section \ref{sec:concl} summarizes our work.
Full proofs are deferred to the supplementary material.

\subsection{Notation}

Let $I := [0,1]$ be the unit interval, $\N$ be natural numbers, and $\N_0 := \N \cup \{0\}$.
For $z \in \N$, $[z] := \{1,2,\ldots,z\}$ is the set of natural numbers that are no more than $z$.
The $d$-th element of vector $b \in \R^D$ is denoted by $b_{d}$, and $b_{-d} = (b_1,...,b_{d-1},b_{d+1},...,b_D)$ for $d \in [D]$.
$\|b\|_q := (\sum_d b_d^q)^{1/q}$ is the $q$-norm for $q \in (0,\infty)$, $\|b\|_\infty := \max_{j \in [J]} |b_j|$, and $\|b\|_0 := \sum_{j \in [J]} \mone_{\{b_j \neq 0\}}$.
For a measure space $(A,\mathcal{B},\mu)$ and a measurable function $f:A \to \R$, let $\|f\|_{L^2(\mu)} := ( \int_A |f(x)|^2d\mu(x))^{1/2}$ denote the $L^2(\mu)$-norm if the integral is finite.
When $\mu$ is the Lebesgue measure on a measurable set $A$ in $\R^D$, we omit $\mu$ and simply write $\|f\|_{L^2(A)}$.
For a set $A \subset \R^D$, $\mathrm{vol}(A)$ denotes the Lebesgue measure of $A$. The tensor product is denoted by $\otimes$.
For a set $R\subset I^D$, let $\mone_{R}:I^D \to \{0,1\}$ denote the indicator function of $R$; that is, $\mone_{R}(x) = 1$ if $x \in R$, and $\mone_{R}(x) = 0$ otherwise.
For the sequences $\{a_n\}_n$ and $\{b_n\}_n$, $a_n \lesssim b_n$ means that there exists $C>0$ such that $a_n \leq  C b_n$ holds for every $n \in \N$.
$a_n \gtrsim b_n$ denotes the opposite of $a_n \lesssim b_n$.
Furthermore, $a_n \asymp b_n$ denotes both $a_n \gtrsim b_n$ and $a_n \lesssim b_n$.
$a_n \lnsim b_n$ denotes $|a_n/b_n| \to 0$ as $n \to \infty$, and $a_n \gnsim b_n$ is its opposite.
For a set of parameters $\theta$, $C_\theta > 0$ denotes an existing finite constant depending on $\theta$.
Let $O_\Pr$ and $o_\Pr$ be the Landau big O and small o in probability.
$\tilde{O}$ ignores every multiplicative polynomial of logarithmic factors.

\section{Deep Neural Networks} \label{sec:DNN}

A deep neural network (DNN) is a model of functions defined by a layered structure.  
Let $L \in \N$ be the number of layers in DNNs, and 
for $\ell \in [L+1]$, let $D_\ell \in \N$ be the dimensionality of variables in the $\ell$-th layer. 
DNNs have a matrix parameter $A_\ell \in \R^{D_{\ell + 1} \times D_\ell}$ and a vector parameter $b_\ell \in \R^{D_\ell}$ for $\ell \in [L]$ to represent weights and biases, respectively.
We introduce an activation function $\eta:\R \to \R$, which will be specified later.
For a vector input $z \in \R^{d}$, $\eta(x) = (\eta(z_1),...,\eta(z_d))^\top$ denotes an element-wise operation.
For $\ell \in [L-1]$, with an input vector $z \in \R^{D_{\ell}}$, we define $g_{\ell}:\R^{D_{\ell}} \to \R^{D_{\ell + 1}}$ as $g_{\ell}(z) = \eta(A_\ell z + b_\ell)$.
We also define $g_L(z) = A_L z + b_L$ with $z \in \R^{D_L}$.
Thereafter, we define a function $g : \R^{D_1} \to \R^{D_{L+1}}$ of DNNs with $(A_1,b_1),...,(A_{L},b_{L})$ by 
\begin{align}
    g(x) = g_L \circ g_{L-1} \circ \cdots \circ g_1(x). \label{def:DNN}
\end{align}
Intuitively, $g(x)$ is constituted by compositions of $L$ maps.

For each $g$ with the form \eqref{def:DNN}, we introduce several operators to extract information of $g$.
Let $L(g)=L$ be the number of layers, $S(g) := \sum_{\ell \in [L]} \|\vect (A_\ell)\|_0 + \|b_\ell\|_0$ as a number of non-zero elements in the parameter matrix and tensor in $g$, and $B(g) := \max_{\ell \in [L]} \|\vect (A_\ell)\|_\infty \vee \max_{\ell \in [L]} \|b_\ell\|_\infty$ be the largest absolute value of the parameters.
Here, $\vect(\cdot)$ is a vectorization operator for matrices.

We define the set of functions of DNNs.
With a tuple $(L', S', B') \in \N^3$, we write it as
\begin{align*}
    &\mG(L',S',B'):= \Bigl\{ g\in L^\infty(I^D)\mid g \mbox{~as~\eqref{def:DNN}}: L(g) \leq L', S(g) \leq S', B(g) \leq B', \|g\|_{L^\infty(I^D)} \leq F  \Bigr\},
\end{align*}
where $F > 0$ is a threshold.
Since the form of DNNs is flexible, we control the size and complexity of it through the layers and parameters through the tuple $(L',S',B')$.
Here, the internal dimensionality $D_\ell$ is implicitly regularized by the tuple.

\begin{figure}[htbp]
\begin{center}
\includegraphics[width=0.19\hsize]{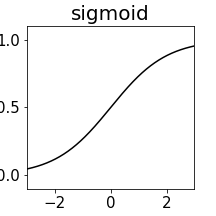}
\includegraphics[width=0.19\hsize]{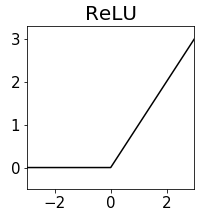}
\includegraphics[width=0.19\hsize]{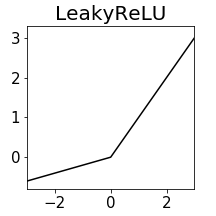}
\includegraphics[width=0.19\hsize]{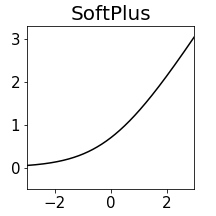}
\includegraphics[width=0.19\hsize]{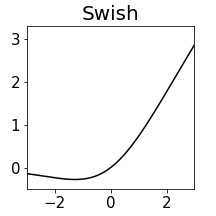}
\caption{Common activation functions; Sigmoid: $\eta(x) = 1/(1+\exp(x))$, ReLU: $\eta(x)=\max\{x,0\}$, LeakyReLU: $\eta(x)=\max\{x,0\} + 0.2 \min\{x,0\}$, SoftPlus: $\eta(x) = \log(1+\exp(x))$, and Swish: $\eta(x)=x/(1+\exp(x))$. \label{fig:act}}
\end{center}
\end{figure}

The explicit form of $\eta$ plays a critical role in DNNs, and numerous variations of activation functions have been suggested.
We select several of the most representative activation functions in Figure \ref{fig:act}.
To investigate a wide class of activation functions, we introduce the following assumption.
\begin{assumption} \label{asmp:activation}
    An activation function $\eta$ satisfies either of the following conditions:
    \begin{itemize}
        \item [(i)] For $N \in \N$, there exists $q \geq  1$ and $k \in \{0,1\}$ such that $\partial^{j}\eta$ exists at every point and is bounded for $j=1,...,N+1$.
        Furthermore, there exists $x' \in \R$ such that $\min_{j=1,...,N} |\partial^j \eta(x')| \geq c_\eta > 0$ with a constant $c_\eta > 0$, and the followings hold:
        \begin{align*}
            |\eta(x) - \overline{c}x^k| = O(1/x^q), ~ (x \to \infty), \mbox{~~and~~}|\eta (x) - \underline{c}| = O(1/|x|^q), ~ (x \to -\infty), 
        \end{align*}        
        with some constants $\overline{c} > \underline{c} \geq 0$.
        There also exists $C_K > 0$ such as $|\eta(x)| \leq C_K (1+|x|^k)$ for any $x \in \R$.
        \item [(ii)] There exist constants $c_1 > c_2 \geq 0$ such that
        \begin{align*}
                \eta(x) = 
            \begin{cases}
                c_1 x& \mbox{~if~}x \geq 0 \\
                c_2 x& \mbox{~if~} x < 0.
            \end{cases}
        \end{align*}
    \end{itemize}
\end{assumption}
The condition $(i)$ describes smooth activation functions such as the sigmoid function, the softplus function, and the Swish function with $N=\infty$.
The lower bound on $|\partial^j \eta(x')|$ describes a non-vanishing property of the derivatives.
The condition $(ii)$ indicates a piecewise linear function such as the rectified linear unit (ReLU) function and the leaky ReLU function, which require another technique to investigate.

\subsection{Regression Problem and Estimator by DNNs}

We consider the least square estimator by DNNs for the regression problem \eqref{eq:reg}.
Let the $D$-dimensional cube $I^D$ ($D \geq 2$) be a space for input variables $X_i$.
Suppose we have a set of observations $(X_i,Y_i) \in  I^D \times \R$ for $i \in [n]$ which is independently and identically distributed with the data generating process \eqref{eq:reg} where $f^* :I^D \to \R$ is an unknown true function and $\xi_i$ is  Gaussian noise with mean $0$ and variance $\sigma^2 > 0$ for $i \in [n]$.
We also suppose that $X_i$ follows a marginal distribution $P_X$ on $I^D$ and it has a density function $p_X$ which is bounded away from zero and infinity.
Then, we define an estimator by empirical risk minimization with DNNs as
\begin{align}
    \hat{f}^{DL} \in \argmin_{ f \in \mG(L,S,B)} \frac{1}{n}\sum_{i =1}^n (Y_i - {f}(X_i))^2. \label{opt:erm}
\end{align}
The minimizer always exists since $\mG(L,S,B)$ is a compact set in $L^\infty(I^D)$ due to the parameter bound and continuity of $\eta$.
Note that we do not discuss the optimization issues from the non-convexity of the loss function, since we mainly focus on an estimation aspect.

\section{Characterization of Functions with Singularity} \label{sec:non-smooth_func}

In this section, we provide a rigorous formulation of functions with singularity on smooth hypersurfaces.
To describe the singularity of functions, we introduce a notion of \textit{piecewise smooth functions}, which have its domain divided into several pieces and smooth only within each of the pieces.
Furthermore, piecewise smooth functions are singular (non-differentiable or discontinuous) on the boundaries of the pieces.

\textbf{Smooth Functions (H\"older space)}:
Let $\Omega$ be a closed subset of $\R^D$ and $\beta ,F> 0$ be parameters.
For a multi-index $a \in \N_0^D$, $\partial^a=\partial^{a_1}\cdots \partial^{a_D}$ denotes a partial derivative operator.
The H\"older space $H^\beta({\Omega})$ is defined as a set of functions $f:{\Omega}\to\R$ such as
\begin{align*}
    &H^\beta(\Omega) := \Biggl\{ f  \left|\, \max_{\|a\|_1 \leq \lfloor \beta \rfloor}\|\partial^a f\|_{L^\infty(\Omega)} + \max_{\|a\|_1 = \lfloor \beta \rfloor} \sup_{x,x' \in \Omega, x \neq x'} \frac{|\partial^a f(x) - \partial^a f(x')|}{|x-x'|^{\beta - \lfloor \beta \rfloor}} < \infty \right.\Biggr\}.
\end{align*}
Also, let $H_F^\beta(\Omega)$ be a ball in $H^\beta(\Omega)$ with its radius $F>0$ in terms of the norm $\|\cdot\|_{L^\infty}$.

\textbf{Pieces in the Domain}:
We describe pieces as subsets of the domain $I^D$ by dividing $I^D$ with several hypersurfaces.
For $j=1,...,J$, let $h_j \in H^\alpha_F(I^{D-1})$ be a function with input $x_{-d_j} \in I^{D-1}$ for some $d_j \in [D]$.
We define a family of $M$ pieces which are an intersection of one side of $J$ hypersurfaces.
Let $I^{+}_j:=\{x\in I^D \mid x_{d_j} \geq  h_j (x_{-{d_j}})\}$ and $I^{-}_j:=\{x\in I^D \mid x_{d_j} \leq  h_j (x_{-{d_j}})\}$.
For a $J$-tuple $t=(t_j)_{j=1}^J \in \{+,-\}^{[J]}$, a unit piece of $I^D$ is defined by  $I_t  := \bigcap_{j\in[J]} I_j^{t_j}$.
Let $\mathcal{T}$ be a subset of $\{+,-\}^{[J]}$, and define {\em piece} $R_\mathcal{T}$ by 
\[
  R_\mathcal{T} :=   \bigcup_{ t \in \mathcal{T}} I_t.
\]
Let $\{\mathcal{T}_1,\ldots,\mathcal{T}_M\}$ be a partition of $\{+,-\}^{[J]}$.  
Then, it is easy to see $\bigcup_{m\in[M]}R_{\mathcal{T}_m} = I^D$ and $R_{\mathcal{T}_m}\cap R_{\mathcal{T}_{m'}}$ is of Lebesgue measure zero. 
The family of pieces that is of size $M$ is given by 
\begin{align*}
    \mR_{\alpha,M} := \left\{ \{R_{\mathcal{T}_m}\}_{m\in[M]} \mid \{\mathcal{T}_m\}_{m\in[M]} \text{ is a partition of }  \{+,-\}^{[J]}\right\}.
\end{align*}
Intuitively, $\{R_{\mathcal{T}_m} \}_{m \in [M]}\in \mR_{\alpha,M}$ is a partition of $I^D$, allowing overlap on the piecewise $\alpha$-smooth boundaries. 
Figure \ref{fig:p-smooth} presents an example.
In the following, when the partition $\{\mathcal{T}_m\}_{m\in[M]}$ is fixed, we can also write $R_m := R_{\mathcal{T}_m}$ by slightly changing the notation.

\textbf{Piecewise Smooth Functions}:
By using $H_F^\beta(I^D)$ and $\mR_{\alpha,M}$, we introduce a space of piecewise smooth functions as 
\begin{align*}
    &\mF_{\alpha,\beta,M}^{PS} := \left\{ \sum_{m \in [M]} f_m \otimes \mone_{R_m} : f_m \in H_F^{\beta}(I^D), \{R_m\}_{m \in [M]} \in \mR_{\alpha,M} \right\}.
\end{align*}
Since $f_m(x)$ realizes only when $x \in R_m$, the notion of $\mF_{\alpha,\beta,M}^{PS}$ can express a combination of smooth functions on each piece $R_m$.
Hence, functions in $\mF_{\alpha,\beta,M}^{PS} $ are non-smooth (and even discontinuous) on boundaries of $R_m$.
Obviously, $H^\beta(I^D) \subset \mF_{\alpha,\beta,M}^{PS}$ holds for any $M$ and $\alpha$, since $f_1=f_2=\cdots=f_M$ makes the function globally smooth.

\begin{remark}[Similar definitions]
    Several studies  \citep{petersen2017optimal,imaizumi2018deep} also define a class of piecewise smooth functions.
    There are mainly two differences of our definition in this study.
    First, our definition can describe a wider class of pieces.
    The definition utilizes a direct definition of a smooth hypersurface function $h$, while the other definition defines pieces by a transformation of the Heaviside function which is slightly restrictive.
    Second, our definition is less redundant.
    We do not allow the pieces to overlap with one another, whereas some of the other definitions allow pieces to overlap.
    When overlap exists, it may make approximation and estimation errors worse, which is a problem that our definition can avoid.
\end{remark}

\begin{remark}[Comparison with a boundary by \cite{mammen1995asymptotical}]
    We compare our formulation with $h_j \in H_\alpha^\beta(I^{D-1})$ with another way of making boundaries obtained by continuous transformation of a sphere by \cite{mammen1995asymptotical}. 
    In our approach, each $h_j$ can only represent a singularity in one fixed dimensional direction. In contrast, the sphere-based boundary has the flexibility to create boundaries in various dimensional directions at once. 
    However, since our formulation can provide several $I_j^+$ and $I_j^-$ in different dimensional directions, we can reproduce the sphere-based boundary if we combine multiple $I_j^+$ and $I_j^-$ by taking their intersections as the definition of $I_t$.
    We adopt the current formulation because this way of decomposing the sphere-based singularity into its parts is more general. 
    Even if we adopt the sphere-based boundary, we can obtain the same result by using the approximation on step functions in Lemma \ref{lem:main_step}.
\end{remark}

\section{Generalization Error of Deep Neural Networks} \label{sec:rate}

We provide theoretical results regarding DNN performances for estimating piecewise smooth functions.
To begin with, we decompose the estimator error into an approximation error and a complexity error, analogously to the bias-variance decomposition.
{\bc 
By a simple calculation on \eqref{opt:erm}, we obtain the following inequality:
\begin{align}
    &\|\hat{f}^{DL} - f^*\|_n^2 \leq \underbrace{\inf_{f' \in \mG(L,S,B)}\|f^* - f'\|_n^2}_{\mB} + \underbrace{\frac{2}{n}\sum_{i=1}^n \xi_i (\hat{f}^{DL}(X_i) - f(X_i))}_{\mV}, \label{ineq:main_basic}
\end{align}
with some $f \in \mG(L,S,B)$ which will specified later.
}
Here, $\|f\|_n^2 := n^{-1}\sum_{i \in [n]}f(X_i)^2$ is an empirical (pseudo) norm.
We note that $\xi_i$ is the Gaussian noise displayed in \eqref{eq:reg}.
The term $\mB$ in the right hand side is the approximation error, and the term $\mV$ is the complexity error.
In the following section, we bound $\mB$ and subsequently combine it with the bound for $\mV$.

\subsection{Approximation Result}
We evaluate the approximation error with piecewise smooth functions according to the following three preparatory steps: approximating (i) smooth functions, (ii) step functions, and (iii) indicator functions on the pieces.
Thereafter, we provide a theorem for the approximation of piecewise smooth functions.

As the first step, we state the approximation power of DNNs for smooth functions in the H\"older space.
Although this topic has been studied extensively \citep{mhaskar1996neural,yarotsky2017error}, we provide a formal statement because a condition on activation functions is slightly different.
\begin{lemma}[Smooth function approximation] \label{lem:main_smooth}
    Let $\beta>0$ be a constant. Suppose Assumption \ref{asmp:activation} holds with $N > \beta$.
    Then, there exist constants $C_{\beta,D,F},C_{\beta,D,F,q}>0$ such that a tuple $(L,S,B)$ such as $L \geq C_{\beta,D,F} ( \lfloor  \beta \rfloor + \log_2(1/ \varepsilon) + 1)$, $ S \geq C_{\beta,D,F} \varepsilon^{-D/\beta} (  \log_2 (1/\varepsilon))^2$, and $B \geq C_{\beta,D,F,q} \varepsilon^{-C_\beta}$,
    which satisfies
    \begin{align*}
        &\inf_{g \in \mG(L,S,B)} \sup_{f \in H_F^\beta(I^D)} \|g - f\|_{L^2(R)}\leq  \mathrm{vol}(R)\varepsilon,
    \end{align*}
    for any non-empty measurable set $R \subset I^D$ and $\varepsilon > 0$.
\end{lemma}
As we reform the result, the approximation error is written as ${O}(\mathrm{vol}(R)S^{-\beta/D})$ up to logarithmic factors, with $L$ and $B$ satisfying the conditions.

As the second step, we investigate an approximation for a step function $\mone_{\{\cdot \geq 0\}}$, which will play an important role in handling singularities of functions.
\begin{lemma}[Step function approximation] \label{lem:main_step}
    Suppose $\eta$ satisfies Assumption \ref{asmp:activation}.
    Then, for any $\varepsilon \in (0,1)$ and $T>0$, we obtain
    \begin{align*}
        \inf_{g \in \mG(2,6,C_{T,q}\varepsilon^{-8})} \|g - \mone_{\{\cdot \geq 0\}}\|_{L^2([-T,T])} \leq \varepsilon.
    \end{align*}
\end{lemma}
The result states that any activation functions satisfying Assumption \ref{asmp:activation} can approximate indicator functions.
Importantly, DNNs can achieve an arbitrary error $\varepsilon$ with $O(1)$ parameters. 
The approximation with this constant number of parameters is very important in obtaining the desired rate, since the generalization error of DNNs is significantly influenced by the number of parameters.

As the third step, we investigate the approximation error for the indicator function of a piece.
For approximation by DNNs, we reform the indicator function into a composition of a step function and a smooth function:
\begin{align*}
    \mone_{\{x_d \lesseqgtr h_j(x_{-d})\}} = \mone_{\{\cdot \geq 0\}} \circ (x \mapsto \mp (x_d - h_j(x_{-d}) ) ), ~ x \in I^D.
\end{align*}
Using this formulation, we obtain the following result:
\begin{lemma}[Indicator functions for $\mR_{\alpha,J}$] \label{lem:main_indicator}
    Suppose that Assumption \ref{asmp:activation} holds with $N > \alpha$.
    Then, there exist constants $C_{\alpha,D,F,J},C_{F,J,q} > 0$ such that for any $\{R_m\}_{m \in [M]} \in \mR_{\alpha,M}$ and  $\varepsilon>0$ we can find a function $f=(f_1,...,f_M)^\top \in \mG(L , S, B)$ such as $L \geq C_{\alpha,D,F,J} (\lfloor \alpha \rfloor  + \log_2(1/\varepsilon))$, $S \geq C_{\alpha,D,F,J} (J\varepsilon^{-2(D-1)/\alpha} (  \log_2(1/  \varepsilon))^2 + M(  \log_2(1/  \varepsilon))^2 )$,  and $B \geq C_{F,J,q}\varepsilon^{-C_\alpha}$, 
    which satisfies 
    \begin{align*}
        \|\mone_{R_m} - f_m\|_{L^2(I)} \leq \varepsilon,~ \forall m \in [M].
    \end{align*}
\end{lemma}
Since the model of DNNs has a composition structure, DNNs can implicitly decompose the indicator function into a step function and a smooth function, which derives the convergence rate.
Importantly, even though the function $\mone_{\{x_d \lesseqgtr h_j(x_{-d})\}}$ with $h_j \in H_F^\alpha(I^{D-1})$ has singularity on a set $\{x \in I^D \mid x_d = h_j(x_{-d})\}$, DNNs can achieve a fast approximation rate as if the boundary is a smooth function in $H_F^\alpha(I^{D-1})$.

Based on the above steps, we derive an approximation rate of DNNs for a piecewise smooth function $f \in \mF_{\alpha,\beta,M}^{PS}$:
\begin{theorem}[Approximation Error] \label{thm:approx}
    Suppose Assumption \ref{asmp:activation} holds with $N > \alpha \vee \beta$.
    Then, there exist constants $C_{\alpha,\beta,D,F}, C_{\alpha,\beta,D,F,J}, C_{F,M,q}>0$ such that there exists a tuple $(L,S,B)$ such as $L \geq C_{\alpha,\beta,D,F}(\lfloor \alpha \rfloor + \lfloor \beta \rfloor + \log_2( 1/\varepsilon_1) + \log_2( M/\varepsilon_2)) + 1)$, $S \geq C_{\alpha,\beta,D,F,J} ( M  \varepsilon_1^{-D/\beta} ( \log_2( 1/\varepsilon_1))^2 +   (\varepsilon_2 / M)^{-2(D-1)/\alpha} (\log_2(M/\varepsilon_2))^2)$, and $B \geq C_{F,M,q}(\varepsilon_1 \wedge \varepsilon_2)^{-16 \wedge -C_{\alpha,\beta}}$,
    which satisfies
    \begin{align*}
        \inf_{f \in \mG(L,S,B)} \sup_{f^* \in \mF_{\alpha,\beta,M}^{PS}} \|f - f^*\|_{L^2(I^D)} \leq \varepsilon_1 + \varepsilon_2,
    \end{align*}
     for any $\varepsilon_1, \varepsilon_2 \in (0,1)$.
\end{theorem}
The result states that the approximation error contains two main terms.
A simple calculation yields that the error is reformulated as ${O}(S^{-\beta/D} + S^{-\alpha/2(D-1)})$ up to logarithmic factors, with $L$ and $B$ satisfying the conditions.
The first rate ${O}(S^{-\beta/D})$ describes approximation for $f_m \in H^{\beta}_F(I^D)$, and the second rate $O(S^{-\alpha/2(D-1)})$ is for $\mone_{R_m}$.

\subsection{Generalization Result}
We evaluate a generalization error of DNNs, based on the decomposition \eqref{ineq:main_basic}, associated with the bound on $\mB$.
To evaluate the remained term $\mV$, we utilize the celebrated theory of the local Rademacher complexity \citep{bartlett1998sample,koltchinskii2006local}.
Then, we obtain one of our main results as follows.
$\Ep_{f^*}[\cdot]$ denotes the expectation with respect to the true distribution of $(X,Y)$.
\begin{theorem}[Generalization Error] \label{thm:non-bayes}
    Suppose $f^* \in\mF_{\alpha,\beta,M}^{PS}$ and Assumption \ref{asmp:activation} holds with $N > \alpha \vee \beta$.
    Then, there exists a sufficiently large $F$ and a tuple $(S,B,L)$ satisfying $L \geq C_{\alpha,\beta,D,F}(1 + \lfloor\alpha \rfloor + \lfloor \beta \rfloor + \log_2 (n/M))$, $S = C_{\alpha,\beta,D,F,J} (M n^{D/(2\beta + D)}  +  n^{(D-1)/(\alpha + D-1)}) \log^2 n$, and $B \geq C_{F,M,q} n^{C_{\alpha,\beta,D}}$,
    such that there exist $C = C_{\sigma,\alpha,\beta,D,F,J,P_X} > 0 $ and $c_1 > 0$ which satisfy
    \begin{align*}
        &\Ep_{f^*} \left[\|\hat{f}^{DL} -f^*\|_{L^2(P_X)}^2  \right]\leq C M (  n^{-2\beta/(2\beta + D)} +  n^{-\alpha/(\alpha + D - 1)})\log^2n  + \frac{C_{ \sigma, F} \sqrt{\log n}}{n}.
    \end{align*}
\end{theorem}

The dominant term appears in the first term in the right hand, thus it mainly describes the error bound for $\hat{f}^{DL}$.
We note that the main term of the squared error increases linearly in the number of pieces $M$.
To simplify the order of the bound, we provide the following:
\begin{corollary} \label{cor:non-bayes}
    With the settings in Theorem \ref{thm:non-bayes}, we obtain
    \begin{align*}
        \Ep_{f^*} \left[\|\hat{f}^{DL} -f^*\|_{L^2(P_X)}^2\right] = \tilde{O}\left(\max\{n^{-2\beta/(2\beta + D)}, n^{-\alpha/(\alpha + D - 1)}\}\right).
    \end{align*}
\end{corollary}
The order is interpreted as follows.
The first term $n^{-2\beta/(2\beta + D)}$ describes an effect of estimating $f_m \in H_F^\beta(I^D)$ for $m \in [M]$.
The rate corresponds to the minimax optimal convergence rate of generalization errors for estimating smooth functions in $H_F^\beta(I^D)$ (for a summary, see \citet{tsybakov2003introduction}).
The second term $n^{-\alpha / (\alpha + D-1)}$ reveals an effect from estimation of $\mone_{R_m}$ for $m \in [M]$ through estimating the boundaries of $R_m \in \mR_{\alpha,J}$.
The same rate of convergence appears in a problem for estimating sets with smooth boundaries \citep{mammen1995asymptotical}.
Based on the result, we state that DNNs can divide a piecewise smooth function into its various smooth functions and indicators, and estimate them by parts. Thus, the overall convergence rate is the sum of the rates of the parts.
We also note that we consider a sufficiently large $F$ increasing in $n$ such as $F = \Theta(\log n)$, whose effect is asymptotically negligible.

\begin{remark}[Smoothness of $\eta$]
    It is worth noting that the rate in Theorem \ref{thm:non-bayes} holds regardless of smoothness of the activation function $\eta$, because Assumption \ref{asmp:activation} allows both smooth and non-smooth activation functions.
    That is, even when $\hat{f}^{DL}$ by DNNs is a smooth function with smooth activation, we can obtain the rate in Corollary \ref{cor:non-bayes} with non-smooth $f^*$.
\end{remark}

We can consider the error from optimization independently from the statistical generalization.
The following proposition provides the statement.
\begin{proposition}[Effect of Optimization]\label{cor:opt}
    If a learning algorithm outputs $\hat{f}^{\mathrm{Algo}} \in \mG(L,S,B)$ such as
    \begin{align*}
         n^{-1}\sum_{i \in [n]}(Y_i - \hat{f}^{\mathrm{Algo}}(X_i))^2 - \inf_{f \in \mG(L,S,B)}n^{-1}\sum_{i \in [n]}(Y_i - f(X_i))^2 \leq \Delta,
    \end{align*}
    with an existing constant $\Delta > 0$, then the following holds:
\begin{align*}
     &\Ep_{f^*} \left[ \|\hat{f}^{\mathrm{Algo}}-f^*\|_{L^2(P_X)}^2 \right]\leq \Tilde{O}\left( \max\{n^{-2\beta/(2\beta + D)}, n^{-\alpha/(\alpha + D - 1)}\}\right) + \Delta.
\end{align*}
\end{proposition}
We can evaluate the generalization error, including the optimization effect, by combining several results on the magnitude of $\Delta$ (for example, \cite{kawaguchi2016deep} and \cite{allenzhu}).

\subsection{Minimax Lower Bound of Generalization Error}
We investigate the efficiency of the convergence rate in Corollary \ref{cor:non-bayes}.
To this end, we consider the minimax generalization error for a functional class $\mF$ such as
\begin{align*}
    \mR (\mF) := \inf_{\Bar{f} } \sup_{f^* \in {\mF}} \Ep_{f^*} \left[\|\Bar{f} - f^* \|_{L^2(I^D)}^2 \right],
\end{align*}
where $\Bar{f} = \Bar{f}(X_1,...,X_n,Y_1,...,Y_n)$ is taken from all possible estimators depending on the observations.
In this section, we derive a lower bound of the generalization error of DNNs, and then prove that it corresponds to the rate Theorem \ref{thm:non-bayes} up to logarithmic factors.
By the result, we can claim that the estimation by DNNs is (almost) optimal in the minimax sense.

We introduce several new notations.
For set $A$ equipped with a norm $\| \cdot\|$, let $\mN(\varepsilon, A, \| \cdot\|)$ be the covering number of $A$ in terms of $\|\cdot\|$, and  $\mM(\varepsilon, A, \| \cdot\|)$ be the packing number of $A$, respectively.
For sequences $\{a_n\}_n$ and $\{b_n\}_n$, we write $a_n = \Omega (b_n)$ for $\limsup_{n \to \infty} |a_n/b_n| > 0$.
Also, $a_n = \Theta(b_n)$ means that both of $a_n = O(b_n)$ and $a_n = \Omega(b_n)$ hold.

To derive the lower bound, we apply the following information theoretic result by \cite{yang1999information}:
\begin{theorem}[Theorem 6 in \cite{yang1999information}] \label{thm:yangbarron}
    Let ${\mF}$ be a set of functions, and $\varepsilon_n$ be a sequence such that $\varepsilon_n^2 = \mM(\varepsilon_n, {\mF}, \| \cdot\|_{L^2})/n$ holds.
    Then, we obtain
    \begin{align*}
         \mR (\mF)  = \Theta(\varepsilon_n^2).
    \end{align*}
\end{theorem}
Since the minimax rate for $\mF_{\alpha,\beta,M}^{PS}$ is bounded below by that of its subset, we will find a suitable subset of $\mF_{\alpha,\beta,M}^{PS}$ and measure its packing number.
In the rest of this section, we will take the following two steps.
First, we define a subset of $\mF_{\alpha,\beta,M}^{PS}$ by introducing a notion of basic pieces.
Second, we measure a packing number of the subset of $\mF_{\alpha,\beta,M}^{PS}$.

As the first step, we define a \textit{basic piece indicator}, which is a set of piecewise functions whose piece is an embedding of $D$-dimensional balls, and then define a certain subset of $\mF_{\alpha,\beta,M}^{PS}$.
As a preparation, let $\mS^{D-1} := \{x \in \R^D : \|x\|_2=1\}$ is the $D-1$ dimensional sphere, and let $(V_j,F_j)_{j=1}^\ell$ be its coordinate system with some $\ell$ as a $C^\infty$-differentiable manifold such that $F_j: V_j\to \mathring{B}^{D-1}:=\{x\in \R^{D-1}\mid \|x\|<1\}$ is a diffeomorphism.
Compactness of the domain guarantees that we can find finite $\ell$.
A function $f:\mS^{D-1}\to\R$ is said to be in the H\"older class $H^\alpha(\mS^{D-1})$ with $\alpha>0$ if $f\circ F_j^{-1}$ is in $H^\alpha(\mathring{B}^{D-1})$. 

\textbf{Basic Piece Indicator}:
A subset $R\subset I^D$ is called an {\em $\alpha$-basic piece}, if it satisfies two conditions: (i) there is a continuous embedding $g:\{x\in \R^D\mid \|x\|\leq 1\}\to \R^D$ such that its restriction to the boundary $\mS^{D-1}$ is in $H^\alpha(\mS^{D-1})$  and $R=I^D\cap \mathrm{Image}(g)$, (ii) there exist $d \in [D]$ and  $h\in H_F^\alpha(I^{D-1})$ such that the indicator function of $R$ is given by the graph
\begin{align*}
    \mone_{R}(x) = \Psi_d(x_1,\ldots,x_{d-1}, x_d +  h(x_{-d}),x_{d+1},...,x_D), x \in I^D,
\end{align*}
where $\Psi_d(x) := \mone_{\{x_d \geq 0\}}$ is the Heaviside function. 
{\bc 
Then, we define a \textit{basic piece indicator} as an indicator function $\mone_R: I^D \to \R$ with an $\alpha$-basic piece $R$.
We also define a set of {basic piece indicators} as
\begin{align*}
    \mI_\alpha := \left\{ \mone_R \mid R \mbox{~is~an~$\alpha$-basic~piece}\right\}.
\end{align*}
}
The condition (i) tells that a basic piece belongs to the \textit{boundary fragment class} which is developed by \cite{dudley1974metric} and \cite{mammen1999smooth}.
The condition (ii) means $R$ is a set defined by a horizon function discussed in \cite{petersen2017optimal}. 

In the following, we consider a functional class 
\begin{align*}
    H^\beta(I^D) \otimes \mI_\alpha := \left\{ f \otimes \mone_R \mid f \in H^\beta(I^D), \mone_R \in\mI_\alpha \right\},
\end{align*}
and use this set as a key subset for the minimax lower bound.
Obviously, $H^\beta(I^D) \otimes \mI_\alpha \subset \mF_{\alpha,\beta,M}^{PS}$ holds for any $M \in \N$.

At the second step of this section, we measure a packing number of $ H^\beta(I^D) \otimes \mI_\alpha$. 
\begin{proposition}[Packing Bound] \label{prop:packing_subset}
    For any $D \geq 2$, we have
    \begin{align*}
        \log \mM(\varepsilon, H^\beta(I^D) \otimes \mI_\alpha, \|\cdot\|_{L^2(I^D)}) = \Theta \left(\varepsilon^{-D/\beta} + \varepsilon^{-2(D-1)/\alpha}\right),~(\varepsilon \to 0).
    \end{align*}
\end{proposition}

With the bound, we apply Theorem \ref{thm:yangbarron} and thus obtain the minimax lower bound of the estimation with $H^\beta(I^D) \otimes \mI_\alpha$.
By using the relation $\mR(\mF_{\alpha,\beta,M}^{PS}) \geq \mR(H^\beta(I^D) \otimes \mI_\alpha)$, we obtain the following theorem.
\begin{theorem}[Minimax Rate for $\mathcal{\mF}_{\alpha,\beta,M}^{PS}$]
\label{thm:minimax}
    For any $\alpha,\beta \geq 1$ and $M \in \N$ we obtain
    \begin{align*}
        \mR(\mF_{\alpha,\beta,M}^{PS}) = \Omega\left( \max\left\{n^{-2\beta/(2\beta + D)}, n^{-\alpha / (\alpha + D-1)}\right\}\right), ~(n \to \infty).
    \end{align*}
\end{theorem}
This result indicates that the convergence rate by $\hat{f}^{DL}$ is almost optimal in the minimax sense, since the rates in Theorem \ref{thm:non-bayes} correspond to the lower bound of Theorem \ref{thm:minimax} up to a log factor.
In other words, for estimating $f^* \in \mF_{\alpha,\beta,M}^{PS}$, no other methods could achieve a better rate than the estimators by DNNs.

\subsection{Singularity Control by Deep Neural Network} \label{sec:explain_dnn}
In this section, we present an intuition for the optimality of DNNs for the functions with singularities.
In the following, we will present that the approximation error of a non-smooth function by DNNs is as if the function to be approximated is smooth.

We consider an example with $D=2$ and $M=2$.
Let $f^S \in \mF_{\alpha,\beta,2}^{PS}$ be
\begin{align*}
    f^S(x_1,x_2) = \mone_{R}(x_1,x_2), ~ R=\{(x_1,x_2) \in I^2 \mid x_2 \geq h(x_1)\},
\end{align*}
with a function $h \in H_1^\alpha(I)$.
The function $f^S$ is singular on the set $\{(x_1,x_2) \in I^2 \mid x_2 = h(x_1)\}$.
Moreover, the function $f^S$ is rewritten as
\begin{align*}
    f^S(x_1,x_2) = \mone_{\{\cdot \geq 0\}} \circ \underbrace{((x_1,x_2) \mapsto (x_2 - h(x_1)))}_{=: f^H(x_1,x_2)},
\end{align*}
where $\mone_{\{\cdot \geq 0\}}$ is the step function and $f^H  \in H^\alpha(I^2)$ is a smooth function induced by $h$.
We can rewrite the function $f^S$ with the singularities as a composition of the step function and the smooth function $f^H$.

To approximate and estimate $f^S$, we consider an explicit function $g_E $ by DNNs as $g_E(x_1,x_2) = g_s \circ g_H (x_1,x_2)$, where $g_s$ is a DNN approximator for the step function by a DNN, and $g_H$ is a DNN approximator for $f^H$.
Subsequently, we can measure its approximation as
\begin{align*}
    \|f^E - g_E\|_{L^2(I^2)} \lesssim \|\mone_{\{\cdot \geq 0\}} - g_s\|_{L^2(I)} + \|f^H - g_H \|_{L^2(I^2)}.
\end{align*}
For the right hand side, Lemma \ref{lem:main_step} indicates that the first term $\|\mone_{\{\cdot \geq 0\}} - g_s\|_{L^2(I)}$ is negligible, because it is arbitrary small with a constant number of parameters.
Hence, a dominant error appears in the second term $\|f^H - g_H \|_{L^2(I^2)}$, which is an approximation error of a smooth function $f^H \in H^\beta(I^2)$.

In summary, DNNs can approximate and estimate non-smooth $f^E$, as if $f^E$ is a smooth function in $H^\beta(I^2)$.
This is because DNNs can represent a composition of functions, which can eliminate the singularity of $f^E$.

\section{Advantages of DNNs with Singularity}\label{sec:compare}

In this section, we compare the result of DNNs with several other methods.

\subsection{Sub-optimality of Linear Estimators}

We discuss sub-optimality of some of other standard methods in estimating piecewise smooth functions.
To this end, we consider a class of \textit{linear estimators}.
\begin{definition}[Linear Estimator]
The class contains any estimators written as
\begin{align}
    \hat{f}^{\mathrm{lin}} (x) = \sum_{i \in [n]} \Upsilon_i(x; X_1,...,X_n) Y_i, \label{def:lin}
\end{align}
where $\Upsilon_i$ is an arbitrary measurable function which depends on $X_1,...,X_n$.
\end{definition}
Linear estimators include various popular estimators such as kernel ridge regression, sieve regression, spline regression, and Gaussian process regression.
If a regression model is a linear sum of (not necessarily orthogonal) basis functions and its parameter is a minimizer of the sum of square losses, the model is a linear estimator.
Linear estimators have been studied extensively (e.g. \cite{donoho1998minimax} and \cite{korostelev2012minimax}, particularly Section 6 in \cite{korostelev2012minimax}), and the results can be adapted to our setting, thereby providing the following results.

In the following, we prove the sub-optimality of linear estimators:
\begin{proposition}[Sub-Optimality of Linear Estimators] \label{prop:linear}
    For any $\alpha,\beta \geq 1$ and $D \in \N$, we obtain
    \begin{align*}
        \sup_{f^* \in \mF_{\alpha,\beta,M}^{PS}}\Ep_{f^*}\left[\|\hat{f}^{\mathrm{lin}} - f^*\|_{L^2(P_X)}^2\right] = \Omega \left(    n^{-\alpha / (2\alpha + D-1)} \right). 
    \end{align*}
\end{proposition}
This rate is slower than the minimax optimal rate $O \left(     n^{-2\beta / (2\beta + D)} \vee n^{-\alpha / (\alpha + D-1)} \right)$ in Theorem \ref{thm:minimax} with the parameter configuration $ \alpha < 2\beta (D-1) / (D-2\beta) $.
This result implies that linear estimators perform worse than DNNs when it is relatively difficult to estimate a hypersurface with singularity.

\begin{remark}[Linear Estimators with Cross-Validation]
    We discuss a sub-optimality of a certain class of nonlinear estimators associated with cross-validation (CV).
    If we select the hyperparameters of a linear estimator by CV, the estimator is no longer a linear estimator. 
    Typically, a bandwidth parameter of kernel methods and a number of basis functions for sieve estimators are examples of hyperparameters.
    Despite this fact, the generalization error of a CV-supported estimator is bounded below by that of the estimators with optimal hyperparameters, which are sometimes referred to as oracle estimators.
    Because oracle estimators are linear estimators in most cases, we can still claim that the CV-supported estimators are sub-optimal according to Proposition \ref{prop:linear}.
\end{remark}

{\bc 
\begin{remark}[Relation to Other Non-Optimality of Linear Estimators] \label{remark:relation_non-optimal}
    The non-optimality of the linear estimator has been shown in several studies: \cite{donoho1998minimax} showed the non-optimality of linear estimators in estimation for an element of a Besov space with a specific order.
    Proof by \cite{donoho1998minimax} is based on the fact that a linear estimator cannot adapt to the optimal rate on a non-convex set of functions, and shows the non-convexity of a unit ball in the Besov space makes linear estimators sub-optimal.
    
    In contrast, our result on the non-optimality is based on a different approach, which utilizes the shape of a hypersurface on which the singularity of $f^*$ is placed.
    A linear estimator uses a level-set of its basis function $\Upsilon_i$ to approximate the hypersurface. 
    However, due to the limited expressive power and flexibility of the utilization of the level-set, linear estimators cannot adapt to the higher-order smoothness of the hypersurface and lose their optimality. 
    This proof is an application of the image analysis by \cite{korostelev2012minimax} to our setting with the basic pieces.
\end{remark}
}

\subsection{Sub-Optimality of Wavelet Estimator}

We investigate the generalization error of a wavelet series estimator for $f^* \in \mF_{\alpha,\beta,M}^{PS}$.
The estimator by wavelets is one of the most common estimators for the nonparametric regression problem, and can attain the optimal rate in many settings (for a summary, see \cite{gine2015mathematical}).
Moreover, wavelets can handle discontinuous functions.
{\bc 
Since this estimator is a linear estimator, the analysis with wavelets is strictly included in Proposition \ref{prop:linear}.
However, we provide the results here, because it is useful for readers to understand how basis functions handle singularities through wavelets and other basis functions.}
In this section, we will prove the sub-optimality of the wavelet for the class of piecewise smooth functions.
Since these methods are well known for dealing with singularities, their investigation can provide a better understanding of why linear estimators lose their optimality.

Intuitively, wavelets resolve singularity by fitting the support of their basis functions to the singularity's shape.
That is, their approximation error decreases, when the supports fit the shape and fewer basis functions overlap the singularity hypersurface.
As illustrated in the left panel of Figure \ref{fig:support_let}, the wavelet divides $I^D$ into cubes, with each basis is concentrated on each cube.

\begin{figure}[htbp]
    \centering
    \includegraphics[width=0.6\hsize]{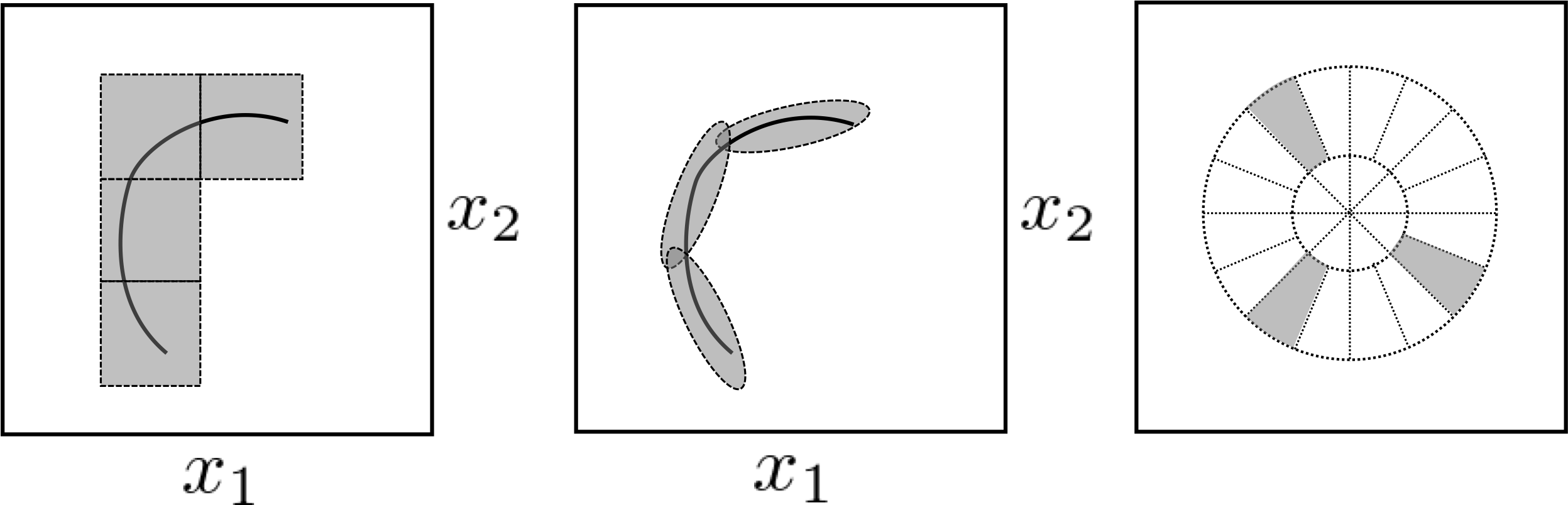}
    \caption{[Left] Hypersurface with singularity (black curve) in $I^2$ and cubes in which wavelet basis functions are concentrated.
[Middle] Hypersurface with singularity (black curve) in $I^2$ and ellipses in which curvelet basis functions are concentrated.
[Right] Domain on the frequency of the curvelet basis. Each basis function is concentrated on a colored shape.}
    \label{fig:support_let}
\end{figure}

To this end, we consider an orthogonal wavelet basis for $L^2(I)$.
We define a set of indexes $\mH := \{(j,k): j \in \{-1,0\} \cup \N, k \in K_j \}$ be the index set with a set $K_j= \{0\} \cup [2^j-1]$ and consider the wavelet basis $\{\phi_\kappa\}_{\kappa \in \mH}$ of $L^2(\R)$ with setting $\phi_{-1,k}$ be a shifted scaling function.
Then, we restrict its domain to $I$ and obtain the wavelet basis for $L^2(I)$.

For a wavelet analysis for multivariate functions, we consider a tensor product of the basis $\{\phi_\kappa\}_{\kappa \in \mH}$.
Namely, let us define $\Phi_{\kappa_1,...,\kappa_D}(x) := \prod_{d \in [D]} \phi_{\kappa_d}(x_d)$ with $x \in \R^D$.
Then, consider a orthonormal basis $\{\Phi_{\kappa_1,...,\kappa_D}\}_{(\kappa_1,...,\kappa_D) \in \mH^{\times D}}$ for $L^2(\R^D)$, then restrict it to $L^2(I^D)$.
$\mH^{\times D}$ is a $D$-times direct product of $\mH$. 
Then, a decomposition of a restricted function $f \in L^2(I^D)$ is formulated as
\begin{align*}
    f = \sum_{(\kappa_1,...,\kappa_D) \in \mH^{\times D}} w_{\kappa_1,...,\kappa_D}(f) \Phi_{\kappa_1,...,\kappa_D},
\end{align*}
where $w_{\kappa_1,...,\kappa_D}(f) = \langle f, \Phi_{\kappa_1,...,\kappa_D} \rangle$.

We define an estimator of $f^*$ by the wavelet decomposition.
For simplicity, let  $P_X$ be the uniform distribution on $I^D$.
Using a truncation parameter $\tau \in \N$, we set $\mH_\tau := \{(j,k): j \in \{-1,0\} \cup [\tau], k \in K_j\} \subset \mH$ be a subset of indexes.
Since the decomposition is a linear sum of orthogonal basis, an wavelet estimator $\hat{f}^{\mathrm{wav}}$ which minimizes an empirical squared loss has the following form
\begin{align*}
    \hat{f}^{\mathrm{wav}} = \sum_{(\kappa_1,...,\kappa_D) \in \mH_\tau^{\times D}} \hat{w}_{\kappa_1,...,\kappa_D} \Phi_{\kappa_1,...,\kappa_D}.
\end{align*}
Moreover, $\hat{w}_{\kappa_1,...,\kappa_D}$ is an empirical analogue version of the inner product as
%\begin{align*}
    $\hat{w}_{\kappa_1,...,\kappa_D} = \frac{1}{n}\sum_{i \in [n]} Y_i \Phi_{\kappa_1,...,\kappa_D}(X_i)$.
%\end{align*}
{\bc 
$\tau$ is selected to follow an order in $n$ which balances a bias and variance of the error and minimizes its generalization error.
}
With the wavelet estimator, we obtain the following result:
\begin{proposition}[Sub-Optimality of Wavelets] \label{prop:wavelet}
    For any $\alpha,\beta \geq 1$ and $D \geq 2$, we obtain
    \begin{align*}
        \sup_{f^* \in \mF_{\alpha,\beta,M}^{PS}}\Ep_{f^*}\left[\|\hat{f}^{\mathrm{wav}} - f^*\|_{L^2(P_X)}^2\right] =\Omega \left(  n^{-1/2} \right). 
    \end{align*}
\end{proposition}
When comparing the derived rate with Theorem \ref{thm:minimax}, we can find that the wavelet estimator cannot attain the minimax rate when both $\beta > D/2$ and $\alpha > (D-1)$ hold.
This result describes wavelets cannot adapt a higher smoothness of $f^*$ and the boundary of pieces.

{\bc 
Note that the truncation parameter $\tau$ in this estimator depends on unknown data distribution, and it should be chosen in a data-dependent way in practice.
Even in this case, the same lower bound in Proposition \ref{prop:wavelet} holds by the data-dependent choice, because the current data-independent choice is the optimal choice that minimizes the error. 

\cite{schmidt2017nonparametric} shows that the wavelet estimator is sub-optimal in estimation of composite functions.
Although our wavelet estimator is identical to that of \cite{schmidt2017nonparametric}, the proof of the non-optimality is different. 
This is because our proof utilizes the difficulty of approximating the singularity of the true function, but the composite function by \cite{schmidt2017nonparametric} has no singularity.
This argument is similar to Remark \ref{remark:relation_non-optimal}.
}

\subsection{Sub-optimality of Harmonic Based Estimator}
We investigate another estimator from the harmonic analysis and its optimality.
The harmonic analysis provides several methods for non-smooth structures, such as curvelets \citep{candes2002recovering,candes2004new} and shearlets \citep{kutyniok2011compactly}.
The methods are designed to approximate piecewise smooth functions on pieces with $C^2$ boundaries.
As illustrated in Figure \ref{fig:support_let}, each base of the curvelet is concentrated on an ellipse with different scales, locations, and angles in $I^2$.
The ellipses covering the hypersurface resolve the singularity.
Each basis has a fan-shaped support with a different radius and angle in the frequency domain (see \cite{candes2004new} for details).

We focus on curvelets as one of the most common methods.
For brevity, we study the case with $D=2$, which is of primary concern for curvelets.
Curvelets can be extended to higher dimensions $D\geq 3$, and we can study the case in a similar manner.

As preparation, we define curvelets.
In this analysis, we consider the domain of $f^*$ is $[-1,1]^D$ for technical simplification.
Furthermore, we set $P_X$ as the uniform distribution on $[-1,1]^D$, as similar to the wavelet case.
Let $\mu := (j,\ell,k)$ be a tuple of scale index $j = 0,1,2,...$,  rotation index $\ell = 0,1,2,...,2^j$, and location parameter $k \in \Z^2$.
For each $\mu$, we define the parabolic scaling matrix
%\begin{align*}
    $D_j = 
    \begin{pmatrix}
        2^{2j}&0 \\ 0 & 2^j
    \end{pmatrix}$,
%\end{align*}
the rotation angle $\theta_{j,\ell} = 2\pi 2^{-j} \ell$, and a location $k = k_\delta = (k_1\delta_1,k_2\delta_2)$ with hyper-parameters $\delta_1$ and $\delta_2$, where $\delta_1,\delta_2>0$.
Thereafter, we consider a curvelet $\gamma_\mu:\R^2 \to \R$ as
%\begin{align*}
    $\gamma_\mu(x) = 2^{3j/2} \gamma (D_j R_{\theta_{j,\ell}}x - k_\delta)$.
%\end{align*}
In this case, $\gamma$ is defined by an inverse Fourier transform of a localized function in the Fourier domain.
Figure \ref{fig:support_let} provides an illustration of its support and its rigorous definition is deferred to the appendix. 
According to \citep{candes2004new}, it is shown that $\{\gamma_\mu\}_{\mu \in \mL}$ is 
%an orthonormal basis 
a tight frame in $L^2(\R^2)$, where $\mL$ is a set of $\mu$.
Hence, we obtain the following formulation $f(x) = \sum_{\mu \in \mL} w_\mu(f) \gamma_\mu(x)$, where $w_\mu(f) = \langle f, \gamma_\mu \rangle$.
For estimation, we consider the truncation parameter $\tau \in \N$ and define an index subset as $\mL_\tau := \{\mu \mid j  \in [\tau]\} \subset \mL$.
Thereafter, similar to the wavelet case, a curvelet estimator which minimizes an empirical squared loss is written as
\begin{align*}
    \hat{f}^{\mathrm{curve}}(x) = \sum_{\mu \in \mL_\tau} \hat{w}_\mu \gamma_\mu(x), \mbox{~where~}\hat{w}_\mu = \frac{1}{n} \sum_{i \in [n]} Y_i \gamma_\mu(X_i).
\end{align*}
In this case, $\tau$ is selected to minimize the generalization error of the curvelet estimator.
We then obtain the following statement.
\begin{proposition}[Sub-Optimality of Curvelets] \label{prop:curvelet}
    For $D=2$ and any $\alpha,\beta \geq 2$, we obtain 
    \begin{align*}
        \sup_{f^* \in \mF_{\alpha,\beta,M}}\Ep_{f^*}\left[\|\hat{f}^{\mathrm{curve}} - f^*\|_{L^2(P_X)}^2\right] = \Omega\left( n^{-1/3}\right). 
    \end{align*}
\end{proposition}
The result implies the sub-optimality of the curvelet; that is, the rate is slower than the minimax rate when $\beta > D/4 = 1/2$ and $\alpha > (D-1)/2 = 1/2$.
Similar to the wavelet estimator, the curvelet estimator does not adapt to the higher smoothness in the nonparametric regression setting.

\subsection{Advantage of DNNs against the Other Methods}

We summarize the sub-optimality of the other methods and compare them with DNNs.
To this end, we provide a formal statement for comparing the estimator $\hat{f}^{DL}$ by DNNs and the other estimators, namely, $\hat{f}^{\mathrm{lin}}, \hat{f}^{\mathrm{wac}}$, and $\hat{f}^{\mathrm{curve}}$.
The following corollary is the second main result of this study:
\begin{corollary}[Advantage of DNNs] \label{cor:compare}
    Fix $M \geq 2$.
   If $ \alpha < 2\beta (D-1) / D $ holds with $\alpha,\beta \geq 1$ and $D\geq 2$, the estimator $\widecheck{f} = \hat{f}^{\mathrm{lin}}$ satisfies
    \begin{align}
        \sup_{f^* \in \mF_{\alpha,\beta,M}^{PS}} \Ep_{f^*} \left[ \|\widecheck{f} - f^*\|_{L^2(P_X)}^2\right] \gnsim \sup_{f^* \in \mF_{\alpha,\beta,M}^{PS}} \Ep_{f^*} \left[ \|\hat{f}^{ DL} - f^*\|_{L^2(P_X)}^2\right]. \label{ineq:cor2}
    \end{align}

    If $\beta > D/2$ and $\alpha > D-1$ hold, \eqref{ineq:cor2} holds with $\widecheck{f} = \hat{f}^{\mathrm{wav}}$.

    If $\beta > D/4$ and $\alpha > (D-1)/2$ hold with $D=2$, \eqref{ineq:cor2} holds with $\widecheck{f} = \hat{f}^{\mathrm{curve}}$.
\end{corollary}
These results are naturally derived from the discussion of the minimax optimal rate of DNN (Corollary \ref{cor:non-bayes} and Theorem \ref{thm:minimax}), and the sub-optimality of the other methods (Propositions \ref{prop:linear}, \ref{prop:wavelet}, and \ref{prop:curvelet}).
We can conclude that DNNs offer a theoretical advantage over the other methods in terms of estimation for functions with singularity; that is, with the parameter configurations in Corollary \ref{cor:compare}, there exist $f^* \in\mF_{\alpha,\beta,M}^{PS}$, such that
\begin{align*}
    \Ep_{f^*} \left[ \|\hat{f}^{DL} - f^*\|_{L^2(P_X)}^2\right] < \Ep_{f^*} \left[ \|\widecheck{f} - f^*\|_{L^2(P_X)}^2\right]
\end{align*}
holds for a sufficiently large $n$.

The parameter configurations introduced in Corollary \ref{cor:compare} are classified into two cases, as illustrated in Figure \ref{fig:parameter}.
First, the configuration for the sub-optimality of the linear estimators appears when $\alpha$ is relatively smaller than $\beta$.
With a small $\alpha$, the rate $O(n^{\alpha/(\alpha + D-1)})$ in the minimax optimal rate dominates the generalization error, because it is more difficult to estimate the boundary of pieces than to estimate the function within the pieces.
In this case, the linear estimators lose optimality, because their generalization error is sensitive to $\alpha$.
The second case is that in which both $\alpha$ and $\beta$ are above certain thresholds.
It is thought that the approximation of the singularity by the orthogonal bases has not been adapted to the higher smoothness of $f^*$.

\begin{figure}[t]
\begin{center}
\includegraphics[width=0.6\hsize]{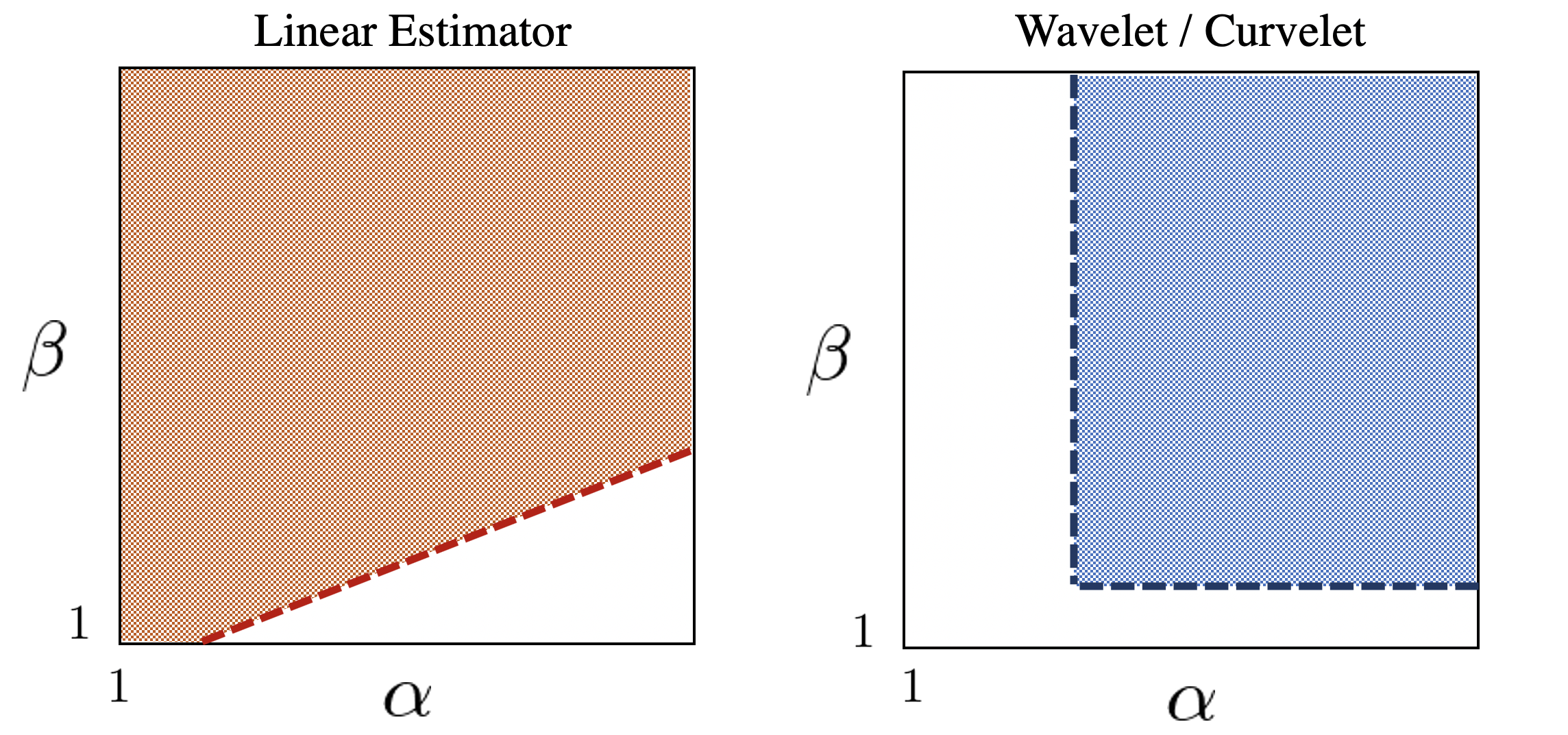}
\caption{Parameter spaces of $(\alpha,\beta)$. [Left] The red dashed line is $\alpha = 2\beta (D-1) /( D-2\beta)$. The red region presents $\{(\alpha,\beta)\mid \alpha < 2\beta (D-1) / ( D-2\beta) \}$, which a set of parameter configurations such that \eqref{ineq:cor2} holds with $\widecheck{f} = \hat{f}^{\mathrm{lin}}$. [Right] The blue region presents $\{(\alpha,\beta)\mid \alpha > D-1, \beta > D/2 \}$, which a set of parameter configurations such that \eqref{ineq:cor2} holds with $\widecheck{f} = \hat{f}^{\mathrm{wav}}$. The blue dashed line is its boundary. \label{fig:parameter}}
\end{center}
\end{figure}

We add a further explanation for the limitations of the other estimators by describing the difficulty of shape fitting to the hypersurface of singularities.
The other estimators take the form of sums of (not necessarily orthogonal) basis functions, and each base has a (nearly) compact support.
When the other methods approximate a function with the hypersurface of singularities, they approximate the hypersurface by fitting its support.
For example, consider the curvelet of which the basis function has an ellipsoid-shaped (nearly) compact support.
The ellipsoid is fitted to the hypersurface of the singularity with rotation, as illustrated in Figure \ref{fig:exp_compare}.
The number of bases determines the magnitude of the error.
However, when the hypersurface has higher-order smoothness, the fitting in the domain cannot adapt to the hypersurface with optimality.

DNNs do not use shape fitting to the hypersurface for singularity, but represent the hypersurface using a composition of functions, as mentioned in Section \ref{sec:explain_dnn}.
Therefore, even if the hypersurface has larger smoothness, DNNs can handle this without losing efficiency.

\begin{figure}[htbp]
\begin{center}
\includegraphics[width=0.5\hsize]{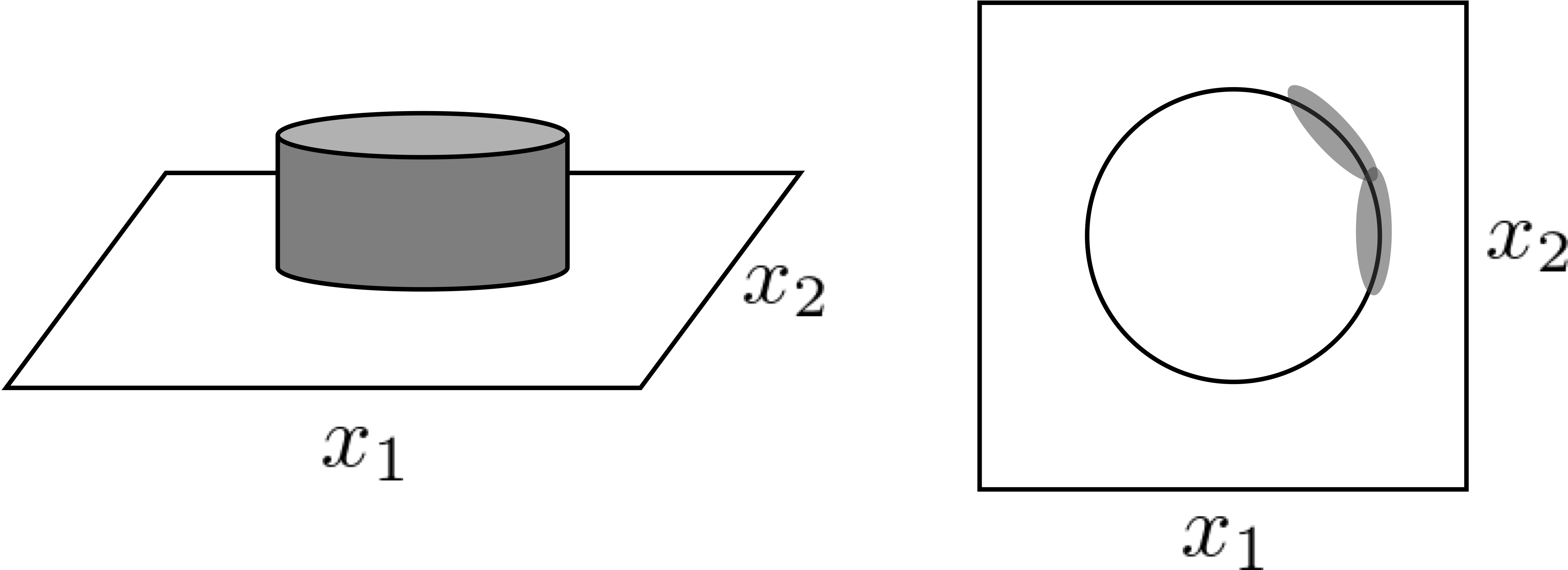}
\caption{[Left] An example of piecewise smooth functions $f^C$ such as the indicator function of a disk. [Right] The singularities of $f^C$ (black circle), and an illustration of approximation for the circle by a curvelet basis. 
The curvelet basis concentrates in the gray ellipses, that cannot have sufficient approximation ability when the curves are too smooth.
\label{fig:exp_compare}}
\end{center}
\end{figure}

\begin{remark}[Limitations of the advantage with learning algorithms]
    We also need to reiterate the limitations of these advantages we claim to have. 
    Needless to say, our analysis is based on the setting that a global optimum of the loss minimization problem has been found in \eqref{opt:erm}.
    In other words, we only study the errors from an approximation power and a degree of freedom of DNN models, and ignore the errors from the training algorithm of DNNs in practical use. 
    However, this algorithmic error requires a completely different analysis, and analyzing them all at the same time would be very difficult and not very meaningful. 
    What we argue in this work is that DNNs can handle singularity more optimally than existing linear estimators, including those for singularity resolution, even if only in terms of approximation power and a degree of freedom derived error aspects.
\end{remark}

\section{Conclusion} \label{sec:concl}

In this study, we have derived theoretical results that explain why DNNs outperform other methods.
We considered the regression setting in the situation whereby the true function is singular on a smooth hypersurface in its domain.
We derived the convergence rates of the estimator obtained by DNNs and proved that the rates were almost optimal in the minimax sense.
We explained that the optimality of DNNs originates from their composition structure, which can resolve singularity.
Furthermore, to analyze the advantage of DNNs, we investigated the sub-optimality of several other estimators, such as linear, wavelet, and curvelet estimators.
We proved the sub-optimality of each estimator with certain parameter configurations.
This advantage of DNNs comes from the fact that the shape of smooth curves for the singularity can be handled by DNNs, while the other methods fail to capture the shape efficiently.
Theoretically, this is a vital step for analyzing the mechanism of DNNs.

\acks{We have greatly benefited from insightful comments and suggestions by Alexandre Tsybakov, Taiji Suzuki, Bharath K Sriperumbudur, Johannes Schmidt-Hieber, and Motonobu Kanagawa.
We also thank Samory Kpotufe and the anonymous reviewers for thoughtful and constructive comments.
M.Imaizumi was supported by JSPS KAKENHI Grant Number 18K18114 and JST Presto Grant Number JPMJPR1852.}

\newpage

\appendix

{\bc 
\section{Proof Overview}

\subsection{Upper Bound on Generalization Error (Theorem \ref{thm:non-bayes})}

We give an overview of the proof of Theorem \ref{thm:non-bayes}, which derives the upper bound on the expected generalization error $\Ep_{f^*} [\|\hat{f}^{DL} -f^*\|_{L^2(P_X)}^2  ]$. 
As shown in \eqref{ineq:main_basic}, the generalization error is decomposed into the approximation error $\mathcal{B}$ and the complexity error $\mathcal{V}$. 
In the proof, we evaluate these two terms separately.
A flowchart of the overview is shown in Figure \ref{fig:flow_chart_theorems}.
}

\begin{figure}[htbp]
    \centering
\tikzset{every picture/.style={line width=0.75pt}} %set default line width to 0.75pt        

\begin{tikzpicture}[x=0.75pt,y=0.75pt,yscale=-1,xscale=0.94]
%uncomment if require: \path (0,300); %set diagram left start at 0, and has height of 300

%Rounded Rect [id:dp501693525959084] 
\draw  [draw opacity=0][fill={rgb, 255:red, 155; green, 155; blue, 155 }  ,fill opacity=0.28 ] (291,26) .. controls (291,15.51) and (299.51,7) .. (310,7) -- (628,7) .. controls (638.49,7) and (647,15.51) .. (647,26) -- (647,83) .. controls (647,93.49) and (638.49,102) .. (628,102) -- (310,102) .. controls (299.51,102) and (291,93.49) .. (291,83) -- cycle ;
%Flowchart: Process [id:dp8223530878516145] 
\draw   (59,32) -- (231,32) -- (231,94) -- (59,94) -- cycle ;
%Flowchart: Process [id:dp7070278166624328] 
\draw   (312,39) -- (459,39) -- (459,86) -- (312,86) -- cycle ;
%Flowchart: Process [id:dp2924703163552089] 
\draw   (480,39) -- (627,39) -- (627,86) -- (480,86) -- cycle ;
%Straight Lines [id:da1569600061184765] 
\draw    (231,61) -- (292,61) ;
\draw [shift={(294,61)}, rotate = 180] [color={rgb, 255:red, 0; green, 0; blue, 0 }  ][line width=0.75]    (10.93,-3.29) .. controls (6.95,-1.4) and (3.31,-0.3) .. (0,0) .. controls (3.31,0.3) and (6.95,1.4) .. (10.93,3.29)   ;
%Rounded Rect [id:dp5016822684453438] 
\draw  [draw opacity=0][fill={rgb, 255:red, 155; green, 155; blue, 155 }  ,fill opacity=0.28 ] (19,165) .. controls (19,155.61) and (26.61,148) .. (36,148) -- (145,148) .. controls (154.39,148) and (162,155.61) .. (162,165) -- (162,269) .. controls (162,278.39) and (154.39,286) .. (145,286) -- (36,286) .. controls (26.61,286) and (19,278.39) .. (19,269) -- cycle ;
%Straight Lines [id:da47835297611601335] 
\draw    (66,154) -- (66,97) ;
\draw [shift={(66,95)}, rotate = 450] [color={rgb, 255:red, 0; green, 0; blue, 0 }  ][line width=0.75]    (10.93,-3.29) .. controls (6.95,-1.4) and (3.31,-0.3) .. (0,0) .. controls (3.31,0.3) and (6.95,1.4) .. (10.93,3.29)   ;
%Rounded Rect [id:dp051149188324305594] 
\draw  [draw opacity=0][fill={rgb, 255:red, 155; green, 155; blue, 155 }  ,fill opacity=0.28 ] (172,163) .. controls (172,154.72) and (178.72,148) .. (187,148) -- (636,148) .. controls (644.28,148) and (651,154.72) .. (651,163) -- (651,268) .. controls (651,276.28) and (644.28,283) .. (636,283) -- (187,283) .. controls (178.72,283) and (172,276.28) .. (172,268) -- cycle ;
%Straight Lines [id:da1926968450028388] 
\draw    (208,156) -- (208,97) ;
\draw [shift={(208,95)}, rotate = 450] [color={rgb, 255:red, 0; green, 0; blue, 0 }  ][line width=0.75]    (10.93,-3.29) .. controls (6.95,-1.4) and (3.31,-0.3) .. (0,0) .. controls (3.31,0.3) and (6.95,1.4) .. (10.93,3.29)   ;
%Flowchart: Process [id:dp7634673089302582] 
\draw   (28,155) -- (156,155) -- (156,204) -- (28,204) -- cycle ;
%Flowchart: Process [id:dp6791508221097109] 
\draw   (28,230) -- (156,230) -- (156,279) -- (28,279) -- cycle ;
%Flowchart: Process [id:dp5976693930078085] 
\draw   (185,155) -- (339,155) -- (339,204) -- (185,204) -- cycle ;
%Flowchart: Process [id:dp4458323352707586] 
\draw   (511,226) -- (648,226) -- (648,275) -- (511,275) -- cycle ;
%Flowchart: Process [id:dp6063004941790251] 
\draw   (330,226) -- (493,226) -- (493,275) -- (330,275) -- cycle ;
%Straight Lines [id:da022436814733154842] 
\draw    (511,252) -- (493,252) ;
\draw [shift={(491,252)}, rotate = 360] [color={rgb, 255:red, 0; green, 0; blue, 0 }  ][line width=0.75]    (10.93,-3.29) .. controls (6.95,-1.4) and (3.31,-0.3) .. (0,0) .. controls (3.31,0.3) and (6.95,1.4) .. (10.93,3.29)   ;
%Flowchart: Process [id:dp24317990403274248] 
\draw   (188,226) -- (311,226) -- (311,275) -- (188,275) -- cycle ;
%Straight Lines [id:da05332596058935657] 
\draw    (329,252) -- (311,252) ;
\draw [shift={(309,252)}, rotate = 360] [color={rgb, 255:red, 0; green, 0; blue, 0 }  ][line width=0.75]    (10.93,-3.29) .. controls (6.95,-1.4) and (3.31,-0.3) .. (0,0) .. controls (3.31,0.3) and (6.95,1.4) .. (10.93,3.29)   ;
%Flowchart: Process [id:dp027399964149891654] 
\draw   (387,155) -- (510,155) -- (510,204) -- (387,204) -- cycle ;
%Straight Lines [id:da27628121147312523] 
\draw    (577,226) -- (511.9,204.62) ;
\draw [shift={(510,204)}, rotate = 378.18] [color={rgb, 255:red, 0; green, 0; blue, 0 }  ][line width=0.75]    (10.93,-3.29) .. controls (6.95,-1.4) and (3.31,-0.3) .. (0,0) .. controls (3.31,0.3) and (6.95,1.4) .. (10.93,3.29)   ;
%Straight Lines [id:da48906385395865526] 
\draw    (386,180) -- (339,180) ;
\draw [shift={(337,180)}, rotate = 360] [color={rgb, 255:red, 0; green, 0; blue, 0 }  ][line width=0.75]    (10.93,-3.29) .. controls (6.95,-1.4) and (3.31,-0.3) .. (0,0) .. controls (3.31,0.3) and (6.95,1.4) .. (10.93,3.29)   ;
%Straight Lines [id:da5329593028528562] 
\draw    (238,225) -- (238,205) ;
\draw [shift={(238,203)}, rotate = 450] [color={rgb, 255:red, 0; green, 0; blue, 0 }  ][line width=0.75]    (10.93,-3.29) .. controls (6.95,-1.4) and (3.31,-0.3) .. (0,0) .. controls (3.31,0.3) and (6.95,1.4) .. (10.93,3.29)   ;
%Straight Lines [id:da22759127733717233] 
\draw    (311,226) -- (385.08,204.56) ;
\draw [shift={(387,204)}, rotate = 523.86] [color={rgb, 255:red, 0; green, 0; blue, 0 }  ][line width=0.75]    (10.93,-3.29) .. controls (6.95,-1.4) and (3.31,-0.3) .. (0,0) .. controls (3.31,0.3) and (6.95,1.4) .. (10.93,3.29)   ;

% Text Node
\draw (145.18,38) node [anchor=north] [inner sep=0.75pt]   [align=left] {\begin{minipage}[lt]{128.13pt}\setlength\topsep{0pt}
\begin{center}
\textbf{Generalization\\Bound}\\(Thm. \ref{thm:non-bayes})
\end{center}

\end{minipage}};
% Text Node
\draw (311,12) node [anchor=north west][inner sep=0.75pt]   [align=left] {\textbf{Applications}};
% Text Node
\draw (384.18,43) node [anchor=north] [inner sep=0.75pt]   [align=left] {\begin{minipage}[lt]{87.77pt}\setlength\topsep{0pt}
\begin{center}
Convergence Rate\\(Cor. \ref{cor:non-bayes})
\end{center}

\end{minipage}};
% Text Node
\draw (25,9) node [anchor=north west][inner sep=0.75pt]   [align=left] {\textbf{\underline{Main Result}}};
% Text Node
\draw (555.18,43) node [anchor=north] [inner sep=0.75pt]   [align=left] {\begin{minipage}[lt]{84.53pt}\setlength\topsep{0pt}
\begin{center}
Algorithm Effect\\(Prop. \ref{cor:opt})
\end{center}

\end{minipage}};
% Text Node
\draw (255.18,103) node [anchor=north] [inner sep=0.75pt]   [align=left] {\textbf{Bound}\\\textbf{Approx. }$\displaystyle \mathcal{B}$};
% Text Node
\draw (119.18,102) node [anchor=north] [inner sep=0.75pt]   [align=left] {\textbf{Bound}\\\textbf{Complex. }$\displaystyle \mathcal{V}$};
% Text Node
\draw (92.18,161) node [anchor=north] [inner sep=0.75pt]   [align=left] {\begin{minipage}[lt]{86.08pt}\setlength\topsep{0pt}
\begin{center}
Peeling Technique\\(Sec. \ref{sec:entropy})
\end{center}

\end{minipage}};
% Text Node
\draw (94.18,236) node [anchor=north] [inner sep=0.75pt]   [align=left] {\begin{minipage}[lt]{75.87pt}\setlength\topsep{0pt}
\begin{center}
Covering Bound\\(Lem. \ref{lem:entropy_DNN})
\end{center}

\end{minipage}};
% Text Node
\draw (262.18,161) node [anchor=north] [inner sep=0.75pt]   [align=left] {\begin{minipage}[lt]{98.32pt}\setlength\topsep{0pt}
\begin{center}
Approx. on $\mF_{\alpha,\beta,M}^{PS}$\\(Thm. \ref{thm:approx} / Prop. \ref{prop:approx})
\end{center}

\end{minipage}};
% Text Node
\draw (88,207) node [anchor=north west][inner sep=0.75pt]   [align=left] {\textbf{+}};
% Text Node
\draw (579.18,232) node [anchor=north] [inner sep=0.75pt]   [align=left] {\begin{minipage}[lt]{96.06pt}\setlength\topsep{0pt}
\begin{center}
Approx. on step\\(Lem. \ref{lem:main_step} / \ref{lem:step_approx})
\end{center}

\end{minipage}};
% Text Node
\draw (413.18,232) node [anchor=north] [inner sep=0.75pt]   [align=left] {\begin{minipage}[lt]{128.75pt}\setlength\topsep{0pt}
\begin{center}
Approx. on polynomial\\(Lem. \ref{lem:approx_poly1} / \ref{lem:approx_poly2})
\end{center}

\end{minipage}};
% Text Node
\draw (250.18,230) node [anchor=north] [inner sep=0.75pt]   [align=left] {\begin{minipage}[lt]{87.56pt}\setlength\topsep{0pt}
\begin{center}
Approx. on $H^\beta$\\(Lem. \ref{lem:main_smooth} / \ref{lem:smooth_approx})
\end{center}

\end{minipage}};
% Text Node
\draw (449.18,161) node [anchor=north] [inner sep=0.75pt]   [align=left] {\begin{minipage}[lt]{89.27pt}\setlength\topsep{0pt}
\begin{center}
Approx. on piece\\(Lem. \ref{lem:main_indicator} / \ref{lem:indicator_approx})
\end{center}

\end{minipage}};

%\draw   (231, 61) circle [x radius= 5, y radius= 5]   ;

\end{tikzpicture}
    \caption{{\bc Flowchart showing the relationship between theorems/lemmas/propositions for Theorem \ref{thm:non-bayes} on the upper bound of the generalization error. 
    Theorem \ref{thm:non-bayes} is mainly derived from the two groups of results bounding the approximation and complexity errors, and also is applied to the additional results. 
    When multiple theorems/lemmas/propositions are specified within one box, it represents a more generalised result or a variant with different activation functions.}}
    \label{fig:flow_chart_theorems}
\end{figure}
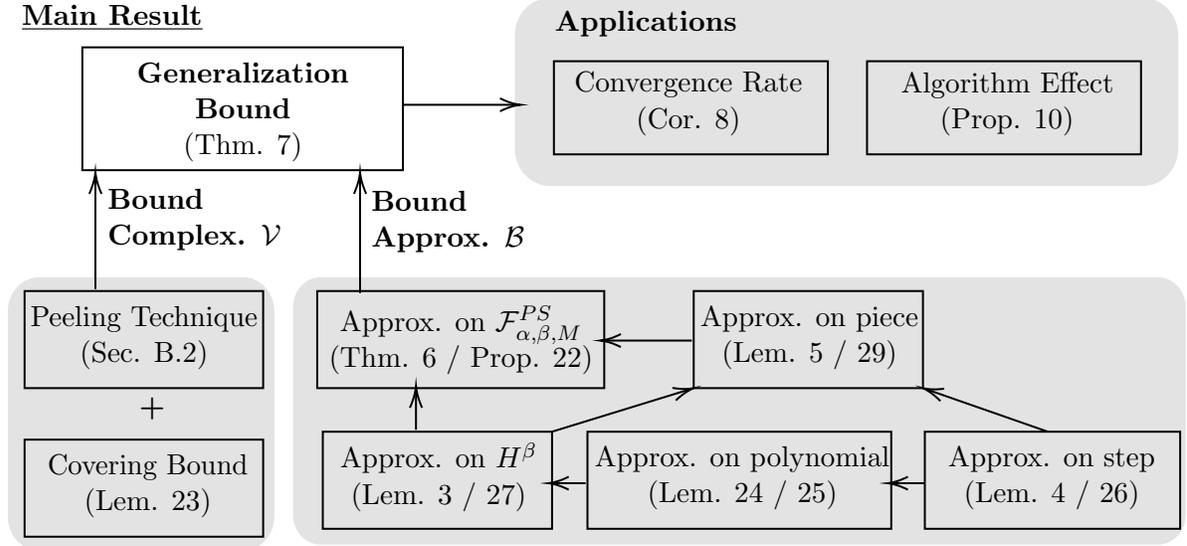

{\bc

The evaluation of the approximation error $\mathcal{B}$, which is carried out in Section \ref{sec:approx}, uses the evaluation of an approximation error of the piecewise smooth functions  $\mF_{\alpha,\beta,M}^{PS}$ by DNNs.
According to the definition of piecewise smooth functions $f \in \mF_{\alpha,\beta,M}^{PS}$ as
\begin{align*}
    f = \sum_{m \in [M]} f_m \otimes \mone_{R_m}, \mbox{~where~}f_m \in H_F^{\beta}(I^D) \mbox{~and~}\{R_m\}_{m \in [M]} \in \mR_{\alpha,M}. 
\end{align*}
the approximation error is described by an approximation error of a simple smooth function $f_m \in H_F^{\beta}(I^D)$ and the approximation error of the indicator function on a piece $\mone_{R_m}$.
This basic idea is originally proposed by \cite{petersen2017optimal}.
To approximate $f_m \in H_F^{\beta}(I^D)$, we apply an approximation technique based on an aggregation of local polynomial functions, which is originally developed by \cite{yarotsky2017error}. 
Since this technique requires a step function $\mone_{\{\cdot \geq 0\}}$, we show that a step function $\mone_{\{\cdot \geq 0\}}$ can be well approximated by the wide class of admissible activation functions in Lemma \ref{lem:main_step} and \ref{lem:step_approx}.
This approximation on a step function is a key part of our proof, and the approximation results on smooth functions are updated to adapt to this key part.
Based on this result, we can approximate polynomials in Lemma \ref{lem:approx_poly1} and \ref{lem:approx_poly2}, and a smooth function $f \in H_F^\beta$ in Lemma \ref{lem:main_smooth} and \ref{lem:smooth_approx}.
In contrast, approximation of the indicator function $\mone_{R_m}$ is evaluated by combining the approximation on both a step function $\mone_{\{\cdot \geq 0\}}$ and a smooth function $f \in H_F^\beta(I^D)$.
This approach takes advantage of the composition structure of DNNs described in Section \ref{sec:explain_dnn}.
As a result, the final approximation error has the following rate:
\begin{align*}
    \mB = \tilde{O}(S^{-2\beta/D} + S^{-\alpha/(D-1)}),
\end{align*}
which is a sum of the approximation errors of $f \in H_F^\beta(I^D)$ and $\mone_{R_m}$.

The evaluation of the complexity error $\mathcal{V}$ in Section \ref{sec:entropy} is based on a covering number analysis on the function set by DNNs $\mG(L,S,B)$.
This makes use of the technique of peeling from the empirical process theory, which is developed, for example, by \cite{koltchinskii2006local}, and its application to DNNs is proposed in \cite{schmidt2017nonparametric} and others.
Specifically, we define a localized subset of functions around a specific function $f \in \mG(L,S,B)$, then use uniform convergence over a finite subset in the localized subset to bound the complexity error $\mathcal{V}$. 
Since we use a covering set to construct the finite subset, the covering number $\log \mN(\varepsilon, \mG(L,S,B), \|\cdot \|_{L^\infty(I^D)})$ of the DNN functions $\mG(L,S,B)$ becomes essential. 
Briefly summarised, an expectation of the complexity error has the following order
\begin{align*}
    {O}\left(\frac{\log \mN(1/n, \mG(L,S,B), \|\cdot \|_{L^\infty(I^D)})}{n}\right) = \tilde{O}\left(\frac{SL \log (SLB)}{n}\right),
\end{align*}
with a number of parameters $S$ and layers $L$ and a parameter volume $B$.
The equality follows the covering number bound in Lemma \ref{lem:entropy_DNN} from \cite{nakada2019adaptive}.

The combination of the above two evaluations bounds the generalization error $\Ep_{f^*} [\|\hat{f}^{DL} -f^*\|_{L^2(P_X)}^2  ]$, together with a choice of  $S$,  $L$, and $B$.
By selecting $L$ and $B$ as a polynomial $\log n$, the generalisation error is bounded as
\begin{align}
    &\Ep_{f^*} [\|\hat{f}^{DL} -f^*\|_{L^2(P_X)}^2  ] 
    = \tilde{O}(S^{-2\beta/D} + S^{-\alpha/(D-1)}) + \tilde{O}(S/n) \label{eq:error_bound}
\end{align}
In \eqref{eq:error_bound}, the first item is an upper bound on the approximation error $\mB$, and the second term is a bound on the complexity error $\mV$. 
Since the approximation error $\mB$ is decreasing in $S$ and the complexity error $\mV$ is increasing in $S$, we select the number of parameter $S = \Theta(n^{D/(2\beta + D)} + n^{(D-1)/(\alpha + D - 1)})$ to handle the trade-off between $\mB$ and $\mV$.
Then, we obtain the statement of Theorem \ref{thm:non-bayes}.
By applying this result, the convergence rate is derived.

\subsection{Lower Bound of Generalization Error (Theorem \ref{thm:minimax})}

We provide an overview of the proof of Theorem \ref{thm:minimax} for the lower bound on the expected generalization error. 
The proof depends mainly on the information-theoretic analysis of the minimax rate, which is developed by \cite{yang1999information}.
We cite Theorem 6 of \cite{yang1999information} as Theorem \ref{thm:yangbarron} in this paper, then focus on a packing number $\mM(\varepsilon_n, \mF_{\alpha,\beta,M}^{PS}, \|\cdot\|_{L^2})$ to apply the theory.
In Proposition \ref{prop:packing_subset}, we investigate packing numbers of each element set that constitutes $\mF_{\alpha,\beta,M}^{PS}$.
We utilize several results from \cite{dudley2014uniform} and derive the bound for the packing number.
Figure \ref{fig:flow_chart_lower} illustrates the relation.
}

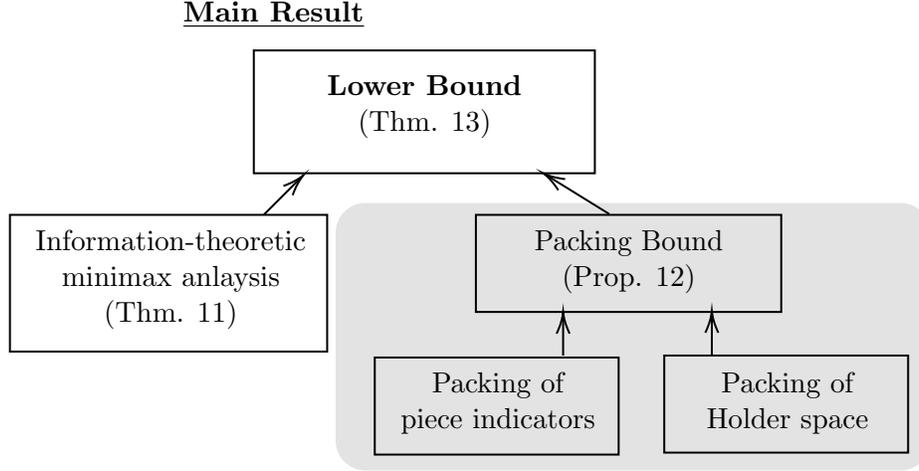
\begin{figure}[htbp]
    \centering
\tikzset{every picture/.style={line width=0.75pt}} %set default line width to 0.75pt        

\begin{tikzpicture}[x=0.75pt,y=0.75pt,yscale=-1,xscale=1]
%uncomment if require: \path (0,300); %set diagram left start at 0, and has height of 300

%Flowchart: Process [id:dp2627702417045745] 
\draw   (137,28) -- (309,28) -- (309,90) -- (137,90) -- cycle ;
%Straight Lines [id:da9802845166288651] 
\draw    (142,111) -- (160.59,92.41) ;
\draw [shift={(162,91)}, rotate = 495] [color={rgb, 255:red, 0; green, 0; blue, 0 }  ][line width=0.75]    (10.93,-3.29) .. controls (6.95,-1.4) and (3.31,-0.3) .. (0,0) .. controls (3.31,0.3) and (6.95,1.4) .. (10.93,3.29)   ;
%Rounded Rect [id:dp03474512623971626] 
\draw  [draw opacity=0][fill={rgb, 255:red, 155; green, 155; blue, 155 }  ,fill opacity=0.28 ] (178.25,120) .. controls (178.25,111.72) and (184.96,105) .. (193.25,105) -- (463.25,105) .. controls (471.53,105) and (478.25,111.72) .. (478.25,120) -- (478.25,225) .. controls (478.25,233.28) and (471.53,240) .. (463.25,240) -- (193.25,240) .. controls (184.96,240) and (178.25,233.28) .. (178.25,225) -- cycle ;
%Straight Lines [id:da8962837207207622] 
\draw    (316.25,110.56) -- (284.7,91.05) ;
\draw [shift={(283,90)}, rotate = 391.73] [color={rgb, 255:red, 0; green, 0; blue, 0 }  ][line width=0.75]    (10.93,-3.29) .. controls (6.95,-1.4) and (3.31,-0.3) .. (0,0) .. controls (3.31,0.3) and (6.95,1.4) .. (10.93,3.29)   ;
%Flowchart: Process [id:dp16870427255019826] 
\draw   (14,111) -- (174,111) -- (174,180) -- (14,180) -- cycle ;
%Flowchart: Process [id:dp3535633163216926] 
\draw   (249,111) -- (403,111) -- (403,160) -- (249,160) -- cycle ;
%Flowchart: Process [id:dp7373852339431987] 
\draw   (198,183) -- (321,183) -- (321,232) -- (198,232) -- cycle ;
%Flowchart: Process [id:dp9281115562618228] 
\draw   (344,182) -- (467,182) -- (467,231) -- (344,231) -- cycle ;
%Straight Lines [id:da3730073342621292] 
\draw    (368,182) -- (368,161.56) ;
\draw [shift={(368,159.56)}, rotate = 450] [color={rgb, 255:red, 0; green, 0; blue, 0 }  ][line width=0.75]    (10.93,-3.29) .. controls (6.95,-1.4) and (3.31,-0.3) .. (0,0) .. controls (3.31,0.3) and (6.95,1.4) .. (10.93,3.29)   ;
%Straight Lines [id:da7765427792020153] 
\draw    (293,182) -- (293,162) ;
\draw [shift={(293,160)}, rotate = 450] [color={rgb, 255:red, 0; green, 0; blue, 0 }  ][line width=0.75]    (10.93,-3.29) .. controls (6.95,-1.4) and (3.31,-0.3) .. (0,0) .. controls (3.31,0.3) and (6.95,1.4) .. (10.93,3.29)   ;

% Text Node
\draw (223.18,39) node [anchor=north] [inner sep=0.75pt]   [align=left] {\begin{minipage}[lt]{87.88pt}\setlength\topsep{0pt}
\begin{center}
\textbf{Lower Bound}\\(Thm. 13)
\end{center}

\end{minipage}};
% Text Node
\draw (100,2) node [anchor=north west][inner sep=0.75pt]   [align=left] {\textbf{\underline{Main Result}}};
% Text Node
\draw (95.18,117) node [anchor=north] [inner sep=0.75pt]   [align=left] {\begin{minipage}[lt]{116.27pt}\setlength\topsep{0pt}
\begin{center}
Information-theoretic\\minimax anlaysis\\(Thm. 11)
\end{center}

\end{minipage}};
% Text Node
\draw (326.18,117) node [anchor=north] [inner sep=0.75pt]   [align=left] {\begin{minipage}[lt]{71.34pt}\setlength\topsep{0pt}
\begin{center}
Packing Bound\\(Prop. \ref{prop:packing_subset})
\end{center}

\end{minipage}};
% Text Node
\draw (260.18,189) node [anchor=north] [inner sep=0.75pt]   [align=left] {\begin{minipage}[lt]{73.6pt}\setlength\topsep{0pt}
\begin{center}
Packing of \\piece indicators
\end{center}

\end{minipage}};
% Text Node
\draw (406.18,189) node [anchor=north] [inner sep=0.75pt]   [align=left] {\begin{minipage}[lt]{62.83pt}\setlength\topsep{0pt}
\begin{center}
Packing of\\Holder space
\end{center}

\end{minipage}};

\end{tikzpicture}
    \caption{{\bc Flowchart for Theorem \ref{thm:minimax} on the lower bound of the generalization error. 
     \label{fig:flow_chart_lower}}}
\end{figure}

\section{Proof of Theorem \ref{thm:non-bayes}}

We start with the basic inequality \eqref{ineq:main_basic}.
As preparation, we introduce additional notation.
Given an empirical measure, the empirical (pseudo) norm of a random variable is defined by  $\|Y\|_n := (  n^{-1} \sum_{i \in [n]} Y_i^2)^{1/2}$ and $\|\xi\|_n := (  n^{-1} \sum_{i \in [n]} \xi_i^2)^{1/2}$.
By the definition of $\hat{f}^{DL}$ in \eqref{opt:erm}, we obtain the following inequality 
\begin{align*}
    \|Y - \hat{f}^{DL}\|_n^2 \leq \|Y - f\|_n^2
\end{align*}
for all $f \in \mG(L,S,B)$.
It follows from $Y_i = f^*(X_i) + \xi_i$ that
\begin{align*}
    \|f^* +\xi - \hat{f}^{DL}\|_n^2 \leq \|f^* + \xi - f\|_n^2,
\end{align*}
then, simple calculation yields
\begin{align*}
    \| \hat{f}^{DL} - f^*\|_n^2 & \leq \|f^* - f\|_n^2 + \frac{2}{n}\sum_{i \in [n]} \xi_i (\hat{f}^{DL}(X_i) - f(X_i)).
\end{align*}
We set $f \in \mG(L,S,B)$ as satisfying $\|f^* - f\|_n = \inf_{f' \in \mG(L,S,B)}\|f^* - f'\|_n$, then we obtain
\begin{align}
    \| \hat{f}^{DL} - f^*\|_n^2 & \leq \inf_{f' \in \mG(L,S,B)}\|f^* - f'\|_n^2 + \frac{2}{n}\sum_{i \in [n]} \xi_i (\hat{f}^{DL}(X_i) - f(X_i)) \label{ineq:basic2} \\
    &=: \mB + \mV. \notag
\end{align}

In the first subsection, we bound $\mB$ by investigating an approximation power of DNNs.
In the second subsection, we evaluate $\mV$ by evaluating the complexity of the estimator.
Afterward, we combine the results and derive an overall rate. 

\subsection{Approximate Piecewise Smooth Functions by DNNs} \label{sec:approx}

In this subsection, we provide proof of Theorem \ref{thm:approx}.
The result follows the following proposition:
\begin{proposition}[General Version of Theorem \ref{thm:approx}] \label{prop:approx}
    Suppose Assumption \ref{asmp:activation} holds.
    Then, for any $\varepsilon_1 \in (0,1)$ and $\varepsilon_2 \in (0,1)$, there exists a tuple $(L,S,B)$ such as
    \begin{itemize}
        \item $L = C_{\alpha,\beta,D,F}(\lfloor \alpha \rfloor + \lfloor \beta \rfloor + \log_2( 1/\varepsilon_1) + \log_2( M/\varepsilon_2)) + 1)$,
        \item $S =  C_{\alpha,\beta,D,F,J} ( M  \varepsilon_1^{-D/\beta} ( \log_2( 1/\varepsilon_1))^2 +   (\varepsilon_2/M)^{-2(D-1)/\alpha} \log_2(M/\varepsilon_2)^2 + M(\log_2(M/\varepsilon_2))^2)$,
        \item $B = C_{F,M,q}(\varepsilon_1 \wedge \varepsilon_2)^{-16 \wedge -C_{\alpha,\beta}}$,
    \end{itemize}
    which satisfy
    \begin{align*}
        \inf_{f \in \mG(L,S,B)} \sup_{f^* \in \mF_{\alpha,\beta,M}^{PS}} \|f - f^*\|_{L^2} \leq \varepsilon_1 + \varepsilon_2.
    \end{align*}
\end{proposition}
\begin{proof}[Proof of Proposition \ref{prop:approx}]
Fix $f^* \in\mF_{\alpha,\beta,M}^{PS}$ such that $f^* = \sum_{m \in [M]} f^*_m \otimes \mone_{R^*_m}$ with $f^*_m\in H_F^\beta(I^D)$ and $\{R_m\}_{m \in [M]} \in \mathcal{R}_{\alpha,M}$ for $m \in [M]$.
By Lemma \ref{lem:smooth_approx}, for any $\delta_1 \in (0,1)$, there exist a constant $c_1 > 0$ and functions $g_{f,1},...,g_{f,M} \in \mG(C_{\beta,D,F}(\lfloor \beta \rfloor +  ( \log_2(1/ \delta_1) + 1), C_{\beta,D,F} \delta_1^{-D/\beta} ( \log_2(1/\delta_1))^2 ,$ $ C_{F,q}\delta_1^{-16 \wedge -C_\beta} )$ such that $\|g_{f,m} - f^*_m\|_{L^2(R_m)} \leq \mathrm{vol}(R_m) \delta_1$ for $m \in [M]$.
Similarly, by Lemma \ref{lem:indicator_approx}, we can find $g_R \in \mG(C_{\alpha,D,F,J}(\lfloor \alpha \rfloor + \log_2(1/\delta_2) + 1), C_{\alpha,D,F,J}( \delta_2^{-2(D-1)/\alpha} \log_2(1/\delta_2)^2 + 1 + M(\log_2(1/\delta_2))^2),$ $ C_{F,q}\delta_2^{-16 \wedge -C_\alpha})$ such that $\|g_{R,m}- \mone_{R_m^*}\|_{L^2} \leq \delta_2$ for $\delta_2 \in (0,1)$.
For approximation, we follow \eqref{ineq:multi} in the proof of Lemma \ref{lem:approx_poly2} as $g_c \in \mG(\log_2(1/\delta_1),C((\log_2(1/\delta))^2 + 1) , C)$ which approximates a multiplication $\|((x,x') \mapsto xx') - g_c(x,x')\|_{L^\infty([-F,F]^2)} \leq F^2 \delta_1$.

With these components, we construct a function $\hat{g} \in \mG( C_{\alpha,\beta,D,F}(\lfloor \alpha \rfloor + \lfloor \beta \rfloor + \log_2( 1/\delta_1) + \log_2( 1/\delta_1)) + 1), C_{\alpha,\beta,D,F,J} ( M  \delta_1^{-D/\beta} ( \log_2( 1/\delta_1))^2 +   \delta_2^{-(D-1)/\alpha} (\log_2(1/\delta_2))^2 + 1 + M(\log_2(1/\delta_2))^2) ,$ $C_{F,q}(\delta_1^{-16 \wedge C_\alpha} \wedge \delta_2^{-16 \wedge C_\beta}))$ as
\begin{align}
    \hat{g}(x) = \sum_{m\in [M]} g_c(g_{f,m}(x),g_{R,m}(x)), \label{def:ghat}
\end{align}
by setting a parameter matrix $A_{L} = (1,1,...,1)^\top$ in the last layer of $\hat{g}$.
Then, we evaluate the distance between $f^*$ and the combined DNN:
\begin{align}
    &\|f^* - \hat{g}\|_{L^2(I^D)} \notag \\
    &= \left\| \sum_{m \in [M] }f^*_m \mone_{R^*_m} -\sum_{m\in [M]} g_c(g_{f,m}(\cdot),g_{R,m}(\cdot))\right\|_{L^2(I^D)} \notag \\
    & \leq \sum_{m \in [M] } \left\| f^*_m \otimes \mone_{R^*_m} - g_{f,m} \otimes g_{R,m}\right\|_{L^2(I^D)} \notag  \\
    & \quad + \sum_{m \in [M] } \left\|  g_{f,m} \otimes g_{R,m} - g_c(g_{f,m}(\cdot),g_{R,m}(\cdot))\right\|_{L^2(I^D)} \notag  \\
    & \leq \sum_{m \in [M] } \left\| ( f^*_m -  g_{f,m}) \otimes  \mone_{R^*_m} \right\|_{L^2(I^D)} +\sum_{m \in [M] }  \left\| g_{f,m}\otimes  (\mone_{R^*_m}  - g_{R,m})\right\|_{L^2(I^D)} \notag \\
    &\quad + \sum_{m \in [M] } \left\|  g_{f,m} \otimes g_{R,m} - g_c(g_{f,m}(\cdot),g_{R,m}(\cdot))\right\|_{L^2(I^D)} \notag  \\
    & =: \sum_{m \in [M]} B_{1,m} + \sum_{m \in [M]} B_{2,m} +\sum_{m \in [M] } B_{3,m}. \label{ineq:decomp1}
\end{align}
We will bound $B_{1,m},B_{2,m}$ and $B_{3,m}$ for $m \in [M]$.
About $B_{1,m}$, 
the choice of $g_{f,m}$ gives 
\begin{align*}
    & B_{1,m} = \left\| ( f^*_m -  g_{f,m}) \otimes  \mone_{R^*_m} \right\|_{L^2(I^D)} = \left\| f^*_m -  g_{f,m} \right\|_{L^2(R_m^*)} \leq \mathrm{vol}(R_m^*) \delta_1.
\end{align*}
About $B_{2,m}$, similarly, the H\"older inequality yields
\begin{align*}
    B_{2,m} &=  \left\| g_{f,m}\otimes  (\mone_{R^*_m}  - g_{R,m})\right\|_{L^2(I^D)} \\
    &\leq \left\|   g_{f,m}\right\|_{L^\infty(I^D)}\left\| \mone_{R^*_m}  - g_{R,m} \right\|_{L^2(I^D)} \\
     & \leq (1 +  \delta_2)\delta_2. 
\end{align*}
About $B_{3,m}$, since $g_{f,m}$ and $g_{R,m}$ is a bounded function by $F$, we obtain 
\begin{align*}
    B_{3,m} \leq \|(x,x' \mapsto xx') \to g_c\|_{L^\infty([-F,F]^2)} \leq F^2 \delta_2.
\end{align*}

We combine the results about $B_{1,m},B_{2,m}$ and $B_{3,m}$.
Substituting the bounds for \eqref{ineq:decomp1} yields
\begin{align*}
    \|f^* - \hat{g}\|_{L^2(I^D)} &\leq \sum_{m \in [M]} \left\{\mathrm{vol}(R_m^*) \delta_1 + \delta_2 + \delta_2^{2} + F^2 \delta_2 \right\}\\
    & \leq \delta_1 +  M \delta_2 + M \delta_2^2 + M F^2 \delta_2,
\end{align*}
where the second inequality follows $\sum_{m \in [M]} \mathrm{vol}(R_m^*) = \mathrm{vol}(I^D) = 1$.
Set $\delta_1 = \varepsilon_1$ and $\delta_2 = C_F \varepsilon_2 /M$.  Then we obtain 
\begin{align*}
     \|f^* - \hat{g}\|_{L^2(I^D)} \leq \varepsilon_1 + \varepsilon_2/2 + \varepsilon_2^2/2 \leq \varepsilon_1 + \varepsilon_2.
\end{align*}
Adjusting the coefficients, we obtain the statement.
\end{proof}

We are now ready to prove Theorem \ref{thm:approx}.
\begin{proof}[Proof of Theorem \ref{thm:approx}]
    We note that $\mG(L,S,B)$ has an inclusion property, namely, for $L' \geq L, S' \geq S$ and $B' \geq B$, we obtain
    \begin{align*}
        \mG(L',S',B') \supseteq \mG(L,S,B).
    \end{align*}
    Applying Proposition \ref{prop:approx} and adjusting the coefficients yield the statement.
\end{proof}

\subsection{Combining the Bounds} \label{sec:entropy}

Here, we evaluate $\mV$ in \eqref{ineq:basic2} by the empirical process and its applications \citep{koltchinskii2006local, gine2015mathematical,schmidt2017nonparametric}.
We then combine the result with Theorem \ref{thm:approx}, obtaining Theorem \ref{thm:non-bayes}.
Recall that $F$ denotes an upper bound of $g \in \mG(L,S,B)$ by its definition.

\begin{proof}[Proof of Theorem \ref{thm:non-bayes}]
The proof starts with the basis inequality \eqref{ineq:basic2} and follows the following two steps: (i) apply the covering number bound for $\mV$ in \eqref{ineq:basic2}, and (ii) combine the results with the approximation result (Theorem \ref{thm:approx}) on $\mB$.

\textbf{Step (i). Covering bound for the cross term.}
We bound an expected term $ |\Ep[\mV]|=  |\Ep[\frac{2}{n}\sum_{i \in [n] } \xi_i (\hat{f}^{DL}(X_i) - f(X_i)) ]|$ by the covering number of $\mG(L,S,B)$.
We fix $\delta \in (0,1)$ and consider a covering set $\{g_j\}_{j=1}^N \subset \mG(L,S,B)$ for $N= \mN(\delta, \mG(L,S,B), \|\cdot\|_{L^\infty})$, that is, for any $g \in \mG(L,S,B)$, there exists $g_j$ with $j \in [N]$ such that $\|g - g_j\|_{L^\infty} \leq \delta$.
For $\hat{f}^{DL}$, let $\hat{j} \in [N]$ be such that $\|\hat{f}^{DL} - g_{\hat{j}}\|_{L^\infty} \leq \delta$ holds.
Then, we bound the expected term as
\begin{align*}
    &\left| \Ep\left[ \frac{2}{n}\sum_{i \in [n]}\xi_i (\hat{f}^{DL}(X_i) - f(X_i))\right]  \right|\\
    & \leq \left|  \Ep\left[ \frac{2}{n}\sum_{i \in [n]} \xi_i (\hat{f}^{DL}(X_i) - g_{\hat{j}}(X_i))  \right]\right| +  \left| \Ep\left[ \frac{2}{n}\sum_{i \in [n]} \xi_i (g_{\hat{j}}(X_i) - f(X_i)) \right]  \right|\\
    & \leq 2 \delta \Ep\left[ \frac{1}{n}\sum_{i \in [n]} |\xi_i|\right] +  \Ep\left[ \left( \frac{\|\hat{f}^{DL} - f\|_n + \delta}{\|g_{\hat{j}} - f\|_n }\right) \left|\frac{2}{n}\sum_{i \in [n]} \xi_i (g_{\hat{j}}(X_i) - f(X_i)) \right|\right]   \\
    & \leq 2\delta (2 \sigma^2/\pi)^{1/2}+  \Ep\left[ \left( \frac{\|\hat{f}^{DL} - f\|_n + \delta}{\|g_{\hat{j}} - f\|_n }\right) \left|\frac{2}{n}\sum_{i \in [n]} \xi_i (g_{\hat{j}}(X_i) - f(X_i)) \right|\right]\\
    & \leq 2\delta (2 \sigma^2/\pi)^{1/2}+  2\Ep\Biggl[ \left( \frac{\|\hat{f}^{DL} - f\|_n + \delta}{ \sqrt{n} }\right) \Biggl|\underbrace{\frac{\sum_{i \in [n]} \xi_i (g_{\hat{j}}(X_i) - f(X_i))}{ \sqrt{n}\|g_{\hat{j}} - f\|_n} }_{=: \eta_{\hat{j}}}\Biggr|\Biggr] \\
    & \leq 2\delta (2 \sigma^2/\pi)^{1/2} +  \frac{2(\Ep[\|\hat{f}^{DL} - f\|_n^2]^{1/2} + \delta)}{\sqrt{n}} \Ep[\eta_{\hat{j}}^2]^{1/2},
\end{align*}
where the second inequality follows $\|g_{\hat{j}} - f\|_n \leq \|\hat{f}^{DL} - f\|_n + \delta$, and the last inequality follows the Cauchy-Schwartz inequality.
With conditional on the observed covariates $X_1,...,X_n$, $\eta_{j}$ follows a centered Gaussian distribution with its variance $\sigma^2$, hence $\Ep[\eta_{\hat{j}}^2] \leq  \sigma^2\Ep[\max_{j \in [N]}\eta_{j}^2] \leq 3 \log N + 1$ by Lemma C.1 in \citep{schmidt2017nonparametric}.
Then, we have
\begin{align}
    \left| \Ep\left[ \mV\right]  \right| &\leq 2\delta (2 \sigma^2/\pi)^{1/2}+  \frac{2\sigma^2(\Ep[\|\hat{f}^{DL} - f\|_n^2]^{1/2} + \delta)}{\sqrt{n}} (3 \log \mN(\delta, \mG(L,S,B), \|\cdot\|_{L^\infty}) + 1)^{1/2} \notag \\
    &\leq c_N\delta +  \frac{2\sigma^2\Ep[\|\hat{f}^{DL} - f\|_n^2]^{1/2} }{\sqrt{n}} (3 \log \mN(\delta, \mG(L,S,B), \|\cdot\|_{L^\infty}) + 1)^{1/2}, \label{ineq:var9}
\end{align}
where $c_N > 0$ is a constant.
The last inequality with $c_N$ follows $\mN(\delta, \mG(L,S,B), \|\cdot\|_{L^\infty}) / \sqrt{n} = O(1)$.

\textbf{Step (ii). Combine the results.}
We combine the results the bound for $\mB$ by Theorem \ref{thm:approx} and $|\Ep[\mV]|$ in the Step (i), then evaluate $\Ep[\|\hat{f}^{DL} - f^*\|_{L^2(P_X)}]$. 
Combining the bound \eqref{ineq:var9} with \eqref{ineq:basic2} yields that
\begin{align*}
    &\Ep \left[\|\hat{f}^{DL} - f^*\|^2_{n}\right] \\
    &\leq \Ep[\mB] + c_N \delta +  \frac{2\sigma^2\Ep[\|\hat{f}^{DL} - f\|_n^2]^{1/2} }{\sqrt{n}} (3 \log \mN(\delta, \mG(L,S,B), \|\cdot\|_{L^\infty}) + 1)^{1/2}.
\end{align*}
For any $a,b,c \in \R$, $a \leq b + c\sqrt{a}$ implies $a^2 \leq c^2 + 2b$.
Hence, we obtain
\begin{align}
    &\Ep \left[\|\hat{f}^{DL} - f^*\|^2_{n}\right] \notag \\ &\leq 2\Ep[\mB] + 2c_N \delta +  \frac{12 \sigma^2\log \mN(\delta, \mG(L,S,B), \|\cdot\|_{L^\infty}) + 4 }{{n}} \notag \\
    &= 2 \Ep\left[ \|\hat{g} - f^*\|_{n}^2\right]+ 2c_N \delta +  \frac{12\sigma^2\log \mN(\delta, \mG(L,S,B), \|\cdot\|_{L^\infty}) + 4 }{{n}}, \label{ineq:var8}
\end{align}
by following the definition of $\mB$. % \leq \|\hat{g} - f^*\|_n^2$.
Here, we apply the inequality (I) in the proof of Lemma 4 of \cite{schmidt2017nonparametric} with $\varepsilon = 1$ and apply \eqref{ineq:var8} as
\begin{align*}
     &\Ep \left[\|\hat{f}^{DL} - f^*\|^2_{L^2(P_X)}\right]\\
     &\leq 2 \left\{\Ep\left[ \|\hat{f}^{DL} - f^*\|_n^2\right] + \frac{2 F^2}{n }(12 \log \mN(\delta,\mG(L,S,B),\|\cdot\|_\infty) + 70) + 26 \delta F \right\}\\
     & \leq 4\Ep\left[ \|\hat{g} - f^*\|_{n}^2\right] + (52F + 4c_N) \delta  \\
     & \quad +  \frac{24(\sigma^2+2F^2) \log \mN(\delta, \mG(L,S,B), \|\cdot\|_{L^\infty}) + 8 + 280F^2 }{{n}}.
\end{align*}
Substituting the covering bound in Lemma \ref{lem:entropy_DNN} yields
\begin{align}
    &\Ep\left[\|\hat{f}^{DL} - f^*\|^2_{L^2(P_X)}\right]\notag \\
    &\leq  4 B_P\|\hat{g} - f^*\|_{L^2(I^D)}^2 +\frac{24(\sigma^2+2F^2) S }{n} \left\{  \left(\log (nL B^L (S+1)^L) \vee 1\right) \right\}^2 \notag  \\
    &\quad  + \frac{8 + 2c_N + 26F +  280F^2}{n}, \label{ineq:bound_all_pre}
\end{align}
by setting $\delta = 1/(2n)$.
Here, $p_X$ is a density of $P_X$ and $\sup_{x \in I^D}p_X(x) \leq B_P$ is finite by the assumption.
About the last inequality, we apply the following 
\begin{align}
    \Ep\left[ \|\hat{g} - f^*\|_n^2 \right] 
    &=  \int_{I^D} (\hat{g} - f^*)^2 d \lambda \frac{dP_X}{d\lambda} \leq \|\hat{g} - f^*\|_{L^2(I^D)}^2 \sup_{x \in I^D}p_X(x) \label{ineq:bound_l2p}
\end{align}
by the H\"older's inequality.  

At last, we substitute the result of Theorem \ref{thm:approx}, and then adjust the coefficients.
For $\varepsilon$ in the statement of Theorem \ref{thm:approx}, we set 
\begin{align*}
    \varepsilon_1 = n^{-\beta/(2\beta + D)} ,\mbox{~and~}\varepsilon_2 = M n^{-\alpha/2(\alpha + D-1)} ,
\end{align*}
then, we rewrite the condition of Theorem \ref{thm:approx} as
\begin{align*}
    &L \geq C_{\alpha,\beta,D,F}(1 + \lfloor\alpha \rfloor + \lfloor \beta \rfloor + \log_2 n),\\
    &S \geq C_{\alpha,\beta,D,F,J} (M n^{D/(2\beta + D)}  +  n^{(D-1)/(\alpha + D-1)}) \log^2 n,
\end{align*}
and $B \geq C_{F,M,q} n^{C_{\alpha,\beta,D}}$.
Then, substitute them into \eqref{ineq:bound_all_pre} and obtain
\begin{align*}
    &\Ep \left[\|\hat{f}^{DL} - f^*\|^2_{L^2(P_X)} \right]\\
    &\leq  4B_P( n^{-2\beta/(2\beta + D)} +  n^{-\alpha/(\alpha + D - 1)})\\
    & +\frac{C_{\alpha,\beta,D,F,J}  (\sigma^2 + F^2)  }{n}( Mn^{D/(2\beta + D)} +  M n^{(D - 1)/(\alpha + D - 1)})  C_{\alpha,\beta,D,F} \log^2 (n)   + \frac{C_F}{n}\\
    &\leq C_{\sigma,\alpha,\beta,D,F,J,P_X}  ( M n^{-2\beta/(2\beta + D)} +  M n^{-\alpha/(\alpha + D - 1)})\log^2n  + \frac{C_F}{n}.
\end{align*}
Then, we adjust the coefficients and obtain the statement.
\end{proof}

We provide a lemma which provides an upper bound for a covering number of $\mG(L,S,B)$.
Although similar results are well studied in several studies \citep{anthony2009neural,schmidt2017nonparametric}, we cite the following result in \cite{nakada2019adaptive}, which is more suitable for our result.
\begin{lemma}[Covering Bound: Lemma 22 in \cite{nakada2019adaptive}] \label{lem:entropy_DNN}
    For any $\varepsilon > 0$, we have
    \begin{align*}
        \log \mN(\varepsilon, \mG(L,S,B), \|\cdot \|_{L^\infty(I^D)}) \leq S \log \left( \frac{2 LB^L (S+1)^L}{\varepsilon} \right).
    \end{align*}
\end{lemma}

\section{Proof of Theorem \ref{thm:minimax}}

We first provide a proof of Proposition \ref{prop:packing_subset}, and then prove Theorem \ref{thm:minimax} by applying Theorem \ref{thm:yangbarron}.

\begin{proof}[Proof of Proposition \ref{prop:packing_subset}]
We give an upper bound and lower bound separately.

\textbf{ (i) Upper bound}:
First, we bound the packing number by a covering number as
\begin{align*}
    \log \mM(\varepsilon, H^\beta(I^D) \otimes \mI_\alpha, \|\cdot \|_{L^2(I^D)}) \leq \log \mN(\varepsilon/2, H^\beta(I^D) \otimes \mI_\alpha, \|\cdot \|_{L^2(I^D)}),
\end{align*}
by Section 2.2 in \cite{van1996weak}.

Further, we decompose the covering number of $H^\beta(I^D) \otimes \mI_\alpha$.
Let $\{f_j\}_{j=1}^{N_1} \subset H^\beta(I^D)$ be a set of centers of covering points with $N_1 = \mN(\varepsilon/2, H^\beta(I^D), \|\cdot \|_{L^2(I^D)})$, and $\{\iota_j\}_{j=1}^{N_2} \subset \mI_\alpha$ be centers of $\mI_\alpha$ as $N_2 = \mN(\varepsilon/2, \mI_\alpha, \|\cdot \|_{L^2(I^D)})$.
Then, we consider a set of points $\{f_j \otimes \iota_{j'}\}_{j,j'=1}^{N_1, N_2} \subset H^\beta(I^D) \otimes \mI_\alpha$ whose cardinality is $N_1N_2$.
Then, for any element $f \otimes \mone_R \in H^\beta(I^D) \otimes \mI_\alpha$, there exists $f_j \in H^\alpha(I^D)$ and $\iota_{j'}$ such as $\|f-f_j\|_{L^2(I^D)} \vee \|\mone_R-\iota_{j'}\|_{L^2(I^D)} \leq \varepsilon/2$ due to the property of covering centers.
Also, we can obtain
\begin{align*}
    &\|f \otimes \mone_R - f_j \otimes  \iota_{j'}\|_{L^2(I^D)} \\
    &\leq \|f \otimes (\mone_R - \iota_{j'}) \|_{L^2(I^D)} +  \|(f - f_j) \otimes  \iota_{j'} \|_{L^2(I^D)} \\
    & \leq \|f\|_{L^\infty(I^D)}\| \mone_R - \iota_{j'} \|_{L^2(I^D)} + \|f - f_j \|_{L^2(I^D)}\| \iota_{j'} \|_{L^\infty(I^D)} \\
    & \leq \frac{(F +1) \varepsilon}{2},
\end{align*}
where the second inequality follows the H\"older's inequality.
By this result, we find that the set $\{f_j \otimes \iota_{j'}\}_{j,j'=1}^{N_1, N_2}$ is a $(F+1)\varepsilon/2$ covering set of $H^\beta(I^D) \otimes \mI_\alpha$.
Hence, we obtain
\begin{align}
    &\log \mM((F+1)\varepsilon/2, H^\beta(I^D) \otimes \mI_\alpha, \|\cdot \|_{L^2(I^D)}) \notag \\
    & \leq \log \mM(\varepsilon/2, H^\beta(I^D) , \|\cdot \|_{L^2(I^D)}) + \log \mM(\varepsilon/2, \mI_\alpha , \|\cdot \|_{L^2(I^D)}). \label{ineq:ent_joint}
\end{align}

We will bound the two entropy terms for $H^\beta(I^D)$ and $\mI_\alpha$.
For $H^\beta(I^D)$, Theorem 8.4 in \citep{dudley2014uniform} provides
\begin{align}
    \log \mN(\varepsilon/2, H^{\beta}(I^D),\|\cdot\|_{L^2(I^D)}) &\leq \log \mN(\varepsilon/2, \notag  H^{\beta}(I^D),\|\cdot\|_{L^\infty(I^D)}) \\
    &= C_H (\varepsilon/2)^{-D / \beta}, \label{ineq:ent_h}
\end{align}
with a constant $C_H > 0$.
About the covering number of $\mI_{\alpha}$, we use the relation
\begin{align*}
    \|\mone_R - \mone_{R'}\|_{L^2}^2 &= \int_{I^D} (\mone_R(x) - \mone_{R'}(x))^2 dx = \int |\mone_R(x) - \mone_{R'}(x)| dx \\
    &=\int_{I^D}(\mone_{R\cup R'}(x) - \mone_{R' \cap R'}(x))  dx =: d_1(R,R'),
\end{align*}
for basic sets $R,R' \subset I^D$, and $d_1$ is a difference distance with a Lebesgue measure for sets by \cite{dudley1974metric}.
Here, we consider a \textit{boundary fragment class} $\Tilde{\mR}_\alpha$ defined by \cite{dudley1974metric}, which is a set of subset of $I^D$ with $\alpha$-smooth boundaries.
Since $R \subset I^D$ such that $\mone_R \in \mI_\alpha$ is a basis piece, we can see $R \in \Tilde{\mR}_\alpha$.
Hence, we obtain
\begin{align}
    &\log \mM(\varepsilon/2, \mI_\alpha , \|\cdot \|_{L^2(I^D)}) \notag \\
    &=\log \mN(\varepsilon^2/16, \{R \subset I^D \mid \mone_R \in \mI_\alpha\}, d_1) \notag \\
    &\leq \log \mN(\varepsilon^2/16, \Tilde{\mR}_\alpha, d_1)\notag \\
    &= C_{\lambda} (\varepsilon/2)^{-2(D-1)/\alpha}, \label{ineq:ent_i}
\end{align}
with a constant $C_\lambda > 0$.
Here, the last equality follows Theorem 3.1 in \citep{dudley1974metric}.

Combining \eqref{ineq:ent_h} and \eqref{ineq:ent_i} with \eqref{ineq:ent_joint}, we obtain
\begin{align*}
    &\log \mM((F+1)\varepsilon/2, H^\beta(I^D) \otimes \mI_\alpha, \|\cdot \|_{L^2(I^D)}) \\
    &\leq C_H (\varepsilon/2)^{-D / \beta} + C_{\lambda} (\varepsilon/2)^{-2(D-1)/\alpha}.
\end{align*}
Adjusting the coefficients yields the upper bound.

\textbf{(ii) Lower bound}:
We provide the lower bound by evaluating the packing number $ \log \mM(\varepsilon, H^\beta(I^D) \otimes \mI_\alpha, \|\cdot \|_{L^2(I^D)})$ directory.
Let $1(x)$ be a constant function $1(x) := 1, \forall x \in I^D$.

Here, we claim $ \mI_\alpha \subset H^\beta(I^D) \otimes \mI_\alpha$, because $1 \in H^\beta(I^D)$ yields that $H^\beta(I^D) \otimes \mI_\alpha \supset  \{1\} \otimes \mI_\alpha = \mI_\alpha$.
Hence, we have
\begin{align}
     \log \mM(\varepsilon, H^\beta(I^D) \otimes \mI_\alpha, \|\cdot \|_{L^2(I^D)}) &\geq  \log \mM(\varepsilon, \mI_\alpha, \|\cdot \|_{L^2(I^D)}) \notag \\
     &= C_{\lambda} \varepsilon^{-2(D-1)/\alpha}, \label{ineq:ent_low1}
\end{align}
by Theorem 3.1 in \cite{dudley1974metric}.
Similarly, $1 = \mone_{I^D} \in \mI_\alpha$ yields $ H^\beta(I^D) \subset H^\beta(I^D) \otimes \mI_\alpha$.
Hence, we achieve
\begin{align}
     \log \mM(\varepsilon, H^\beta(I^D) \otimes \mI_\alpha, \|\cdot \|_{L^2(I^D)}) &\geq  \log \mM(\varepsilon,H^\beta(I^D), \|\cdot \|_{L^2(I^D)}) \notag \\
     &= C_H \varepsilon^{-D / \beta}, \label{ineq:ent_low2}
\end{align}
by Theorem 8.4 in \cite{dudley2014uniform}.
Combining \eqref{ineq:ent_low1} and \eqref{ineq:ent_low2}, we obtain
\begin{align*}
    &\log \mM(\varepsilon, H^\beta(I^D) \otimes \mI_\alpha, \|\cdot \|_{L^2(I^D)}) \geq \max\{ C_H \varepsilon^{-D / \beta},C_{\lambda} \varepsilon^{-2(D-1)/\alpha} \}.
\end{align*}
Adjusting the coefficients, we obtain the statement.
\end{proof}

Now, we are ready to prove Theorem \ref{thm:minimax}.

\begin{proof}[Proof of Theorem \ref{thm:minimax}]
In this proof, we develop a lower bound of
\begin{align*}
    \inf_{\Bar{f}}\sup_{f^* \in  H^\beta(I^D) \otimes \mI_\alpha} \Ep_{f^*} \left[ \|\bar{f} - f^*\|_{L^2(P_X)}^2\right],
\end{align*}
where $\Bar{f}$ is any estimator. Then the statement is immediate because of the following inequality
\begin{align}
    \sup_{f^* \in \mathcal{\mF}_{\alpha,\beta,M}^{PS}} \Ep_{f^*} \left[ \|\bar{f} - f^*\|_{L^2(P_X)}^2\right] \geq \sup_{f^* \in  H^\beta(I^D) \otimes \mI_\alpha} \Ep_{f^*} \left[ \|\bar{f} - f^*\|_{L^2(P_X)}^2\right] \label{ineq:minimax_comp}
\end{align}
due to $H^\beta(I^D) \otimes \mI_\alpha \subset \mathcal{\mF}_{\alpha,\beta,M}^{PS}$.

To develop the lower bound with $H^\beta(I^D) \otimes \mI_\alpha$, we apply Theorem \ref{thm:yangbarron}.
Let $\varepsilon_n^*$ be a sequence for $n \in \N$ which satisfies
\begin{align*}
    (\varepsilon_n^*)^2 = \log \mM(\varepsilon_n^*, H^\beta(I^D) \otimes \mI_\alpha, \|\cdot\|_{L^2(I^D)})/n.
\end{align*}
From Proposition \ref{prop:packing_subset}, we obtain
\begin{align*}
     (\varepsilon_n^*)^2 = \Theta\left((\varepsilon_n^*)^{-D/\beta} + (\varepsilon_n^*)^{-2(D-1)/\alpha}\right)/n.
\end{align*}
Solving this equation gives 
\begin{align}
    (\varepsilon_n^*)^2 = \Theta\left( n^{-2\beta/(2\beta + D)} + n^{-\alpha/(\alpha + D-1)}\right). \label{eq:eps}
\end{align}
Application of Theorem \ref{thm:yangbarron} with  \eqref{eq:eps} derives
\begin{align*}
    \inf_{\Bar{f}}\sup_{f^* \in  H^\beta(I^D) \otimes \mI_\alpha} \Ep_{f^*} \left[ \|\bar{f} - f^*\|_{L^2(P_X)}^2\right] &= \Theta(\varepsilon_n^2) \\
    &=\Theta\left( n^{-2\beta/(2\beta + D)} + n^{-\alpha/(\alpha + D-1)}\right),
\end{align*}
which with \eqref{ineq:minimax_comp} yields the claim.
\end{proof}

\section{Approximation Results of DNNs with General Admissible Activation}

First, we show that the activation functions with Assumption \ref{asmp:activation} is suitable for DNNs to approximate polynomial functions, including an identity function.

\begin{lemma} \label{lem:approx_poly1}
    Suppose $\eta$ satisfies the condition $(i)$ in Assumption \ref{asmp:activation}.
    Then, for $\gamma \in \N \cup \{0\}$ with $\gamma \leq N+1$, any $\varepsilon > 0$ and $T>0$, there exists a tuple $(L,S,B)$ such that
    \begin{itemize}
        \item $L=2$,
        \item $S=3(\gamma + 1)$,
        \item $B=C_{\gamma,T} \varepsilon^{-C_\gamma}$,
    \end{itemize}
    and it satisfies
    \begin{align*}
        \inf_{g \in \mG( L, S,B )} \|g-(x \mapsto x^\gamma)\|_{L^\infty([-T,T])} \leq  \varepsilon.
    \end{align*}
%    where $s>0$ is a constant depending on $\gamma$ and $B$.
\end{lemma}
\begin{proof}[Proof of Lemma \ref{lem:approx_poly1}]
Consider the following neural network with one layer:
\begin{align}
    g_{sp}(x) := \sum_{j=1}^{\gamma + 1} a_{2,j} \eta (a_{1,j} x + b_j).\label{eq:one-layer}
\end{align}
Since $\eta$ is $N$-times continuously differentiable by the condition $(i)$ in Assumption \ref{asmp:activation}, we set $b_j = x'$ for $j=1,...,\gamma+1$ and consider the Taylor expansion of $\eta$ around $b_j = x'$.
Then, for $j=1,...,\gamma+1$, we obtain
\begin{align*}
    \eta (a_{1,j}x + b_j) = \sum_{k=0}^\gamma \frac{\partial^k \eta(x') a_{1,j}^k x^k}{k!} + \frac{\partial^{\gamma+1} \eta(\Bar{x}) a_{1,j}^{\gamma+1}x^{\gamma + 1}}{(\gamma+1)!},
\end{align*}
with some $\Bar{x}$.
We substitute it into \eqref{eq:one-layer} and obtain
\begin{align}
    &\sum_{j=1}^{\gamma + 1} a_{2,j} \left( \sum_{k=0}^\gamma \frac{\partial^k \eta(x') a_{1,j}^k x^k}{k!} + \frac{\partial^{\gamma+1} \eta(\Bar{x}) a_{1,j}^{\gamma+1}x^{\gamma + 1}}{(\gamma+1)!} \right) \notag \\
    &= \sum_{k=0}^\gamma  \frac{\partial^k \eta(x')x^k}{k!} \sum_{j'=0}^\gamma a_{2,j'+1} a_{1,j'+1}^k + \sum_{j=1}^{\gamma + 1} a_{2,j} \frac{\partial^{\gamma+1} \eta(\Bar{x}) a_{1,j}^{\gamma+1}x^{\gamma + 1}}{(\gamma+1)!} . \label{eq:one-layer2}
\end{align}
For each $j'=0,...,\gamma$, we set $a_{1,j'+1} = j'/\bar{a}$ with $\Bar{a} > 0$ and $a_{2,j'+1} = (-1)^{-\gamma + j' } \bar{a}^\gamma \binom{\gamma}{j'} / \partial^{\gamma}\eta(x')$.
Note that $\partial^\gamma \eta(x') > 0$ holds by Assumption \ref{asmp:activation}.
Then, we obtain the following equality:
\begin{align*}
    \sum_{j'=0}^\gamma a_{2,j'+1} a_{1,j'+1}^k& =  \sum_{j'=0}^\gamma (-1)^{-\gamma + j'} \bar{a}^{\gamma-k} \binom{\gamma}{j'}\frac{{j'}^k}{\partial^{\gamma}\eta(x')}\\
    &= \frac{\bar{a}^{\gamma-k}}{(-1)^{\gamma }\partial^{\gamma}\eta(x')}\sum_{j'=0}^\gamma (-1)^{ j'}  \binom{\gamma}{j'}{j'}^k\\
    &=  
    \begin{cases}
        \frac{\bar{a}^{\gamma-k}}{(-1)^{\gamma }\partial^{\gamma}\eta(x')} \gamma!~(-1)^\gamma,&\mbox{~if~}k=\gamma\\
        0&\mbox{~otherwise},
    \end{cases}
\end{align*}
where the last equality follows the Stirling number of the second kind (described in \citep{graham1989concrete}).
Substituting it for \eqref{eq:one-layer2} yields
\begin{align*}
    \sum_{j=1}^{\gamma + 1} a_{2,j} \eta (a_{1,j} x + b_j)&=x^\gamma  +  \sum_{j=1}^{\gamma + 1} \frac{(-1)^{-\gamma + j-1 } \partial^{\gamma+1} \eta(\Bar{x}) (j-1)^{\gamma + 1} x^{\gamma + 1} }{ \partial^\gamma \eta(x')\bar{a} (\gamma+1)!}  \binom{\gamma}{j-1}\\
    &=x^\gamma  + \frac{ \partial^{\gamma+1} \eta(\Bar{x}) x^{\gamma + 1}}{\partial^\gamma \eta(x')\bar{a} (\gamma+1)!} \sum_{j=1}^{\gamma + 1} (-1)^{-\gamma + j-1 } (j-1)^{\gamma + 1}  \binom{\gamma}{j-1}.
\end{align*}
Regarding the second term, we obtain
\begin{align*}
    &\left|   \frac{ \partial^{\gamma+1} \eta(\Bar{x}) x^{\gamma + 1}}{\partial^\gamma \eta(x')\bar{a} (\gamma+1)!} \sum_{j=1}^{\gamma + 1} (-1)^{-\gamma + j-1 } (j-1)^{\gamma + 1}  \binom{\gamma}{j-1}\right|\\
    &\leq \frac{ T^{\gamma + 1}}{\Bar{a} c_\eta} \|\partial^{\gamma + 1} \eta\|_{L^\infty([-B,B])} \gamma^{\gamma + 1}(1+e^\gamma) =: \frac{C_{T,\gamma,\eta}}{\Bar{a}}.
\end{align*}
As setting $\Bar{a} = C_{T,\gamma,\eta}/\varepsilon$, we obtain the approximation with $\varepsilon$-error.
From the result, we know that $g_{sp} \in \mG(2,3(\gamma + 1), C_{T,\gamma} \varepsilon^{-C_{\gamma}})$.
\end{proof}

\begin{lemma} \label{lem:approx_poly2}
    Suppose $\eta$ satisfies the condition $(ii)$ in Assumption \ref{asmp:activation}.
    Then, for $\gamma \in \N \cup \{0\}$ with $\gamma \leq N+1$, any $\varepsilon > 0$ and $T>0$, there exists a tuple $(L,S,B)$ such that
    \begin{itemize}
        \item $L=(\gamma + 1)(\log_2(C_{\gamma,T}/\varepsilon)/2+1)$,
        \item $S=C_{\gamma,T} ((\log_2(1/\varepsilon))^2 + \log_2(1/\varepsilon))$,
        \item $B=C_{\gamma,T} \varepsilon^{-C_\gamma}$,
    \end{itemize}
    and it satisfies
    \begin{align*}
        \inf_{g \in \mG( L, S,B )} \|g-(x \mapsto x^\gamma)\|_{L^\infty([-T,T])} \leq  \varepsilon.
    \end{align*}
\end{lemma}
\begin{proof}[Proof of Lemma \ref{lem:approx_poly2}]
As a preparation, we construct a \textit{saw-tooth function} by \cite{yarotsky2017error} with our Assumption \ref{asmp:activation}.
Let us define a \textit{teeth function} $g_w : [0,1] \to [0,1]$ by a difference of two $\eta$ as
\begin{align*}
    g_w(x) &:= \frac{2 c_2 + 2 c_1}{c_2(c_2 - c_1)} \eta(x) - \frac{4}{ c_2 - c_1} \eta(x-1/2) - \frac{2 \overline{c}}{ c_2 - c_1} \\
    &= 
    \begin{cases}
    2x, & \mbox{~if~}x \in [0,1/2],\\
    -2x + 2, & \mbox{~if~}x \in [1/2,1].
    \end{cases}
\end{align*}
Then, we consider the $t$-hold composition of $g_w \in \mG(2,6,c_w)$ with $c_w > 0$ as $g_t = g_w \circ \cdots \circ g_w \in \mG(t+1, 3(t+1), c_w )$ which satisfies
\begin{align*}
    g_t(x) = 
    \begin{cases}
        2^t(x-2k2^{-t}), \mbox{~if~} x \in [2k 2^{-t}, (2k+1)2^{-t}], k = 0,1,...,2^{t-1}-1\\
        -2^t(x-2k2^{-t}), \mbox{~if~} x \in [(2k-1) 2^{-t}, 2k2^{-t}], k = 1,...,2^{t-1}.
    \end{cases}
\end{align*}
Here, the domain $[0,1]$ of $g_t$ is divided into $2^{t+1}$ sub-intervals.

Then, we approximate a quadratic function by a linear sum of $g_t$.
For $m \in \N$, We define a function $g_m:[0,1] \to [0,1]$ as
\begin{align*}
    g_m(x) := x-\sum_{t=1}^m \frac{g_t(x)}{2^{2t}},
\end{align*}
and it satisfies $\|g_m - (x \mapsto x^2)\|_{L^\infty([0,1])} \leq 2^{-2-2m}$ for $g_m \in \mG(m+1, 3m^2/2 + 5m/2 + 1,c_m)$ with $c_m > 0$, by the similar way in Proposition 3 in \citep{yarotsky2017error}.

Further, we approximate a multiplicative function by $f_m$, intuitively, we represent a multiplication by a sum of a quadratic function as $xx' = \{|x+x'|^2 - |x|^2 - |x'|^2\} /2$.
To the aim, we define an absolute function $g_a:[-1,1] \to [0,1]$ by DNNs as
\begin{align*}
    g_a(x) := (\eta(x)-\eta(-x))/(c_2 + c_1) = |x|.
\end{align*}
Then, we define $g_c:[-T,T]\times [-T,T] \to \R$ as
\begin{align}
    g_c(x,x') := \frac{T^2}{2} \left\{  g_m \circ g_a ((x + x')/T) - g_m \circ g_a (x/T) - g_m \circ g_a (x'/T) \right\}. \label{ineq:multi}
\end{align}
By the similar proof in Proposition 3 in \cite{yarotsky2017error}, we obtain $\|((x,x') \mapsto xx') - g_c(x,x')\|_{L^\infty([-T,T]^2)} \leq T^2 2^{-2m}$ with $g_c \in \mG(m+1,(9m^2 + 15m)/2 + 10 , c_c)$ with $c_c > 0$. 

Finally, we approximate the polynomial function $x^\gamma$ by an induction by $g_c$.
When $\gamma = 1$, we consider
\begin{align*}
    g_{p,1}(x) := (c_2 + c_1)^{-1} (\eta(x) - \eta(-x)) = x,
\end{align*}
then obviously $\|g_{p,1}-(x \mapsto x)\|_{L^\infty([-T,T])} = 0$ holds with $g_{p,1} \in \mG(2,3,c_{p,1})$ with a constant $c_{p,1} > 0$.

Now, for the induction, assume that there exists a function by DNNs $g_{p,\gamma-1} \in \mG(\gamma(m+1), c_{\gamma-1,T} ((9m^2 + 15m)/2 + 10), c_p)$ with constants $c_p > 0$ and $C_{\gamma-1,T}$ depending $\gamma-1$ and $T$, and it satisfies $\|g_{p,\gamma-1} - (x \mapsto x^{\gamma-1})\|_{L^\infty([-T,T])} \leq C_{\gamma-1,T}2^{-2m}$ with $C_{\gamma-1,T} > 0$.
Also, we set $g_m$ as \eqref{ineq:multi}.
Then, we consider the following function by DNNs as
\begin{align*}
    g_{p,\gamma}(x) := g_m(g_{p,\gamma-1}(x),g_{p,1}(x)).
\end{align*}
Then, we consider the following approximation error
\begin{align*}
    &\|g_{p,\gamma} - (x \mapsto x^\gamma)\|_{L^\infty([-T,T])} \\
    &= \|g_m(g_{p,\gamma-1}(\cdot),g_{p,1}(\cdot)) - (x \mapsto x^{\gamma-1}) \otimes (x \mapsto x)\|_{L^\infty([-T,T])}\\
    &\leq \|g_m(g_{p,\gamma-1}(\cdot),g_{p,1}(\cdot)) - g_{p,\gamma-1} \otimes g_{p,1}\|_{L^\infty([-T,T]^2)}\\
    &\quad  + \| g_{p,\gamma-1} \otimes g_{p,1} - (x \mapsto x^{\gamma-1}) \otimes g_{p,1}\|_{L^\infty([-T,T])}\\
    &\quad  + \| (x \mapsto x^{\gamma-1}) \otimes g_{p,1} - (x \mapsto x^{\gamma-1}) \otimes (x \mapsto x)\|_{L^\infty([-T,T])}\\
    & \leq 2T^2 2^{-2m} + 2T \|g_{p,\gamma-1} - (x \mapsto x^{\gamma-1})\|_{L^\infty([-T,T])} + 0\\
    &\leq 2T^2 2^{-2m} + 2TC_{\gamma-1,T}2^{-2m} \\
    &= (2T^2 + 2TC_{\gamma-1,T})2^{-2m} =: C_{\gamma,T}2^{-2m}.
\end{align*}
Then, we obtain the statement with the condition $(ii)$ in Assumption \ref{asmp:activation}.

We combine the result with the conditions $(i)$ and $(ii)$, and obtain the statement with $\varepsilon = 2^{-2m}$.
\end{proof}

\begin{lemma}[General version of Lemma \ref{lem:main_step}] \label{lem:step_approx}
    Suppose $\eta$ satisfies Assumption \ref{asmp:activation}.
    Then, for any $\varepsilon \in (0,1)$ and $T>0$, we obtain
    \begin{align*}
        \inf_{g \in \mG(2,6,C_{T,q}\varepsilon^{-8\vee (2q-1/2)})} \|g - \mone_{\{\cdot \geq 0\}}\|_{L^2([-T,T])} \leq \varepsilon.
    \end{align*}
\end{lemma}
\begin{proof}[Proof of Lemma \ref{lem:step_approx}]
First, we consider $\eta$ that satisfies the condition $(i)$ in Assumption \ref{asmp:activation}.
Without loss of generality, we set $\overline{c}=1$ and $\underline{c}=0$.
We start with $\eta$ with $k=0$, where $k$ is given in Assumption 1.
Let us consider a shifted activation $\eta(ax)$ with $a>0$.
Then, the difference between $\eta(ax)$ and  $\mone_{\{\cdot \geq 0\}}$ is decomposed as
\begin{align*}
    |\eta(ax) - \mone_{\{\cdot \geq 0\}}| \leq 
    \begin{cases}
        c_\eta(a^{-q} x^{-q} \wedge 1) , & \mbox{~if~} x > 0, \\
        c_\eta(a^{-q}|x|^{-q} \wedge 1) , & \mbox{~if~} x < 0,
    \end{cases}
\end{align*}
with an existing constant $c_\eta > 0$.
The upper bound by $1$ comes from the uniform bound by $C_K$ in Assumption \ref{asmp:activation}.
Hence, the difference of them in terms of $L^2(\R)$ norm is
\begin{align*}
    \|\eta(a \cdot ) -\mone_{\{\cdot \geq 0\}} \|_{L^2(\R)} &\leq \left(2c_\eta\int_{0}^\infty ( a^{-2q}x^{-2q} \wedge 1 ) dx \right)^{1/2}\\
    &= \left(\frac{2 c_\eta}{a} + \frac{2 c_\eta}{(2q-1)a^{4q-1}} \right)^{1/2}.
\end{align*}
As we set $a \geq (4 c_\eta / \varepsilon^2) \vee ( 4 c_\eta (2q-1)/ \varepsilon^2)^{4q-1}$, we obtain the statement with $\mG(2,1,C_{q}\varepsilon^{-(2q-1/2)})$.

We consider $\eta$ with $k= 1$.
With a scale parameters $a$ and a shift parameters $\delta/2 > 0$, we consider a difference of two $\eta$ with difference shift as 
\begin{align*}
    g_s(x;a,\delta) = \eta(ax - \delta/2) - \eta(ax + \delta/2).
\end{align*}
Let $t \in (0,B)$ be a parameter for threshold.
When $x > t$, the property of $\eta$ yields
\begin{align*}
    g_s(x;a,\delta) \in [ \delta \pm 2c_\eta |ax-\delta/2|^{-q}].
\end{align*}
We set $\delta = 1$, hence we obtain
\begin{align}
    |g_s(x;a,1)-1| \leq 2c_\eta |ax-1/2|^{-q}, \label{ineq:step1}
\end{align}
for $x > t$.
Similarly, when $x < -t$, we have
\begin{align}
    &g_s(x;a,\delta) \in [0 \pm 2c_\eta |ax-\delta/2|^{-q}]. \label{ineq:step2}
\end{align}
When $x \in [-t,t]$, the bounded property of $\eta$ yields
\begin{align}
   g_s(x;a,\delta) \in [0  \pm  C_K (2+|ax-\delta/2| + |ax + \delta / 2|)]. \label{ineq:step3}
\end{align}
Combining the inequalities \eqref{ineq:step1}, \eqref{ineq:step2} and \eqref{ineq:step3} with $\delta = 1$, we bound the difference between $\eta(ax - \delta/2) - \eta(ax + \delta/2)$ and $\mone_{\{\cdot \geq 0\}}$ as
\begin{align*}
    &\|g_s(\cdot;a,1) - \mone_{\{\cdot \geq 0\}} \|_{L^2([-B,B])}^2 \\
    &\leq \|g_s(\cdot;a,1) \|_{L^2([-B,-t))}^2 + \|g_s(\cdot;a,1) - \mone_{\{\cdot \geq 0\}} \|_{L^2((t,B])}^2  \\
    &\quad +\|g_s(\cdot;a,1) - \mone_{\{\cdot \geq 0\}} \|_{L^2([-t,t])}^2\\
    &\leq 4c_\eta^2 \int_{[-B,-t)}|ax-1/2|^{-2q}dx + 4c_\eta^2 \int_{(t,B]}|ax-1/2|^{-2q}dx \\
    &\quad + C_K^2 \int_{[-t,t]} \{2 + |ax - 1/2| + |ax + 1/2|\}^2 dx\\
    &=: T_{s,1} + T_{s,2} + T_{s,3}.
\end{align*}
About $T_{s,2}$, simple calculation yields
\begin{align*}
    T_{s,2} \leq 4c_\eta^2 (B-t) |at - 1/2|^{-2q} \leq 4c_\eta^2 B|at - 1/2|^{-2q}.
\end{align*}
By the symmetric property, we can also obtain $T_{s,1} = 4c_\eta^2 B (at - 1/2)^{-2q}$.
About $T_{s,3}$, the similar calculation yields
\begin{align*}
    T_{s,3} \leq 8 C_K^2 t (1+|at + 1/2|)^2 \leq 8 C_K^2 t (at + 3/2)^2.
\end{align*}
Combining the results of the terms, we bound the norm $\|g_s(\cdot;a,1) - \mone_{\{\cdot \geq 0\}} \|_{L^2([-B,B])}^2$.
Here, we set $t = a^{-3/4}$ and $a \geq 1/16$, then we obtain
\begin{align*}
    \|g_s(\cdot;a,1) - \mone_{\{\cdot \geq 0\}} \|_{L^2([-B,B])}^2 &= 4 c_\eta^2 B a^{-3q/4} + 8 C_K^2 a^{-3/4}(a^{1/4}+3/2)^2\\
    &= 4 c_\eta^2 B a^{-3q/4} + 8 C_K^2 (a^{-1/4} + 3 a^{-1/2} + 9a^{-3/4}/4).
\end{align*}
Then, we set $a= 16c_\eta^2 B \varepsilon^{-8/3q} \vee 32 C_K^2 \varepsilon^{-8} \vee 96C_K^2 \varepsilon^{-4} \vee 9 \varepsilon^{-8/3} \vee 1/16$, then we obtain the statement, then we have $g_s \in \mG(2,6,C_{B,q}\varepsilon^{-8})$.

Second, we consider $\eta$ satisfies the condition $(ii)$ in Assumption \ref{asmp:activation}.
We consider a sum of two $\eta(x) = c_1x + (c_2 - c_1)x_+$ with some scale change as
\begin{align*}
    \eta(x) + \eta \left( -\frac{c_1}{c_2} x \right) &= ( c_1 x + (c_2-c_1)x_+ ) + \left( -c_1 x + \left( c_1 - \frac{c_1^2}{c_2} \right)x_+ \right)\\
    &= \left( c_2 - \frac{c_1^2}{c_2} \right)x_+ =: \eta_s(x:c_1,c_2).
\end{align*}
Then, we consider a difference of two $\eta_s(x:c_1,c_2)$ with shift change by $\delta/2 > 0$ and scale change $a > 0$ as
\begin{align*}
    \eta_d(x)&=\eta_s(ax - \delta/2:c_1,c_2) - \eta_s(ax + \delta/2:c_1,c_2)\\
    &=
    \begin{cases}
        \left( c_2 - \frac{c_1^2}{c_2} \right)\delta,&\mbox{~if~}x > \delta/2a,\\
        \left( c_2 - \frac{c_1^2}{c_2} \right)(ax-\delta/2),&\mbox{~if~}x \in [-\delta/2a,\delta/2a],\\
        0,&\mbox{~if~}x < -\delta/2a.
    \end{cases}
\end{align*}
Here, we set $\delta = ( c_2 - {c_1^2}/{c_2} )^{-1} < ( c_2 - c_1 )^{-1} < \infty$.
Then, the $L^2([-B,B])$-distance between $g_d(x)$ and $\mone_{\{\cdot \geq 0\}}$ is written as
\begin{align*}
    \|g_d - \mone_{\{\cdot \geq 0\}}\|_{L^2([-B,B])} = \left( \int_{[-\delta/2a,\delta/2a]} \left| \frac{ax}{\delta} \right|^2 dx \right)^{1/2}= \frac{\delta^{1/2}}{12^{1/2}a}.
\end{align*}
As we set $a = \delta^{1/2}/(12^{1/2}\varepsilon)$, we obtain the statement with $g_d \in \mG(2,6,C_{B,q}\varepsilon^{-1})$.

We combine the results with all the conditions and consider the largest functional set, and then we obtain the statement.
\end{proof}

\begin{lemma}[General Version of Lemma \ref{lem:main_smooth}] \label{lem:smooth_approx}
    Suppose Assumption \ref{asmp:activation} holds with $N \geq \alpha \vee \beta$.
    For any non-empty measurable set $R \subset I^D$ and $\delta > 0$, a tuple $(L,S,B)$ such as
    \begin{itemize}
        \item $L =  C_{\beta,D,F,J} (\lfloor \beta \rfloor + \log_2(1/\delta) + 1)$
        \item $  S =  C_{\beta,D,F,J}  \delta^{-D/\beta} (  \log_2(1/ \delta))^2 $
        \item $B= C_{\beta,D,F,J,q}  \delta^{-16 \wedge -C_\beta}$,
    \end{itemize}
    satisfies
    \begin{align*}
        &\inf_{g \in \mG(L,S,B)} \sup_{f \in H_F^\beta(I^D)} \|g - f\|_{L^2(R)}\leq  \mathrm{vol}(R)\delta.
    \end{align*}
\end{lemma}
\begin{proof}[Proof of Lemma \ref{lem:smooth_approx}]

Before the central part of the proof, we provide some preparation.
We divide the domain $I^D$ into several hypercubes.
Let $\ell \in \N$ and consider a $D$-dimensional multi-index $\lambda \in \{1,2,...,\ell\}^{D}=: \Lambda$.
For each $\lambda$, define a hypercube
\begin{align*}
    I_\lambda := \prod_{d = 1}^D \left[ \frac{\lambda_d -1}{\ell}, \frac{\lambda_d }{\ell} \right]
\end{align*}
By the definition, $I_\lambda$ is a $D$-dimensional hypercube whose side has a length $1/\ell$, and $\bigcup_{\lambda \in \Lambda} I_\lambda= I^D$.
Also, the center of $I_\lambda$ is denoted by $x_\lambda$; i.e.,
\begin{align*}
    x_\lambda := \left( \frac{\lambda_1 - 1/2}{\ell},\frac{\lambda_2 - 1/2}{\ell},...,\frac{\lambda_D - 1/2}{\ell} \right) \in I^D
\end{align*}

We provide a Taylor polynomial for a smooth function.
Fix $\lambda \in \Lambda$ and $f \in H_F^\beta(I^D)$ arbitrary.
Let $a \in \N^D$ be a multi-index.
Then, we consider the Taylor expansion of $f$ in $I_\lambda$ as
\begin{align*}
    f(x) &= f(x_\lambda) + \sum_{a: |a| \leq \lfloor \beta \rfloor -1 } \frac{\partial^a f(x_\lambda)}{a!}(x-x_\lambda)^a \\
    & \quad + \sum_{a: |a| = \lfloor \beta \rfloor } \frac{1}{a!}(x-x_\lambda)^a \int_0^1 (1-t)^{\lfloor \beta \rfloor-1} \partial^a  f(x_\lambda + t(x-x_\lambda))  dt\\
    &= f(x_\lambda) + \sum_{a: |a| \leq \lfloor \beta \rfloor -1 } \frac{\partial^a f(x_\lambda)}{a!}(x-x_\lambda)^a \\
    & \quad + \sum_{a: |a| = \lfloor \beta \rfloor } \frac{1}{a!}(x-x_\lambda)^a \int_0^1 (1-t)^{\lfloor \beta \rfloor -1 } \partial^a f(x_\lambda ) dt\\
    & \quad + \sum_{a: |a| = \lfloor \beta \rfloor } \frac{1}{a!}(x-x_\lambda)^a \int_0^1 (1-t)^{\lfloor \beta \rfloor-1}\{ \partial^a f(x_\lambda + t(x-x_\lambda))- \partial^a f(x_\lambda)\} dt\\ 
    &=f(x_\lambda) + \sum_{a: |a| \leq \lfloor \beta \rfloor } \frac{\partial^a f(x_\lambda)}{a!}(x-x_\lambda)^a \\
    & \quad + \sum_{a: |a| = \lfloor \beta \rfloor } \frac{1}{a!}(x-x_\lambda)^a \int_0^1 (1-t)^{\lfloor \beta \rfloor-1} \{ \partial^a f(x_\lambda + t(x-x_\lambda)) - \partial^a f(x_\lambda) \} dt\\
    &=: f(x_\lambda) +  f_{\lfloor \beta \rfloor }(x;x_\lambda) + R_\lambda(x),
\end{align*}
where $a! = \prod_{d \in [D]}a_d!$, 
 $f_{\lfloor \beta \rfloor -1}(x;x_\lambda)$ is the Taylor polynomial with an order $\lfloor \beta \rfloor$, and $R_\lambda(x)$ is the remainder.
By the H\"older continuity and the bounded property of $\partial^{a} f(x)$, the remainder $R_\lambda(x)$ is bounded as
\begin{align*}
    |R_\lambda(x)| &\leq  \sum_{a: |a| = \lfloor \beta \rfloor } \frac{F}{a!}|(x-x_\lambda)^a|\int_{0}^1 (1-t)^{\lfloor \beta \rfloor - 1}  |t(x-x_\lambda)|^{ \beta - \lfloor \beta \rfloor} dx\\
    & \leq F \sum_{a:|a|=\lfloor\beta\rfloor}(a!)^{-1} |x-x_\lambda|^\beta \leq C_{D,\beta,F} \left(\frac{1}{\ell} \right)^\beta,
\end{align*}
for $x \in I_\lambda$.
Here, $C_{F,\beta} > 0$ is a constant which depends on $F$ and $\beta$.
The last inequality follows since $\|x-x_\lambda\|_{\infty} \leq 1/\ell $ holds for any $x \in I_\lambda$.

Now, we approximate the Taylor polynomial $f_{\lfloor \beta \rfloor}(x;x_\lambda)$ by DNNs for each $\lambda \in \Lambda$.
By the binomial theorem, we can rewrite the Taylor polynomial with a multi-index $b \in \N^{D}$ and $x \in I_\lambda$ as
\begin{align*}
    f_{\lfloor \beta \rfloor}(x;x_\lambda) &= \sum_{a: |a| \leq \lfloor \beta \rfloor  } \frac{\partial^a f(x_\lambda)}{a!} \sum_{b \leq a}\binom{a}{b}(-x_\lambda)^{a-b}x^b\\
    &= \sum_{b:|b| \leq \lfloor \beta \rfloor} x^b \sum_{a:a \geq b, |a| \leq\lfloor \beta \rfloor } \frac{\partial^a f(x_\lambda)}{a!} \binom{a}{b}(-x_\lambda)^{a-b}x^b\\
    &=: \sum_{b:|b| \leq \lfloor \beta \rfloor} x^b c_{b}.
\end{align*}
Since $\|x_\lambda\|_\infty \leq 1$ and $\partial^a f(x_\lambda) \leq F $ by their definition, we obtain $|c_{b}| \leq F$ for any $b$.
Then, for each $d \in [D]$, we define a univariate function by DNNs $g_{\lambda,d} \in \mG(C_{b}(\log_2(2/\varepsilon)+1),C_{\beta}(\log_2(1/\varepsilon))^2,C_{b}\varepsilon^{-C_b})$ which satisfies $\|( x \mapsto x^{b_d}) - g_{\lambda,d}\|_{L^\infty([-1/\ell,1/\ell])} \leq \varepsilon$ with any $\varepsilon > 0$ by following Lemma \ref{lem:approx_poly1} or Lemma \ref{lem:approx_poly2}.
Also, we set $g_{c,D} \in \mG(C_D \log_2(1/\varepsilon),C_D (\log_2(1/\varepsilon))^2,c_c)$ which approximate $D$-variate multiplication as Lemma \ref{lem:multi_multi} by substituting $m = \log_2(1/\varepsilon)/2$.
With the functions, we consider the following difference
\begin{align*}
    &\|(x \mapsto x^b) - g_{c,D} (g_{\lambda,1}(\cdot),...,g_{\lambda,D}(\cdot))\|_{L^\infty([-1/\ell,1/\ell]^D)}\\
    &\leq \|(x \mapsto x^b) - g_{\lambda,1}(\cdot)\otimes \cdots \otimes g_{\lambda,D}(\cdot)\|_{L^\infty([-1/\ell,1/\ell]^D)}\\
    & \quad + \|g_{\lambda,1}(\cdot)\otimes \cdots \otimes g_{\lambda,D}(\cdot) - g_{c,D} (g_{\lambda,1}(\cdot),...,g_{\lambda,D}(\cdot))\|_{L^\infty([-1/\ell,1/\ell]^D)}\\
    & \leq D \ell^{-2}\varepsilon + D\ell^{-2}\varepsilon.
\end{align*}
We then define a function by DNNs as $g_\lambda(x) := f(x_\lambda) +  \sum_{b:|b| \leq \lfloor \beta \rfloor} c_{b} g_{c,D} (g_{\lambda,1}(x + x_{\lambda,1}),...,g_{\lambda,D}(x + x_{\lambda,D})) \in \mG(C_{\beta,D}\lfloor \beta \rfloor( \log_2(1/\varepsilon) + 1),C_{\beta,D} (\log_2 (1/\varepsilon))^2,C_{\beta}\varepsilon^{- C_\beta})$ with $C_\beta > 0$, which satisfies
\begin{align}
    &\|f - g_\lambda\|_{L^\infty(I_\lambda)}\notag \\
    & \leq \sum_{b:|b| \leq \lfloor \beta \rfloor}c_b \| (x \mapsto x^b) - g_{c,D} (g_{\lambda,1}(\cdot + x_{\lambda,1}),...,g_{\lambda,D}(\cdot + x_{\lambda,D}))\|_{L^\infty((I_\lambda)} \notag \\
    & \quad + \|  R_\lambda\|_{L^\infty((I_\lambda)}\notag  \\
    &\leq 2C_\beta D\ell^{-2}\varepsilon + C_{D,\beta,F} \ell^{-\beta}, \label{ineq:g_f_lambda}
\end{align}
for any $\lambda \in \Lambda$.

Finally, we approximate $f \in H^\beta(I^D)$ on $R$.
For a preparation, we define an approximator for the indicator function $\mone_{I_\lambda}$ for $\lambda \in \Lambda$.
From Lemma \ref{lem:step_approx}, we define $g_s \in \mG(3,11,C_{B,q}\varepsilon^{-8})$ be an approximator for a step function $\mone_{\{\cdot \geq 0\}}$.
Then, we define 
\begin{align*}
    g_I(x) := ( g_s(x + 1/2) + g_s(-x - 1/2) - 1), 
\end{align*}
which satisfies
\begin{align*}
    &\|g_I - \mone_{\{-1/2 \leq \cdot \leq 1/2\}}\|_{L^2(\R)} \\
    &\leq \|( g_I(\cdot + 1/2) + g_s(-\cdot - 1/2) - 1) - (\mone_{\{\cdot \geq -1/2\}} + \mone_{\{\cdot \leq 1/2\}} -1)\|_{L^2(\R)} \\
    & \quad + \|(\mone_{\{\cdot \geq -1/2\}} + \mone_{\{\cdot \leq 1/2\}} -1) - \mone_{\{-1/2 \leq \cdot \leq 1/2\}}\|_{L^2(\R)}\\
    & \leq \| g_s(\cdot + 1/2) - \mone_{\{\cdot \geq -1/2\}} \|_{L^2(\R)} + \|g_s(-\cdot - 1/2)  + \mone_{\{\cdot \leq 1/2\}}\|_{L^2(\R)} + 0 \\
    & \leq 2\varepsilon.
\end{align*}
Further, for $\lambda \in \Lambda$ and $x \in I^D$, we define $g_{I_\lambda} \in \mG(4+(\log_2(1/\varepsilon)/2+1)(D-1),25D + (\log_2(1/\varepsilon)/2)^2 D,C_{B,q}\varepsilon^{-8})$ as
\begin{align*}
    g_{I_\lambda} (x) = g_{c,D} (g_I((x_1 - x_{\lambda,1})\ell),...,g_I((x_D - x_{\lambda,D})\ell)),
\end{align*}
which is analogous to $\mone_{I_\lambda}(x) = \Pi_{d \in [D]} \mone_{\{x_{\lambda,d}- 1/(2\ell) \leq x_d \leq x_{\lambda,d}+ 1/(2\ell)\}}(x)$.
We bound the distance as
\begin{align}
    &\|\mone_{I_\lambda} - g_{I_\lambda} \|_{L^2(I^D)} \notag \\
    & \leq \|\Pi_{d = 1}^D \mone_{\{x_{\lambda,d} -1/(2\ell) \leq x_d \leq x_{\lambda,d}+ 1/(2\ell)\}} -\Pi_{d = 1}^D g_c((x_d - x_{\lambda,d})\ell )\|_{L^2(I^D)}\notag \\
    & \quad + \|\Pi_{d = 1}^D g_c((x_d - x_{\lambda,d})\ell) - g_\times (g_c((x_1 - x_{\lambda,1})\ell),...,g_c((x_D - x_{\lambda,D})\ell)) \|_{L^2(I^D)}  \notag \\
    &\leq \Sigma_{d=1}^D \|(\mone_{\{x_{\lambda,d} -1/(2\ell) \leq x_d \leq x_{\lambda,d}+ 1/(2\ell)\}} -g_c((x_{d} - x_{\lambda,d})\ell ) ) \notag \\
    &\quad  \times \Pi_{d'\neq d}^{d-1} \mone_{\{x_{\lambda,d'} -1/(2\ell) \leq x_{d'} \leq x_{\lambda,d'}+ 1/(2\ell)\}} \vee g_c((x_{d'} - x_{\lambda,d'})\ell ) \|_{L^2(I^D)} + D \varepsilon \notag \\
    &\leq \Sigma_{d=1}^D  \|\mone_{\{x_{\lambda,d} -1/(2\ell) \leq x_d \leq x_{\lambda,d}+ 1/(2\ell)\}} -g_c((x_{d} - x_{\lambda,d})\ell ) \|_{L^2(I^D)}  + D \varepsilon \notag \\
    & \leq D\varepsilon /\ell + D \varepsilon. \label{ineq:g_i_lamda}
\end{align}
Here, the second last inequality follows a bounded property of the indicator functions and $g_c$, and the H\"older's inequality.

Finally, we unify the approximator on a set $R\subset I^D$.
Let us define $\Lambda_R := \{\lambda: \mathrm{vol}(R \cap I_\lambda) \neq 0\}$, and 
\begin{align*}
    I_R := \bigcup_{\lambda \in \Lambda_R} I_\lambda.
\end{align*}
Now, we can find a constant $C_\Lambda$ such that we have
\begin{align}
    \left| \mathrm{vol}(R) - \mathrm{vol}(I_R) \right|& \leq \bigcup_{\lambda \in \Lambda : \partial R \cap I_\lambda \neq \emptyset} \mathrm{vol}(I_\lambda) \leq C_{\Lambda,J} \ell^{-1},\label{ineq:R}
\end{align}
where $\partial R$ is a boundary of $R$.
The second inequality holds since $\partial R$ is a $(D-1)$-dimensional set in the sense of the box counting dimension.
Then, we define an approximator $g_f \in \mG(C_{\beta,D}(\lfloor \beta \rfloor + \log_2(1/\varepsilon) + 1), C_{\beta,D}\ell^{D}((\log_2(1/\varepsilon))^2 + 1) , C_{B,q} \varepsilon^{-8 \wedge -C_\beta})$ as
\begin{align*}
    g_f(x) := \sum_{\lambda \in \Lambda_R} g_c(g_\lambda(x), g_{I_\lambda}(x)).
\end{align*}
Then, the error between $g_f$ and $f$ is decomposed as
\begin{align*}
    &\|f-g_f\|_{L^2(R)} \\
    &\leq  \|\Sigma_{\lambda \in \Lambda_R} \mone_{I_\lambda} \otimes f - \Sigma_{\lambda \in \Lambda_R}  g_\times (g_\lambda(\cdot),g_{I_\lambda}(\cdot)) \|_{L^2(I_\lambda)} \\
    & \leq \Sigma_{\lambda \in \Lambda} \|\mone_{I_\lambda} \otimes  f - g_\lambda \otimes g_{I_\lambda}\|_{L^2(I^D)} + |\Lambda_R| 2 \ell^{-1} \varepsilon \\
    & \leq \Sigma_{\lambda \in \Lambda} \|\mone_{I_\lambda}\otimes (f - g_{I_\lambda}) \|_{L^2(I^D)} + \|( \mone_{I_\lambda}- g_{I_\lambda}) \otimes g_{\lambda}\|_{L^2(I^D)} + |\Lambda_R| 2 \ell^{-1} \varepsilon\\
    & \leq \Sigma_{\lambda \in \Lambda} \|f - g_{I_\lambda}\|_{L^2(I_\lambda)}   + \| g_{\lambda} \|_{L^\infty(I^D)} \| (\mone_{I_\lambda}- g_{I_\lambda})\|_{L^2(I^D)} + |\Lambda_R| 2 \ell^{-1} \varepsilon\\
    & \leq \mathrm{vol}(I_R) ( 2C_\beta D \ell^{-2}\varepsilon + D \varepsilon / \ell +  C_{D,\beta,F} \ell^{-\beta})\\
    & \quad +|\Lambda_R|( F (D\varepsilon + 2 D\varepsilon) +  2 \ell^{-1} \varepsilon)\\
    & \leq \mathrm{vol}(R) (2 C_\beta D \ell^{-2}\varepsilon  + D \varepsilon / \ell + C_{D,\beta,F} \ell^{-\beta}) (1 +  C_{\Lambda,J}\ell^{-1})\\
    & \quad +\ell^{-D}(3 F D\varepsilon+  2 \ell^{-1} \varepsilon),
\end{align*}
by the H\"older's inequality and the result in \eqref{ineq:g_f_lambda}, \eqref{ineq:g_i_lamda} and \eqref{ineq:R}.
%Here, $C_g$ denotes an upper bound of $\|g_\lambda\|_{L^\infty(I^D)}$ following Lemma \ref{lem:bounded}.
We set $\ell=\lceil \delta^{-1/\beta} \rceil$ and $\varepsilon = \delta^2$ with $\delta > 0$, hence we have
\begin{align*}
    &\|f-g_f\|_{L^2(R)}\\
    &\leq C_{\beta,D,F,J}\mathrm{vol}(R)( \delta + \delta^{2+2/\beta} +  \delta^{2+1/\beta}) + C_{D,F} (\delta^2 + \delta^{2 + (D+1)/\beta}).
\end{align*}
Then, adjusting the coefficients as we can ignore the smaller order terms than $\Theta(\delta^2)$ as $\delta \to 0$, we obtain the statement.
\end{proof}

\begin{lemma}\label{lem:multi_multi}
    Suppose Assumption \ref{asmp:activation} holds.
    Then, for any $m \in \N$, $B>0$, and $D' \geq 2$, we obtain
    \begin{align*}
        \inf_{g \in \mG((m+1)(D'-1),h_c(m)(D'-1),c_m)}\left\| (x \mapsto \Pi_{d \in [D']}x_d) - g \right\|_{L^\infty([-B,B]^{D'})} \leq D'B^2 2^{-2m},
    \end{align*}
where $h_c(m):=\frac{9m^2+15m}{2}+10$.
\end{lemma}
\begin{proof}[Proof of Lemma \ref{lem:multi_multi}]
Let $g_c \in \mG(m+1, h_c(m), c_c)$ from the proof of Lemma \ref{lem:approx_poly2}.
We prove it by induction.
When $D'=2$, the statement holds by the property of $g_c$.
Then, consider $D' = \overline{D}-1$ case with $\overline{D}$, and suppose that $g_{c,\overline{D}} \in \mG((m+1)(\overline{D}-2),h_c(m)(\overline{D}-2),c_m)$ satisfies $\| (x \mapsto \Pi_{d \in [D']}x_d) - g \|_{L^\infty([-B,B]^{D'})} \leq (D'-1) B^2 2^{-2m}$.
Then, we define $g_{c,\overline{D}+1} \in \mG((m+1)\overline{D},h_c(m)\overline{D},c_m)$ such as
\begin{align*}
    g_{c,\overline{D}}(x_1,...,x_{\overline{D}}) = g_c(g_{c,\overline{D}-1}(x_1,...,x_{\overline{D}-1}),x_{\overline{D}}).
\end{align*}
Then, we bound the difference as
\begin{align*}
    &\|(x \mapsto \Pi_{d \in [\overline{D}]} x_d) - g_{c,\overline{D}} \|_{L^\infty([-B,B]^{\overline{D}})}\\
    &\leq B^2 2^{-2m} + \|(x \mapsto \Pi_{d \in [\overline{D}-1]} x_d) - g_{c,\overline{D}-1} \|_{L^\infty([-B,B]^{\overline{D}-1})}\\
    &\leq B^2 2^{-2m} + (\overline{D}-1) B^2 2^{-2m} = \overline{D} B^2 2^{-2m}.
\end{align*}
Then, by the induction, we obtain the statement for any $D'\geq 2$.
\end{proof}

\begin{lemma}[General Version of Lemma \ref{lem:main_indicator}]\label{lem:indicator_approx}
    Suppose Assumption \ref{asmp:activation} holds with $N > \alpha$.
    Then, for $\{R_m\}_{m \in [M]} \in \mR_{\alpha,M}$ and any $\varepsilon> 0$ and $m' \in \N$, there exists $f \in \mG(C_{\alpha,D,F,J} (\lfloor \alpha \rfloor  + \log_2(1/\varepsilon) + 1), C_{\alpha,D,F,J} (\varepsilon^{-2(D-1)/\alpha} (  \log_2(1/  \varepsilon))^2 + M(  \log_2(1/  \varepsilon))^2  + 1), C_{F,J,q}\varepsilon^{-16 \wedge -C_\alpha})$ with a $M$-dimensional output $f(x) = (f_1(x),...,f_M(x))^\top$ such that 
    \begin{align*}
        \|\mone_{R_m} - f_m\|_{L^2(I)} \leq  \varepsilon,
    \end{align*}
    and 
    \begin{align*}
        \|f_m\|_{L^\infty(I^D)} \leq 1 + \varepsilon,
    \end{align*}
    for all $m \in [M]$.
\end{lemma}
\begin{proof}[Proof of Lemma \ref{lem:indicator_approx}]
We define a function by DNNs $g_{h,j} \in \mG(C_{\alpha,D,F} (\lfloor \alpha \rfloor + \log_2(1/\delta) + 1), C_{\alpha,D,F}  \delta^{-(D-1)/\alpha} (  \log_2(1/  \delta))^2, C_{\alpha,D,F}\delta^{-16 \wedge -C_\alpha})$ such that $\|h_j - g_{h,j}\|_{L^2(I^{D-1})} \leq \delta $ by Lemma \ref{lem:smooth_approx}, for $\delta > 0$.
Also, we define $g_{c,J} \in \mG(J(\log_2(1/\delta)),C(\log_2(1/\delta))^2,C)$ with $m' \geq 1$ as Lemma \ref{lem:multi_multi}, and $g_s \in \mG(2,6,C_{F,q}\delta^{-8})$ such that  $\|g_s - \mone_{\{\cdot \geq 0\}}\|_{L^2([-F,F])} \leq \delta$ as Lemma \ref{lem:step_approx}.
Then, for $m \in [M]$, we define a function $g_{h,j,+}, g_{h,j,-} \in \mG(C_{\alpha,D,F} (\lfloor \alpha \rfloor + \log_2(1/\delta) + 1), C_{\alpha,D,F} (\delta^{-(D-1)/\alpha} (  \log_2(1/  \delta))^2 + 1), C_{B,q}\delta^{-16 \wedge -C_\alpha})$ to approximate $ \mone_{\{x_{d_j} \lesseqgtr h_j(x_{-d_j})\}}$ as
\begin{align*}
    g_{h,j,+}(x) = g_s(x_{d_j} - h_j(x_{-d_j})) , \mbox{~and~}g_{h,j,-}(x) = g_s(-x_{d_j} + h_j(x_{-d_j})).
\end{align*}
To approximate $\mone_{R_m}(x) = \Pi_{j \in [J]} \mone_{\{x_{d_j} \lesseqgtr h_j(x_{-d_j})\}}$, we define  $g_{R,m} \in \mG(C_{\alpha,D,F} (\lfloor \alpha \rfloor + \log_2(1/\delta) + 1), C_{\alpha,D,F} (J\delta^{-(D-1)/\alpha} (  \log_2(1/  \delta))^2 + \log_2(1/  \delta) + 1), C_{F,q}\delta^{-16})$ as $g_{R,m}:[0,1]^D \to [0,1] $ such that 
\begin{align*}
    g_{R,m}(x) = g_{c,J}( g_{h,1,\pm}(x),..., g_{h,J,\pm}(x)),
\end{align*}
and define $g_{R} \in \mG(C_{\alpha,D,F} (\lfloor \alpha \rfloor + \log_2(1/\delta) + 1), C_{\alpha,D,F} (J\delta^{-(D-1)/\alpha} (  \log_2(1/  \delta))^2 + M(\log_2(1/  \delta))^2 + 1), C_{F,q}\delta^{-16 \wedge -C_\alpha})$ as $g_{R}:[0,1]^D \to [0,1]^M $ as
\begin{align*}
    g_{R}(x) = (g_{R_1}(x),...,g_{R_M}(x))^\top.
\end{align*}

Then, its approximation error is bounded as
\begin{align*}
    & \left\| \mone_{R_m}  -  g_{R,m} \right\|_{L^2(I^D)} \notag \\
    & \leq \left\| \left(x \mapsto  \Pi_{j \in [J]} \mone_{\{x_{d_j} \lesseqgtr h_j(x_{-d_j})\}}\right) - \Pi_{j \in [J]} g_{h,j,\pm}\right\|_{L^2(I^D)} +J \delta\\
    & \leq \sum_{j\in[J]} \left\|(x \mapsto  \mone_{\{x_{d_j} \lesseqgtr h_j(x_{-d_j})\}}) -  g_{h,j,\pm}\right\|_{L^2(I^D)} \\
    & \quad \times \prod_{j' \in [j] }\left\|(x \mapsto  \mone_{\{x_{d_j} \lesseqgtr h_j(x_{-d_j})\}})\right\|_{L^\infty(I^D)} \prod_{j''\in [J] \backslash [j]}\| g_{h,j,\pm}\|_{L^\infty(I^D)} + J \delta\\
    & \leq \sum_{j \in[J]} \left\| (x \mapsto  \mone_{\{x_{d_j} \lesseqgtr h_j(x_{-d_j})\}}) -  g_{h,j,\pm}\right\|_{L^2(I^D)} + J \delta\\
    &\leq \sum_{j \in[J]} \left\| (x \mapsto \mone_{\{x_{d_j} \lesseqgtr h_j(x_{-d_j})\}}) -  (x \mapsto  \mone_{\{x_{d_j} \lesseqgtr g_{h,j}(x_{-d_j})\}})\right\|_{L^2(I^D)} \\
    &\quad + \sum_{j \in[J]} \left\|  (x \mapsto  \mone_{\{x_{d_j} \lesseqgtr g_{h,j}(x_{-d_j})\}}) - g_{h,j,\pm}\right\|_{L^2(I^D)} + J \delta\\
    &=: \sum_{j \in[J]} T_{h,1,j} + \sum_{j \in[J]} T_{h,2,j} +J \delta,
\end{align*}
where $\| g_{h,j,\pm} \|_{L^\infty(I^D)}\leq 1$ is used in the last inequality.

We evaluate each of the two terms $T_{h,1,j}$ and $T_{h,2,j}$.
As preparation, for sets $\Omega, \Omega' \subset I^{D-1}$, we define $\Omega \Delta \Omega' := (\Omega \cup \Omega') \backslash (\Omega \cap \Omega')$.
For each $j \in [J]$, we obtain
\begin{align*}
    T_{h,1,j} &=\left\| (x \mapsto  \mone_{\{x_{d_j} \lesseqgtr h_j(x_{-d_j})\}}) -  (x \mapsto  \mone_{\{x_{d_j} \lesseqgtr g_{h,j}(x_{-d_j})\}})\right\|_{L^2(I^D)} \\
    & = \lambda \left(\{x \in I^{D} \mid x_{d_j} \lesseqgtr h_j(x_{-d_j}) \} \Delta \{x \in I^{D} \mid x_{d_j} \lesseqgtr g_{h,j}(x_{-d_j}) \} \right)^{1/2}\\
    &= \|h_j - g_{h,j}\|_{L^1(I^{D-1})}^{1/2} \leq  \|h_j - g_{h,j}\|_{L^2(I^{D-1})}^{1/2} \leq \delta^{1/2},
\end{align*}
where the second last inequality follows the Cauchy-Schwartz inequality.
About $T_{h,2,j}$, we obtain
\begin{align*}
    T_{h,2,j} &= \left\|  (x \mapsto  \mone_{\{x_{d_j} \lesseqgtr g_{h,j}(x_{-d_j})\}}) - g_{h,j,\pm}\right\|_{L^2(I^D)}\\
    &=\left\| \mone_{\{\cdot \geq 0\}} \circ  (x \mapsto  x_{d_j} - g_{h,j}(x_{-d_j})) -g_s \circ  (x \mapsto  x_{d_j} - g_{h,j}(x_{-d_j})) \right\|_{L^2(I^D)}\\
    &\leq C_F \left\| \mone_{\{\cdot \geq 0\}} - g_s \right\|_{L^1([-F,F])} \leq C_F \left\| \mone_{\{\cdot \geq 0\}} - g_s \right\|_{L^2([-F,F])} \leq C_F \delta, %\| (x \mapsto  x_{d_j} - g_{h,j}(x_{-d_j})) \|_{L^\infty(I^D)}\\
\end{align*}
by the setting of $g_s$ and the second last inequality follows the Cauchy-Schwartz inequality.

Combining the results on $T_{h,1,j}$ and $T_{h,2,j}$, we obtain
\begin{align*}
    \left\| \mone_{R_m}  -  g_{R,m} \right\|_{L^2(I^D)} \leq J(\delta^{1/2} + C_F \delta + \delta).
\end{align*}
For the second inequality of the statement, we apply the following inequality:
\begin{align*}
    \|g_{R,m}\|_{L^\infty(I^D)}&\leq \|g_{R,m} - ( x \mapsto  \Pi_{j \in [J]} \mone_{\{x_{d_j} \lesseqgtr h_j(x_{-d_j})\}})\|_{L^\infty(I^D)}\\
    & \qquad +  \|( x \mapsto  \Pi_{j \in [J]} \mone_{\{x_{d_j} \lesseqgtr h_j(x_{-d_j})\}})\|_{L^\infty(I^D)} \\
    & \leq \|g_{c,J} - (x \mapsto \Pi_{j \in J}x_j)\|_{L^\infty([0,1]^J)} + 1 \\
    & \leq J \delta + 1,
\end{align*}
We set $\varepsilon = C_F J \delta^{1/2}$, we obtain the statement.
\end{proof}

\section{Proof of Proposition \ref{prop:linear}}

The sub-optimality is well studied by Section 6 in \cite{korostelev2012minimax}.
We slightly adapt the result to our setting, and obtain the following proof.

\begin{proof}[Proof of Proposition \ref{prop:linear}]
We divide this proof into the following five steps: (i) preparation, (ii) define a sub-class of functions, (iii) reparametrize a lower bound of errors, (iv) define a subset of parameters, and (v) combine all the results.

\textbf{Step (i). Preparation}.
First, we decompose the distance $\|f^* - \hat{f}^{\mathrm{lin}}\|_{L^2(P_X)}^2$.
Let us define $\Upsilon_i(\cdot) := \Upsilon_i(\cdot ; X_1 ,...,X_n)$.
By the definition of linear estimators, we obtain
\begin{align*}
    \|f^* - \hat{f}^{\mathrm{lin}}\|_{L^2(P_X)}^2
    &=\left\|f^* - \sum_{i=1}^n (f^*(X_i) + \xi_i) \Upsilon_i \right\|_{L^2(P_X)}^2 \\
    & = \left\|f^* - \sum_{i=1}^n f^*(X_i) \Upsilon_i \right\|_{L^2(P_X)}^2 + \left\| \sum_{i=1}^n \xi_i \Upsilon_i \right\|_{L^2(P_X)}^2\\
    & \quad  + 2 \left\langle f^* - \sum_{i=1}^n f^*(X_i) \Upsilon_i, \sum_{i=1}^n \xi_i \Upsilon_i \right\rangle_{L^2(P_X)} \\
    & =: T_1^{(L)} + T_2^{(L)} + T_3^{(L)},
\end{align*}
where $\langle f,f' \rangle_{L^2(P_X)} := \int f \otimes f' dP_X$ is an inner product with respect to $P_X$.
Since $\xi_i$ is a noise variable which is independent to $X_i$, we can simplify the expectations of the terms as
\begin{align*}
    \Ep_{f^*}[T_3^{(L)}] = 2\sum_{i=1}^m \Ep_{f^*}[\xi_i]\left\langle f^* - \sum_{i=1}^n f^*(X_i) \Upsilon_i,  \Upsilon_i \right\rangle_{L^2(P_X)} = 0,
\end{align*}
and 
\begin{align*}
    \Ep_{f^*}[T_2^{(L)}] &= \sum_{i,i'=1}^n \Ep_{f^*}[\xi_i \xi_{i'}] \langle \Upsilon_i, \Upsilon_{i'}\rangle_{L^2(P_X)} = \sigma^2 \sum_{i=1}^n\|\Upsilon_i\|_{L^2(P_X)}^2.
\end{align*}
Since $T_1^{(L)}$ is a deterministic term with fixed $X_1,...,X_n$, we obtain
\begin{align}
     \Ep_{f^*}\left[\|f^* - \hat{f}^{\mathrm{lin}}\|_{L^2(P_X)}^2 \right] &= \left\|f^* - \sum_{i=1}^n f^*(X_i) \Upsilon_i \right\|_{L^2(P_X)}^2 +  \sigma^2 \sum_{i=1}^n\|\Upsilon_i \|_{L^2(P_X)}^2 \notag \\
    & \geq \left\|f^* - \sum_{i=1}^n f^*(X_i) \Upsilon_i \right\|_{L^2(P_X)}^2 \vee \sigma^2 \sum_{i=1}^n\|\Upsilon_i \|_{L^2(P_X)}^2. \label{ineq:minimax0}
\end{align}

\textbf{Step (ii). Define a class of functions}.
We investigate a lower bound of the term $\sup_{f^* \in \mF_{\alpha,\beta,M}^{PS}}\Ep_{f^*}[\|f^* - \sum_{i=1}^n f^*(X_i) \Upsilon_i \|^2]$ by considering an explicit class of piecewise smooth functions by dividing the domain $I^D$.
For $m=1,...,M-1$, we will consider a smooth boundary function $B_m : I^{D-1} \ni (x_1,...,x_{D-1}) \mapsto x_D \in I$, then define pieces
\begin{align*}
    R_m = 
    \begin{cases}
        \{x \in I^D \mid 0 \leq x_D < B_1(x_{-D})\},&\mbox{~if~}m=1,\\
        \{x \in I^D \mid B_{m-1}(x_{-D}) \leq x_D < B_{m}(x_{-D})\},&\mbox{~if~}m=2,...,M-1,\\
        \{x \in I^D \mid B_{m-1}(x_{-D}) \leq x_D \leq 1\},&\mbox{~if~}m=M.
    \end{cases}
\end{align*}
An explicit form of $B_m$ is provided below.
Let $N \in \N$ be a parameter, and consider a grid for $I^{D-1}$ such that $q_{j} := ((j_{d} -0.5)/N)_{d = 1,..,D-1}$ for $j \in \{1,...,N\}^{D-1} =: \mJ$.
We also define another index set $\mJ^+ := \{N+1,...,2N\}^{D-1}.$
Also, let $\phi \in H_1^\alpha(\R^{D-1})$ be a function such that $\phi(x)=0$ for $ x \notin I^{D-1}$, and $\phi(x) = 1 $ for $x \in [0.1, 0.9]^{D-1}$.
Then, we define a boundary function for $j \in \mJ \cup \mJ^+$ as
\begin{align*}
    B_m(x_{-D};j,r) =
    \begin{cases}
        \frac{m-1}{M} + \frac{r}{MN^\alpha}  \phi( N(x_{-D}-q_{j})),&\mbox{~if~}j \in \mJ, \\
        \frac{m-1}{M}, &\mbox{~if~}j \in \mJ^+.
    \end{cases}
\end{align*}
for $r \in \{r' \in  \N \mid 0 < r' < N^\alpha - 1  \}=: D_R$ and $m=1,...,M-1$.
Here, $\phi( N(x-q_{j}))$ is a smooth approximator for the indicator function of a hyper-cube with a center $q_{j}$, and $B_m$ is constructed by the approximated indicator functions.
Also, we define the following subset of $I^D$ by $B_m$ as
\begin{align*}
    \hat{R}_{j,r,m}
    &:= \left\{x \in I^{D} \mid \frac{m-1}{M} \leq x_D < B_m(x_{-D};j,r) \right\},
\end{align*}
for $m = 1,...,M-1$.
Moreover, we define 
\begin{align*}
    \Bar{R}_m := \left\{x \in I^{D} \mid \frac{m-1}{M} \leq x_D < \frac{m}{M} \right\}.
\end{align*}
for $m=1,...,M$.
Obviously, $\hat{R}_{j,r,m} \subset \bar{R}_m$ for any $j,r$, and $m$.
Figure \ref{fig:proof_linear} provides its illustration.

We provide a specific functional form characterized by $B_m(\cdot; j,r)$ with $r \in D_R$ and $j \in \mJ \cup \mJ^+$ for each $m=1,...,M-1$.
Let $\textbf{j} := (j_1,...,j_{M-1}) \in (\mJ \cup \mJ^+)^{\otimes M-1}$, and $c_m$ be a fixed coefficient for $m=1,...,M-1$ such that $c_{m+1}-c_m = c > 0$ holds.
For $r \in D_R$ and $\textbf{j} \in (\mJ \cup \mJ^+)^{\otimes M-1}$, we define the following function
\begin{align*}
    \check{f}(x; r, \textbf{j}) := \sum_{m=1}^{M-1}c_m \mone_{\hat{R}_{j_m,r,m}}(x).
\end{align*}
$\check{f} \in \mF_{\alpha, \beta, M}^{PS}$ holds by its construction.

\begin{figure}
    \centering
    \includegraphics[width=0.5\hsize]{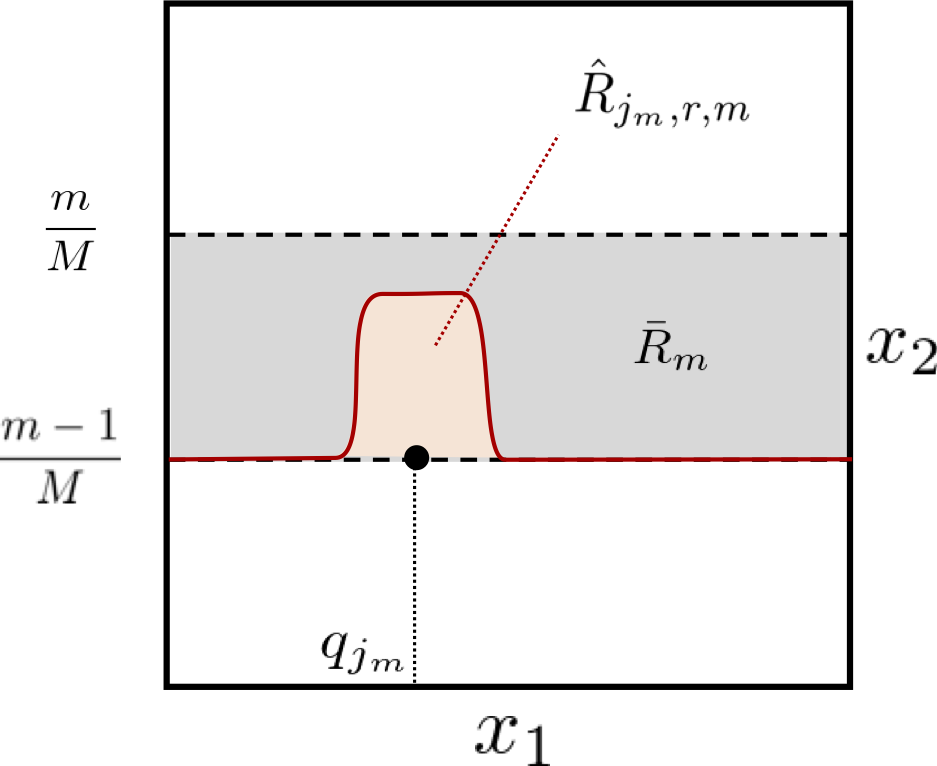}
    \caption{Illustration of $I^D$ with $D=2$ case. 
    The red curve denotes $B_m(x_{-D}; j_m, r)$.
    Also, $\Bar{R}_m$ (gray region) and $\hat{R}_{{j_m},r,m}$ (red region) are illustrated.}
    \label{fig:proof_linear}
\end{figure}

\textbf{Step (iii). Reparametrize a lower bound by the parameters of $B_m$}.
We develop a lower bound of the minimax risk by parameters $r,j,m$ of the boundary function $B_m$, with fixed $c_1,...,c_M$ and $N$.
Now, we provide a lower bound of the minimax risk as
\begin{align}
    &\sup_{f^* \in \mF_{\alpha,\beta,M}} \Ep_{f^*}\left[ \|f^* - \hat{f}^{\mathrm{lin}}\|_{L^2(P_X)}^2 \right] \notag \\
    &= \sup_{f^* \in {\mF}_{\alpha,\beta,M}} \Ep_{f^*}\left[ \left\|f^* - \sum_{i=1}^n f^*(X_i) \Upsilon_i \right\|_{L^2(P_X)}^2 \right] \notag \\
    &\geq \sup_{r\in D_R} \sup_{\textbf{j} \in (\mJ \cup \mJ^+)^{\otimes M-1}}  \Ep_{\check{f}}\left[ \left\|\check{f} - \sum_{i=1}^n \check{f}(X_i;r, \textbf{j} ) \Upsilon_i \right\|_{L^2(P_X)}^2 \right] \notag \\
    & \geq \frac{1}{|D_R|(|\mJ|+|\mJ^+|)^{M-1}} \notag\\
    & \times \sum_{r \in D_R} \sum_{\textbf{j} \in (\mJ \cup \mJ^+)^{\otimes M-1}}   \Ep_{\check{f}}\left[ \left\|\check{f}(\cdot; r, \textbf{j}) - \sum_{i=1}^n \check{f}(X_i; r, \textbf{j}) \Upsilon_i \right\|_{ L^2(P_X)}^2 \right].  \notag \\
    & \gtrsim \frac{1}{|D_R|(|\mJ|+|\mJ^+|)^{M-1}} \label{ineq:minimax1}\\
    & \times \sum_{r \in D_R} \sum_{\textbf{j} \in (\mJ \cup \mJ^+)^{\otimes M-1}}   \sum_{m = 1}^{ M-1} \sum_{j \in \mJ}  \Ep_{\check{f}}\left[ \left\|\check{f}(\cdot; r, \textbf{j}) - \sum_{i=1}^n \check{f}(X_i; r, \textbf{j}) \Upsilon_i \right\|_{ L^2(\hat{R}_{j,r,m})}^2 \right].  \notag 
\end{align}
To derive the second inequality, we consider all possible configurations.
The last inequality holds since $P_X$ has a finite and positive density by its definition.
Also, for any $r\in D_R$, $\hat{R}_{j,r,m} \cap \hat{R}_{j',r,m} = \emptyset$ for $j \neq j' \in \mJ$, and  $\cup_{j \in \mJ}\hat{R}_{j,r,m} \subset \Bar{R}_{m}$ yield the last inequality.

Afterwards, we provide a lower bound of $ \left\|\check{f}(\cdot; r, \textbf{j}) - \sum_{i=1}^n \check{f}(X_i; r, \textbf{j}) \Upsilon_i \right\|_{L^2(\hat{R}_{j,r,m})}^2$.
For $m =1,...,M-1$ and $ r \in D_R$, we can achieve
\begin{align*}
    &\sum_{j_{m} \in \mJ \cup \mJ^+} \left\|\check{f}(\cdot; r, \textbf{j}) - \sum_{i=1}^n \check{f}(X_i; r, \textbf{j}) \Upsilon_i \right\|_{ L^2(\hat{R}_{j,r,m})}^2 \\
    &= \sum_{j_{m} \in \mJ}\Biggl[ \int_{\hat{R}_{j,r,m}}\Biggl( c_{m} + c \mone_{\{x_D \geq B_{m}(x_{-D}; j_{m},r)\}}(x) \\
    & \qquad \qquad - \sum_{i=1}^n (c_{m} + c \mone_{\{x_D \geq B_{m}(x_{-D}; j_{m},r)\}}(X_i)) \Upsilon_i(x) \Biggr)^2\\
    & \quad + \Biggl( c_{m} + c \mone_{\{x_D \geq \frac{m-1}{M}\}}(x)  - \sum_{i=1}^n (c_{m} + c \mone_{\{x_D \geq \frac{m-1}{M}\}}(X_i)) \Upsilon_i(x) \Biggr)^2 dx \Biggr]\\
    & \geq \sum_{j_{m} \in \mJ} \frac{c^2}{2}\int_{\hat{R}_{j,r,m}} \left( \sum_{i : X_i \in \hat{R}_{j_{m},r,m}} \Upsilon_i(x) - \mone_{\hat{R}_{j_{m},r,m}}(x) \right)^2 dx \\
    & \gtrsim |\mJ| c^2 \left\| \sum_{i : X_i \in \hat{R}_{j,r,m}} \Upsilon_i -  1 \right\|_{L^2(\hat{R}_{j,r,m})}^2.
\end{align*}
The second last inequality follows $x^2 + y^2 \geq (x-y)^2/2$.
Substituting it into \eqref{ineq:minimax1} yield
\begin{align}
    &\sup_{f^* \in \mF_{\alpha,\beta,M}} \Ep_{f^*}\left[ \|f^* - \hat{f}^{\mathrm{lin}}\|_{L^2(P_X)}^2 \right] \notag \\
    &\gtrsim \frac{c^2}{|D_R|2^{M-1}} \sum_{r \in D_R} \sum_{m = 1}^{M-1} \sum_{j \in \mJ} \Ep_{f^*}\left[\left\| \sum_{i : X_i \in \hat{R}_{j,r,m}} \Upsilon_i -  1 \right\|_{L^2(\hat{R}_{j,r,m})}^2 \right]. \label{ineq:minimax2}
\end{align}

\textbf{Step (iv). Define subsets of parameters}.
Here, we will consider suitable subsets $\Tilde{D}_R \subset D_R$ and $\Tilde{\mJ} \subset \mJ$ for a tight lower bound.
Let us define $n_m := |\{X_i \in \Bar{R}_m\}|$ and $\tau_m^2 := \sum_{i: X_i \in \Bar{R}_m} \|\Upsilon_i\|_{L^2(\Bar{R}_m)}^2$ for each $m \in [M]$.
For $m \in [M]$, let $S_{m,r} := \{x \in I^D \mid x_D \in [m/M + r/N^\alpha, m/M + (r+1)/N^\alpha )\}$, for $r = 0,...,N^\alpha - 1$.
Then, let $\Tilde{D}_R$ be a set of integers which satisfies
\begin{align*}
    |\{i: X_i \in S_{m,r}\}| \leq \frac{c n_m}{N^\alpha}, \mbox{~and~}\sum_{i; X_i \in S_{m,r}} \| \Upsilon_i \|_{L^2(I^D)} \leq  \frac{c'\tau_m^2}{ N^\alpha},
\end{align*}
where $c,c' > 2$ are coefficients.
We can claim that at least $ (1-1/c) N^\alpha$ and $ (1-1/c') N^\alpha$ integers from $D_R$ satisfy each of the conditions, because the rest $N^\alpha/c$ and $N^\alpha/c'$ integers from $D_R$ should sum to $N^\alpha$.
Then, at least $ (1-1/c) N^\alpha + (1-1/c') N^\alpha -  N^\alpha = (1-1/c - 1/c')N^\alpha$ integers from $D_R$ satisfies the both conditions simultaneously.
We set $\Tilde{D}_R$ as a set of such the integers, then $|\Tilde{D}_R| \geq (1-1/c - 1/c')N^\alpha \gtrsim N^\alpha$ holds since $c,c' > 2$.

Further, for each $r \in \Tilde{D}_R$, we consider the subset $\Tilde{\mJ} \subset \mJ$.
Similarly, we consider $\Tilde{\mJ}$ is a set of indexes $j \in \mJ$ such as
\begin{align*}
    |\{i: X_i \in \hat{R}_{j,r,m}\}| \leq  \frac{c'' n_m}{N^\alpha N^{D-1}}, \mbox{~and~}    \sum_{i; X_i \in \hat{R}_{j,r,m}} \| \Upsilon_i\|_{L^2(I_D)} \leq \frac{c''' \tau_m^2}{ N^\alpha N^{D-1}},
\end{align*}
with coefficients $c'',c''' >2$.
Repeating the argument for $\Tilde{D}_R$, we can claim that there exist at least $(1-1/c'' - 1/c''')N^\alpha N^{D-1}$ indexes which satisfy the conditions simultaneously, then we have $\Tilde{\mJ} \geq (1-1/c'' - 1/c''')N^\alpha N^{D-1} \gtrsim N^{\alpha + D - 1}$ since $c'',c''' > 2$. 

\textbf{Step (v). Combining all the results}.
With the subsets $ \Tilde{D}_R$ and $\tilde{\mJ}$, we derive a lower bound of the norm $\left\| \sum_{i : X_i \in \hat{R}_{j,r,m}}  \Upsilon_i -  1 \right\|_{L^2(\hat{R}_{j,r,m'})}$.
For $m \in [M]$, $r \in \Tilde{D}_R$, and $j \in \Tilde{\mJ}$, we obtain
\begin{align*}
    &\left\| \sum_{i : X_i \in \hat{R}_{j,r,m}} \Upsilon_i -  1 \right\|_{L^2(\hat{R}_{j,r,m})}\\
    &\geq \left\| 1 \right\|_{L^2(\hat{R}_{j,r,m})} - \left\| \sum_{i : X_i \in \hat{R}_{j,r,m}} \Upsilon_i \right\|_{L^2(\hat{R}_{j,r,m})}\\
    &= \mathrm{vol}(\hat{R}_{j,r,m})^{1/2} - |\{i:X_i \in \hat{R}_{j,r,m}\}|^{1/2} \left(\int_{\hat{R}_{j,r,m}} \sum_{i:X_i \in \hat{R}_{j,r,m}} \Upsilon_i(x)^2dx \right)^{1/2}\\
    &\gtrsim \left( \frac{1}{M N^{D-1}N^\alpha} \right)^{1/2} -  \frac{ n_m^{1/2} \tau_m }{N^{D-1}N^\alpha} .
\end{align*}
Then, we substitute the result into \eqref{ineq:minimax2}.
Here, note that $n_m \leq n$.
Also, by \eqref{ineq:minimax0}, we have $\sum_{m \in [M]} \tau_m^2 \leq \sigma^{-2} \|f^* - \hat{f}^{\mathrm{lin}}\|_{L^2(P_X)}^2$ and its integrable conditions, $\sum_{m \in [M]} \tau_m^2 = O(1)$ with probability at least $1/2$.
Then, we continue the inequality as
\begin{align*}
    &\sup_{f^* \in \mF_{\alpha,\beta,M}} \Ep_{f^*} \left[ \|f^* - \hat{f}^{\mathrm{lin}}\|_{L^2(P_X)}^2 \right] \\
    & \gtrsim \frac{c^2}{|D_R|2^{M-1}} \sum_{m =2}^M \sum_{r \in \tilde{D}_R} \sum_{j \in \tilde{\mJ}} \Ep_{f^*}\left[\left\| \sum_{i : X_i \in \hat{R}_{j,r,m}} \Upsilon_i -  1 \right\|_{L^2(\hat{R}_{j,r,m'})}^2 \right]\\
    &\gtrsim N^{D-1} \sum_{m =2}^M \left\{ \left( \frac{1}{M N^{D-1}N^\alpha} \right)^{1/2} -  \frac{ n_m^{1/2} \tau_m }{N^{D-1}N^\alpha}\right\}^2 .
\end{align*}
We substitute $N=[ n^{1/(2\alpha + D-1)} ]$ and ignore negligible terms, then obtain the statement.
\end{proof}

\section{Sub-optimality of Wavelet Estimators}

\begin{proof}[Proof of Proposition \ref{prop:wavelet}]
As preparation, we derive a lower bound of the minimax risk.
Note that $P_X$ is a uniform distribution on $I^D$, we have $\|f\|_{L^2(P_X)} = \|f\|_{L^2}$ for any measurable $f:I^D \to \R$.
By the definition of $\hat{f}^{\mathrm{wav}}$ and the Parseval's equality, we obtain
\begin{align}
    \|f^* - \hat{f}^{\mathrm{wav}}\|_{L^2(P_X)}^2 &=\sum_{(\kappa_1,...,\kappa_D) \in \mH_\tau^{\times D}} (\hat{w}_{\kappa_1,...,\kappa_D} -{w}_{\kappa_1,...,\kappa_D}(f^*))^2  \notag \\
    & \quad + \sum_{(\kappa_1,...,\kappa_D) \in {\mH}^{\times D} \backslash \mH_\tau^{\times D}} {w}_{\kappa_1,...,\kappa_D}(f^*)^2. \label{ineq:wavelet0}
\end{align}
For the first term in the right hand side, we evaluate the expectation of $(\hat{w}_{\kappa_1,...,\kappa_D} -{w}_{\kappa_1,...,\kappa_D}(f^*))^2$.
Since $\Ep_{f^*}[\hat{w}_{\kappa_1,...,\kappa_D}] = \langle f^*, \Phi_{\kappa_1,...,\kappa_D} \rangle = {w}_{\kappa_1,...,\kappa_D}(f^*)$, we rewrite the expectation as
\begin{align*}
    &\Ep_{f^*} [(\hat{w}_{\kappa_1,...,\kappa_D} -{w}_{\kappa_1,...,\kappa_D}(f^*))^2]\\
%    &=\mathrm{Var}(\hat{w}_{\kappa_1,...,\kappa_D}) \\
    &=\mathrm{Var} \left( \frac{1}{n}\sum_{i \in [n]} Y_i \Phi_{\kappa_1,...,\kappa_D}(X_{i})\right) \\
    &= \frac{1}{n} \mathrm{Var}_{f^*} \left( (f^*(X_i) + \xi_i) \Phi_{\kappa_1,...,\kappa_D}(X_{i})\right)\\
    &= \frac{1}{n} \Ep_{X} \left[  \mathrm{Var}_{\xi} \left( (f^*(X_i) + \xi_i) \Phi_{\kappa_1,...,\kappa_D}(X_{i}) \mid X_i\right) \right] \\
    & \geq \frac{\sigma^2}{n} \Ep_{X}[\Phi_{\kappa_1,...,\kappa_D}(X_i)^2] \\
    & = \frac{\sigma^2}{n},
\end{align*}
where the third equality follows the iterated law of expectation, and the last equality follows $\Phi_{\kappa_1,...,\kappa_D}$ is an orthonormal function.
Substituting the result into \eqref{ineq:wavelet0} yields
\begin{align}
     \Ep_{f^*}\left[\|f^* - \hat{f}^{\mathrm{wav}}\|_{L^2(P_X)}^2\right] \geq \frac{\sigma^2|\mH_\tau|^D}{n}  + \sum_{(\kappa_1,...,\kappa_D) \in {\mH}^{\times D} \backslash\mH_\tau^{\times D}} {w}_{\kappa_1,...,\kappa_D}(f^*)^2. \label{ineq:wavelet2}
\end{align}

We will prove the statement by providing a specific configuration of $f^*$.
Let $\dot{R} \subset I^D$ be a hyper-rectangle as
\begin{align*}
    \dot{R} := \left\{x \in I^D \mid 0 \leq x_d \leq \frac{2}{3}, d \in [D] \right\}.
\end{align*}
Then, we define $f^* \in \mF_{\alpha,\beta,M}^{PS}$ as
\begin{align*}
    f^* := \mone_{\dot{R}}.
\end{align*}
Since $\dot{R}$ is regarded as $\cap_{d \in [D]} \{x \in I^D \mid x_d \leq 2/3\}$, $\mone_{\dot{R}}$ is a piecewise smooth function for any $M \geq 2, \alpha \geq 1$ and $\beta \geq 1$.

Then, we define the coefficient ${w}_{\kappa_1,...,\kappa_D}(f^*)$.
Fix $j_d \in \{-1,0,1,2,...\}$ for all $d \in [D]$.
Since $\mone_{\dot{R}}(x) = \prod_{d \in [D]} \mone_{\{\cdot \leq 2/3\}}(x_d)$, we can decompose the coefficient as
\begin{align}
    \langle f^*, \Phi_{\kappa_1,...,\kappa_D} \rangle &=\int_{I^D} \prod_{d \in [D]}  \mone_{\{\cdot \leq 2/3 \}}(x_d) \prod_{d \in [D]}  \phi_{\kappa_d}(x_d) d(x_1,...,x_D)\notag \\
    &= \prod_{d \in [D]} \int_I \mone_{\{\cdot \leq 2/3 \}}(x_d) \phi_{\kappa_d}(x_d) dx_d. \notag
\end{align}
For each $ d \in [D]$, a simple calculation yields
\begin{align*}
    \int_I \mone_{\{\cdot \leq 2/3 \}}(x_d) \phi_{\kappa_d}(x_d) dx_d 
    =
    \begin{cases}
        2^{-j_d/2} c_{\kappa_d}&\mbox{~if~}2/3 \in [k_{d}, k_d + 2^{-j_d})\\
        0&\mbox{~if~}2/3 \notin [k_{d}, k_d + 2^{-j_d}),
    \end{cases}
\end{align*}
with $c_{\kappa_d} > c> 0$ with a constant $c>0$.
When $2/3 \notin [k_{d}, k_d + 2^{-j_d})$, $\mone_{\{\cdot \leq 2/3 \}}(x_d)$ is a constant, then the integration is zero by the definition of $\phi_{\kappa_d}$.
Let $k_d^*$ be a $k_d \in K_{j_d}$ such that $2/3 \in [k_{d}, k_d + 2^{-j_d})$.
By the result of integration, we can rewrite $w_{\kappa_1,...,\kappa_D}$ as
\begin{align*}
    w_{\kappa_1,...,\kappa_D} \geq \prod_{d \in [D]} c 2^{-j_d / 2} \mone_{\{k_d = k_d^*\}}.
\end{align*}

Finally, we will select the truncation parameter $\tau$ for $\mH_\tau$ and update the inequality \eqref{ineq:wavelet2}.
By its definition, we obtain $|\mH_\tau| = \sum_{\ell = -1}^{\tau} 2^\ell = 2^{\tau+1}$.
Also, about the second term of \eqref{ineq:wavelet2}, we have
\begin{align*}
    \sum_{(\kappa_1,...,\kappa_D) \in {\mH}^{\times D} \backslash\mH_\tau^{\times D}} {w}_{\kappa_1,...,\kappa_D}(f^*)^2 &\geq \sum_{j_1 > \tau} \sum_{j_2 > \tau} \cdots \sum_{j_D > \tau} \prod_{d \in [D]} c^2 2^{-j_d} \\
    &= c^{2D}  \prod_{d \in [D]} \left(\sum_{j_d > \tau}2^{-j_d} \right) \\
    &= c^{2D}2^{-\tau D}.
\end{align*}
Substituting the results into \eqref{ineq:wavelet2}, we obtain
\begin{align}
     \Ep_{f^*}\left[\|f^* - \hat{f}^{\mathrm{wav}}\|_{L^2(P_X)}^2\right] \geq \frac{2 \sigma^2 2^{\tau D}}{n}  + c^{2D}2^{-\tau D}. \label{ineq:wavelet6}
\end{align}
By setting $\tau =[(2D)^{-1} \log_2 n]$ to minimize the right hand side of \eqref{ineq:wavelet6}, we obtain the statement.
\end{proof}

\section{Sub-Optimality of Other Harmonic Estimators}

\begin{proof}[Proof of Proposition \ref{prop:curvelet}]

In this proof, we provide an explicit example of $f^* \in \mF_{\alpha,\beta,M}^{PS}$, and then derive a lower bound of a risk of $\hat{f}^{\mathrm{curve}}$.
Note that $P_X$ is the uniform distribution on $[-1,1]^2$.
Also, we consider $f^*$ to be the following non-smooth function
\begin{align}
    f^*(x_1,x_2) = \mone_{\{x_1 \geq 0\}} \cdot \mone_{\{x_2 \geq 0\}}, \label{def:step}
\end{align}
and restrict it to $[-1,1]^2$.

For analysis of curvelets, we consider a Fourier transformed curvelets $\hat{\gamma}_\mu$.
As shown in (2.9) in \cite{candes2004new}, with $\xi = (\xi_1,\xi_2)$,  $\hat{\gamma}_\mu$ is written as
\begin{align*}
    \hat{\gamma}_\mu(\xi) = 2\pi \cdot \chi_{j,\ell}(\xi)\cdot u_{j,k}(R_{\theta_J}^*\xi).
\end{align*}
Here, $\chi_{j,\ell}$ is a polar symmetric window function
\begin{align*}
    \chi_{j,\ell}(\xi) = \omega(2^{-2j}\|\xi\|_2) (\nu_{j,\ell}(\theta(\xi)) + \nu_{j,\ell}(\theta(\xi) + \pi)),
\end{align*}
where $\theta(\xi) = \arcsin(\xi_2/\xi_1)$.
Here, $ \omega: \R \to \R$ is a compactly supported function, such as the Meyer wavelet, and we introduce $\nu_{j,\ell}(z) = \nu(2^j z - \pi \ell)$ where $\nu:\R \to \R$ is a function with a support $[-\pi,\pi]$ and satisfies $|\nu(\theta)|^2 + |\nu(\theta - \pi)|^2 = 1$ for $\theta \in [0,2\pi)$.
Without loss of generality, we assume that there exists a constant $c > 0$ such as $\mathrm{vol}(\{z \mid \omega(z) \geq 0\})/2 \leq \mathrm{vol}(\{z \mid \omega(z) \geq c\})$ and $\mathrm{vol}(\{z \mid \nu(z) \geq 0\})/2 \leq \mathrm{vol}(\{z \mid \nu(z) \geq c\})$, and the support of $\omega$ is $[1,2]$.
Also, $u_{j,k}: \R^2 \to \R$ is defined as
\begin{align*}
    u_{j,k}(\xi) = \frac{2^{-3j/2}}{2\pi \sqrt{\delta_1 \delta_2}} \exp(i (k_1 + 1/2)2^{-2j}\xi_1/\delta_1) \exp(ik_2 2^{-j}\xi_2/\delta_2),
\end{align*}
and $\{u_{j,k}\}_{k}$ is an orthonormal basis for an $L^2$-space on a rectangle which covers the support of $\chi_{j,\ell}$, with fixed $j$ and $\ell$.

Further, we provide a Fourier transform of $f^*$.
Since a Fourier transform of $\mone_{\{x \geq 0\}}$ is $\frac{1}{i \xi}$ and $f^*$ a product of two step functions, its Fourier transform $\hat{f}^*$ is written as $\frac{-1}{\xi_1 \xi_2}$.

To obtain the statement of Proposition \ref{prop:curvelet}, we repeat the argument for \eqref{ineq:wavelet2}, and obtain
\begin{align}
     \Ep_{f^*}\left[\|f^* - \hat{f}^{\mathrm{curve}}\|_{L^2(P_X)}^2\right] \geq \frac{\sigma^2|\mL_\tau|}{n}  + \sum_{\mu \in \mL \backslash \mL_\tau} {w}_{\gamma}(f^*)^2. \label{ineq:curvelet-1}
\end{align}

Let us consider a partial sum of the coefficients $\{\gamma_\mu\}_{\mu}$.
Here, fix $j$ and $\ell$, then consider a subset of indexes $\mL_{j',\ell'} := \{\mu \mid  j = j', \ell = \ell'\}$.
Then, since $\{u_{j,k}\}_{k}$ is an orthonormal basis, we obtain
\begin{align*}
    \sum_{\mu \in \mL_{j,\ell}} |w_\mu(f^*)|^2 = \int |\hat{f}^*(\xi)|^2 |\chi_{j,\ell}(\xi)|^2 d\xi.
\end{align*}
Then, we utilize the form of $\hat{f}^*$ and a compact support of $\chi_{j,\ell}$, hence obtain
\begin{align}
    \int |\hat{f}^*(\xi)|^2 |\chi_{j,\ell}(\xi)|^2 d\xi &= \int \frac{1}{\xi_1^2 \xi_2^2} |\chi_{j,\ell}(\xi)|^2 d\xi \notag \\
    & \geq \frac{c}{2} \mathrm{vol}(\mathrm{Supp}(\chi_{j,\ell})) \inf_{\xi \in \mathrm{Supp}(\chi_{j,\ell})}\frac{1}{\xi_1^2 \xi_2^2}. \label{ineq:curve0}
\end{align}
Since $\mathrm{Supp}(\chi_{j,\ell})$ is a set
\begin{align*}
    \left\{ \xi \in \R^2 \mid 2^{2j} \leq  \|\xi\|_2 \leq 2^{2j+1}, |\theta(\xi) - \pi \ell 2^{-j}| \leq \pi 2^{-j-1} \right\}.
\end{align*}
Hence, simply we obtain
\begin{align}
    \mathrm{vol}(\mathrm{Supp}(\chi_{j,\ell})) = \frac{3\pi}{4}2^{4j + 2} = 3 \pi 2^{4j}. \label{ineq:curve1}
\end{align}
Also, about the infimum term, we obtain
\begin{align}
    \inf_{\xi \in \mathrm{Supp}(\chi_{j,\ell})}\frac{1}{\xi_1^2 \xi_2^2} \geq  \inf_{\xi : \|\xi\|_2 \leq 2^{2j+1}}\frac{1}{\xi_1^2 \xi_2^2} = \frac{1}{(2^{2j+1}/\sqrt{2})^4} = 2^{-8j -2 }. \label{ineq:curve2}
\end{align}
Substituting \eqref{ineq:curve1} and \eqref{ineq:curve2} into \eqref{ineq:curve0}, then we obtain
\begin{align}
      \sum_{\mu \in \mL_{j,\ell}} |w_\mu(f^*)|^2 = \int |\hat{f}^*(\xi)|^2 |\chi_{j,\ell}(\xi)|^2 d\xi \geq 3 \pi 2^{-4j -2 }. \label{ineq:curve3}
\end{align}

Now, we will establish a lower bound by specifying $\mL_\tau$ and associate it with \eqref{ineq:curvelet-1}.
Let us define $\mL_{j'} := \{\mu \mid j = j' \}$.
By the setting of $\ell$ and $k$, we can obtain $|\mL_{j}| = 2^j (1 + 2^{j/2}) = 2^j + 2^{3j/2}$.
Also, with the truncation parameter $\tau$,we have $|\mL_{\tau}| = c_{\tau} ( 2^{\tau} + 2^{3{\tau}/2} - 1)$ with a coefficient $c_{\tau} > 0$.

For an approximation error, we consider
\begin{align*}
    \sum_{ \mu \in \mL \backslash \mL_{\tau}} w_\mu(f^*)^2 &= \sum_{ j \in \N_{0} \backslash ([{\tau}] \cup \{0\}) } \sum_{ \ell} \sum_{k} w_{(j,\ell,k)}(f^*)^2 \\
    & \geq \sum_{ j \in \N_{0} \backslash ([\tau] \cup \{0\}) } \sum_{ \ell} 3 \pi 2^{-4j-2} \\
    &= \frac{3 \pi}{4} \sum_{ j \in \N_{0} \backslash ([\tau] \cup \{0\}) } 2^{-3j} \\
    &=\frac{3\pi}{28} 2^{-3{\tau}}.
\end{align*}
where the inequality follows \eqref{ineq:curve3} and the second equality follows $\ell = 0,1,...,2^j-1$.

Combining the results with \eqref{ineq:curvelet-1} by setting $\mL_\tau = \mL_{\tau}$, we obtain
\begin{align}
     \Ep_{f^*}\left[\|f^* - \hat{f}^{\mathrm{curve}}\|_{L^2(P_X)}^2\right] \geq \frac{\sigma^2 c_{\tau}(2^{\tau} + 2^{3{\tau}/2} - 1)}{n}  + \frac{3 \pi}{28} 2^{-3{\tau}}.
\end{align}
As we set ${\tau}$ as $2^{3{\tau}/2} = \Theta(n^{1/3})$, then obtain the statement.
\end{proof}

\bibliography{bib_master}

\end{document}